\documentclass[10pt,final,journal]{IEEEtran}

\usepackage{moreverb,url}
\usepackage{amssymb}
\usepackage{amsmath}
\usepackage{amsthm}

\usepackage{algorithmic}
\usepackage{algorithm}
\usepackage{eqparbox}
\usepackage{subcaption}
\usepackage{bm}
\usepackage{mathrsfs}
\usepackage{tikz}

\usepackage{natbib}
\setcitestyle{authoryear,open={(},close={)}}
\let\cite\citep

\newcommand\BibTeX{{\rmfamily B\kern-.05em \textsc{i\kern-.025em b}\kern-.08em
T\kern-.1667em\lower.7ex\hbox{E}\kern-.125emX}}

\newtheorem{theorem}{Theorem}
\newtheorem{assumption}{Assumption}

\newtheorem{proposition}{Proposition}
\newtheorem{definition}{Definition}
\newtheorem{lemma}{Lemma}

\newcolumntype{C}[1]{>{\centering\let\newline\\\arraybackslash\hspace{0pt}}m{#1}}
\newcolumntype{L}[1]{>{\raggedright\let\newline\\\arraybackslash\hspace{0pt}}m{#1}}

\newcommand{\rev}[1]{#1}

\newcommand{\revv}[1]{}

\newcommand{\ie}{i.e., }
\newcommand{\eg}{e.g., }

\newcommand{\EE}{\mathbb{E}}

\newcommand{\integers}{\mathbb{Z}}
\newcommand{\naturals}{\mathbb{N}}
\newcommand{\reals}{\mathbb{R}}
\newcommand{\tspace}{\mathscr{T}}  
\newcommand{\ts}{\tspace}

\newcommand{\xspace}[1]{\mathscr{X}_{#1}}
\newcommand{\yspace}[1]{\mathscr{Y}_{#1}}

\newcommand{\statespace}{\xspace{}} 
\renewcommand{\ss}{\xspace{}}
\newcommand{\ssos}{\mathscr{O}} 
\newcommand{\ssd}[2]{\xspace{#1}^{#2}}
\newcommand{\controlspace}{\mathscr{U}} 
\newcommand{\cs}{\controlspace} 

\newcommand{\xo}{z} 
\newcommand{\xoo}{z^{\prime}} 
\newcommand{\ue}{\bar{u}} 

\newcommand{\dnode}[2]{\xo_{#1}^{(#2)}}
\newcommand{\rspace}[1]{\mathscr{R}^{#1}}
\newcommand{\rspacec}[1]{\bar{\mathscr{R}}^{#1}} 

\DeclareMathOperator{\rank}{rank}

\newcommand{\tensor}[1]{\boldsymbol{\mathcal{#1}}}

\newcommand{\bvec}[1]{\bm{#1}} 
\newcommand{\mat}[1]{\mathbf{#1}}

\newcommand{\FT}{\hat{f}} 
\newcommand{\funfold}[1]{f^{#1}} 
\newcommand{\fcore}[1]{\mathcal{F}_{#1}} 
\newcommand{\core}[1]{{\cal #1}}

\newcommand{\ffiber}[2]{f_{#1}^{(#2)}} 

\newcommand{\mvf}[1]{\mathcal{#1}} 

\newcommand{\drift}{B}
\newcommand{\diffusion}{\mvf{D}}

\newcommand{\dimx}{d}
\newcommand{\dimw}{d_w}

\newcommand{\dx}{\dimx} 
\newcommand{\du}{d_{u}} 
\newcommand{\dw}{\dimw} 

\newcommand{\ct}{f}
\newcommand{\ctt}{\tilde{f}}
\newcommand{\mdp}{\mathscr{M}}
\newcommand{\costfunc}{c}
\newcommand{\costfuncc}{\bar{c}}
\newcommand{\stagecost}[1]{g^{#1}}
\newcommand{\termcost}[1]{\psi^{#1}}
\newcommand{\hamiltonian}{\check{H}}
\newcommand{\valbase}{v}
\newcommand{\polbase}{w}
\newcommand{\polrhs}[1]{H^{#1}}
\newcommand{\pollin}[1]{\Pi^{#1}_{\mu}} 
\newcommand{\polfp}[1]{T_{\mu}^{#1}} 
\newcommand{\polfpi}[2]{T_{\mu_{#2}}^{#1}}
\newcommand{\polfpopt}[1]{T^{#1}} 
\newcommand{\polfpw}[1]{T_{\costfunc,\mu}^{#1}} 
\newcommand{\val}[1]{\valbase^{#1}}
\newcommand{\vali}[2]{\valbase^{#1}_{#2}}
\newcommand{\polval}[1]{\polbase^{#1}}

\newcommand{\polvalt}[1]{\widetilde{\polbase}^#1}
\newcommand{\polvali}[2]{\polbase^{#1}_{#2}} 
\newcommand{\polvalai}[2]{\widehat{\polbase}^{#1}_{#2}}

\newcommand{\polvala}[1]{\widehat{\polbase}^{#1}}
\newcommand{\polerr}[1]{\Delta \polval{#1}}

\newcommand{\PP}{\mathbb{P}}
\newcommand{\pmat}{\mathbf{P}}
\newcommand{\ptrans}[1]{p^{#1}}

\newcommand{\normP}{q}

\newcommand{\ftcross}{\textnormal{\texttt{ft-rankadapt}} }

\newcommand{\alr}{\varepsilon} 
\newcommand{\crossdelta}{\delta_{\textrm{cross}}} 
\newcommand{\roundeps}{\epsilon_{\textrm{round}}} 

\newcommand{\ctof}[1]{I_{h_{#1}+1}^{h_{#1}}}

\title{High-Dimensional Stochastic Optimal Control using Continuous Tensor Decompositions}

\author{
Alex Gorodetsky \qquad\qquad\qquad
Sertac Karaman \qquad\qquad\qquad
Youssef Marzouk
\thanks{Alex Gorodetsky is with the Department of Aerospace Engineering at the University of Michigan. Sertac Karaman and Youssef Marzouk are with the Department of Aeronautics and Astronautics at the Massachusetts Institute of Technology.}
}

\begin{document}
\maketitle
\begin{abstract} 
Motion planning and control problems are embedded and essential in almost all robotics applications. These problems are often formulated as stochastic optimal control problems and solved using dynamic programming algorithms. Unfortunately, most existing algorithms that guarantee convergence to optimal solutions suffer from the \emph{curse of dimensionality}: the run time of the algorithm grows exponentially with the dimension of the state space of the system. We propose novel dynamic programming algorithms that alleviate the curse of dimensionality in problems that exhibit certain \emph{low-rank} structure. The proposed algorithms are based on continuous tensor decompositions recently developed by the authors. Essentially, the algorithms represent high-dimensional functions (e.g., the value function) in a compressed format, and directly perform dynamic programming computations (e.g., value iteration, policy iteration) in this format. Under certain technical assumptions, the new algorithms guarantee convergence towards optimal solutions with arbitrary precision. Furthermore, the run times of the new algorithms scale polynomially with the state dimension and polynomially with the ranks of the value function. This approach realizes substantial computational savings in ``compressible'' problem instances, where value functions admit low-rank approximations. We demonstrate the new algorithms in a wide range of problems, including a simulated six-dimensional agile quadcopter maneuvering example and a seven-dimensional aircraft perching example. In some of these examples, we estimate computational savings of up to ten orders of magnitude over standard value iteration algorithms. We further demonstrate the algorithms running in real time on board a quadcopter during a flight experiment under motion capture. 
\end{abstract}


\section{Introduction}\label{sec:intro}
The control synthesis problem is to find a feedback control law, or controller, that maps each state of a given dynamical system to its control inputs, often optimizing given performance or robustness criteria~\cite{LaValle2006wu,Bertsekas:2012uq}. Control synthesis problems are prevalent in several robotics applications, such as agile maneuvering~\cite{Mellinger:2012vu}, humanoid robot motion control~\cite{Fallon:2014cd,Feng:2015ix}, and robot manipulation~\cite{Sciavicco:2000tj}, just to name a few.

Analytical approaches to control synthesis problems make simplifying assumptions on the problem setup to derive explicit formulas that determine controller parameters. Common assumptions include dynamics described by linear ordinary differential equations and Gaussian noise. In most cases, these assumptions are so severe that analytical approaches find little direct use in robotics applications. 

On the other hand, computational methods for control synthesis can be formulated for a fairly large class of dynamical systems~\cite{Bertsekas:2011tq,Bertsekas:2012uq,Prajna:2004gq}. However, unfortunately, most control synthesis problems turn out to be prohibitively computationally challenging, particularly for systems with high-dimensional state spaces. 
In fact, Bellman~\cite{Bellman1961} coined the term {\em curse of dimensionality} in 1961 to describe the fact that the computational requirements grow exponentially with increasing dimensionality of the state space of the system. 

In this paper, we propose a novel class of computational methods for stochastic optimal control problems. The new algorithms are enabled by a novel representation of the controller that allows efficient computation of the controller. 
This new representation can be viewed as a type of ``compression'' of the controller. The compression is enabled by a continuous tensor decomposition method, called the {\em function train}, which was recently proposed by the authors~\cite{Gorodetsky2015a} as a continuous analogue of the well-known tensor-train decomposition~\cite{Oseledets2010,Oseledets2011}. Our algorithms result in control synthesis problems with run time that scales polynomially with the dimension and the rank of the optimal value function. These control synthesis algorithms run several orders of magnitude faster than standard dynamic programming algorithms, such as value iteration. The resulting controllers also require several orders of magnitude less storage.


\subsection{Related work}

Computational hurdles are present in most decision making problems in the robotics domain. 
A closely related problem is motion planning: the problem of finding a dynamically-feasible, collision-free trajectory from an initial configuration to a final configuration for a robot operating in a complex environment. 
Motion planning problems are embedded and essential in almost all robotics applications, 
and they have received significant attention since the early days of robotics research~\cite{Latombe:1991vv,LaValle2006wu}. 
However, it is well known that these problems are computationally challenging~\cite{Canny:1988ul}. For instance, a simple version of the motion planning problem is PSPACE-hard~\cite{Canny:1988ul}. 
In other words, it is unlikely that there exists a complete algorithm with running time that scales polynomially with increasing degrees of freedom, \ie the dimensionality of the configuration space of the robot. 
In fact, the run times of all known complete algorithms scale exponentially with dimensionality~\cite{LaValle2006wu,Canny:1988ul}. Most of these algorithms construct a discrete abstraction of the continuous configuration space, the size of which scales exponentially with dimensionality. 

Yet, there are several practical algorithms for motion planning, some of which even provide completeness properties. For instance, a class of algorithms called sampling-based algorithms~\cite{LaValle2006wu,Kavraki:1996uy,Hsu:1997uva,LaValle:2001ww} construct a discrete abstraction, often called a roadmap, by sampling the configuration space and connecting the samples with dynamically-feasible, collision-free trajectories. The result is a class of algorithms that find a feasible solution, when one exists, in a reasonable amount of time for many problem instances, particularly for those that have good ``visibility'' properties~\cite{hsu.latombe.ea.ijrr06,kavraki.kolountzakis.ea.tro98}. 
These algorithms provide probabilistic completeness guarantees, \ie they return a solution, when one exists, with probability approaching to one as the number of samples increases.

In the same way, most practical approaches to motion planning avoid the construction of a grid to prevent intractability. Instead, they construct a ``compact'' representation of the continuous configuration space. The resulting compact data structure not only provides substantial computational gains, but also it still accurately represents the configuration space in a large class of problem instances, \eg those with good visibility properties. These claims can be made precise in provable guarantees such as probabilistic completeness and the exponential of rate of decay of the probability of failure. 

It is worth noting at this point that optimal motion planning, \ie the problem of finding a dynamically-feasible, collision-free trajectory that minimizes some cost metric, has also been studied widely~\cite{Karaman:2011vc}. In particular, sampling-based algorithms have been extended to optimal motion planning problems recently~\cite{Karaman:2011vc}. Various trajectory optimization methods have also been developed and demonstrated~\cite{Ratliff:2009wu,Zucker:2013dq}. 
\rev{
Another relevant problem that attracted attention recently is feedback motion planning, in which the goal is to synthesize a feedback control by generating controllers that track motion or trajectories~\cite{Tedrake:2010jr,Mellinger2011,Richter13}. 
}

The algorithm proposed in this paper is a different, novel approach to stochastic optimal control problems, which provides significant computational savings with provable guarantees. 
\rev{It is different from the traditional trajectory based methods described above. In particular, we formulate our problem as an optimal stochastic control problem and seek a feedback control \textit{offline} that generates optimal behavior. We do not seek trajectories or attempt to follow them; rather we seek actions to be applied by the system in particular states. As such our approach attempts to force the system to ``discover'' behavior by leveraging its dynamics and the rewards or costs a user provides. Furthermore, we do not perform any linearization of the dynamics or around trajectories; we seek a feedback control, using offline computation, for the full nonlinear non-affine system by solving a dynamic programming problem.}

\rev{
Our framework is conceptually similar to the goals of certain approximate dynamic programming algorithms~\cite{Powell2007}. In particular, we look for reduced representations of value functions in an adaptive manner with a prescribed accuracy. The reduced representations are generated by exploiting low-rank \emph{multilinear} structure, and computation is performed in this reduced space. We do not restrict the complexity of the representation: if structure of low multilinear ranks does not exist, our algorithms will attain the exponential growth in complexity that is exhibited by the full problem. In these cases, limiting the rank will indeed result in certain numerical approximations. However, there are many reasons to believe that low-rank multilinear structure is present in many problem formulations, and we discuss these reasons throughout the paper. 

Aside from structured representation of value functions, we do not approximate other aspects of the dynamic programming problem. For example, we do not revert to suboptimal optimization strategies such as myopic optimization, approximate evaluations of the expectation through sampling, rollout, fixed-horizon lookahead, etc. \revv{We further discuss the relationship between our approach and approximate dynamic programming in Section~\ref{sec:relate_adp}.}

In spirit, our approach is similar to other work that attempts to accurately represent multivariate value functions in a structured format and to perform computation entirely in that format. One example, called SPUDD~\cite{hoey1999}, represents value functions as algebraic \revv{decision} diagrams (ADDs). That work derives the computations needed for value iteration in the class of functions represented as ADDs. Dynamic programming updates are then performed for every state in the state space, but the structured representation of the function reduces the complexity of these updates. Our approach differs in several ways from SPUDD: our representation exploits low-rank multilinear structure; our dynamic programming algorithms use this structure to avoid performing updates for every state in a discretized state space; and we consider continuous states and controls.

}

\subsection{Tensor decomposition methods}

In this paper, we propose a new class of algorithms for high-dimensional instances of stochastic optimal control problems that are based on compressing associated value functions.
Moreover, we compress a \textit{functional}, rather than a discretized, representation of the value function. This approach enables fast evaluation of the value function at arbitrary points in the state space, without any decompression. Our approach offers orders of magnitude reduction in the required storage costs.

Specifically, the new algorithms are based on the functional tensor-train (FT) decomposition~\cite{Gorodetsky2015a} recently proposed by the authors. Multivariate functions in the FT format are represented by a set of matrix-valued functions \revv{whose sizes} correspond to the FT rank. Hence, low-rank multivariate functions can be represented in the FT format with a few parameters, and in this paper we use this compression to represent the value function of an associated stochastic optimal control problem.
 
In addition to \textit{representing} the value function, we \textit{compute} with value functions directly in compressed form; no decompression is ever performed.
Specifically, we create compressed versions of value iteration, policy iteration, and other algorithms for solving dynamic programming problems. These DP problems are obtained through consistent discretizations of continuous-time continuous-space stochastic optimal control problems obtained using the Markov chain approximation (MCA) method~\cite{Kushner2001}. Note that even though the MCA method relies on discretization, the FT still allows us to maintain a \textit{functional} representation, valid for any state within the state space.

As a result, the new algorithms provide substantial computational gains in terms of both computation time and storage space, when compared to standard dynamic programming methods. 

The proposed algorithms exploit the {\em low-rank} structure commonly found in {\em separable} functions. This type of structure has been widely exploited within numerical analysis literature on tensor decompositions~\cite{Kolda2009,Hackbusch2009,Hackbusch2012}.
Indeed the FT decomposition itself is an extension of tensor train (TT) decomposition developed by Oseledets~\cite{Oseledets2010,Oseledets2011}. 
 The TT decomposition works with discrete {\em arrays}; it represents a $d$-dimensional array as a multiplication of $d$ matrices. The FT decomposition, on the other hand, works directly with {\em functions}. This representation allows a wider variety of possible operations to be performed in FT format, \eg integration, differentiation, addition, and multiplication of low-rank functions. These operations are problematic in the purely discrete framework of tensor decompositions since element-wise operation with tensors is undefined, \eg one cannot add arrays with different number of elements. 

 We use the term {\em  compressed continuous computation} for such FT-based numerical methods~\cite{Gorodetsky2015a}. Compressed continuous computation algorithms can be considered an extension of the {\em continuous computation} framework to high-dimensional function spaces via the FT-based compressed representation. 
The continuous computation framework, roughly referring to computing directly with functions (as opposed to discrete arrays), was first realized by Chebfun, a Matlab software package developed by Trefethen, Battles, Townsend, Platte, and others~\cite{Platte2010}. In this software package, the user computes with univariate~\cite{Battles2004,Platte2010}, bivariate~\cite{Townsend2013}, and trivariate~\cite{Chebfun3} functions that are represented in Chebyshev polynomial bases. Our recent work~\cite{Gorodetsky2015} extended this framework to the general multivariate case by building a bridge between continuous computation and low-rank tensor decompositions. Our prior work has resulted in the software package, called Compressed Continuous Computation ($C^3$)~\cite{c3}, implemented in the C programming language. Examples presented in this paper use the $C^3$ software package, and they are available online on GitHub~\cite{c3sc}. 

In short, our compressed continuous computation framework leverages both the advantages of continuous computation and low-rank tensor decompositions.
The advantage of {\em compression} is tractability when working directly with high-dimensional computational structures. For instance, a seven-dimensional array with one hundred points in each dimension includes trillion points in total. As a result, even the storage space required cannot be satisfied by any existing computer, let alone the computation times. The computational requirements increase rapidly with increasing dimensionality. 

The advantages of the functional, or {\em continuous}, representation (over the array-based TT) include the ability to compare value functions resulting from different discretization levels and the ability to evaluate these functions outside of some discrete set of nodes. We leverage these advantage in this paper in two ways: {\em (i)} we develop multi-level schemes based on low-rank prolongation and interpolation operators; {\em (ii)} we evaluate optimal policies for any state in the state space during the execution of the controller. 

\subsection{Contributions}

The main contributions of this paper are as follows. First, we propose novel compressed continuous computation algorithms for dynamic programming. Specifically, we utilize the function train decomposition algorithms to design FT-based value iteration, policy iteration, and one-way multigrid algorithms. These algorithms work with a Markov chain approximation for a given continuous-time continuous-space stochastic optimal control problem. They utilize the FT-based representation of the value function to map the discretization due to Markov chain approximation into a compressed, functional representation.

Second, we prove that, under certain conditions, the new algorithms guarantee convergence to optimal solutions, whenever the standard dynamic programming algorithms also guarantee convergence for the same problem instance. We also prove upper bounds on computational requirements. In particular, we show that the run time of the new algorithms scale polynomially with dimension and polynomially with the rank of the value function, while even the storage requirements for existing dynamic programming algorithms clearly scale exponentially with dimension. 

Third, we demonstrate the new algorithms in challenging problem instances. 
In particular, we consider perching problem that features non-linear non-holonomic non-control-affine dynamics. We estimate that the computational savings reach roughly ten orders of magnitude. In particular, the controller that we find fits in roughly 1MB of space in the compressed FT format; we estimate that a full look up table, for instance, one computed using standard dynamic programming algorithms, would have required around 20 TB of memory. 

We also consider the problem of maneuvering a quadcopter through a small window. This leads to a six-dimensional non-linear non-holonomic non-control-affine stochastic optimal control problem. We compute a near-optimal solution. We demonstrate the resulting controller in both simulation and experiment. In experiment, we utilize a motion capture system for full state information, and run the resulting controller in real time on board the vehicle. 

\rev{
A preliminary version of this paper appeared at the Robotics Science and Systems conference~\cite{Gorodetsky2015}. In the present version, we use continuous, rather than discrete, tensor decompositions and add significantly more algorithmic development, theory, and validation. First, the theoretical grounding behind the methodology is more thorough: the assumptions are more explicit, and the bounds are more relevant and intuitive than those provided in our prior work. Second, the approach is validated on a wider range of problems, including minimum time problems and onboard an experimental system. Third, the methodology is extended to a broader range of algorithms including policy iteration and multigrid techniques, as opposed to only value iteration. 
}

\subsection{Organization}

The paper is organized as follows. We introduce the stochastic optimal control problem in Section~\ref{sec:soc} and the Markov chain approximation method in Section~\ref{sec:discrete}. We briefly describe the compressed continuous computation framework in Section~\ref{sec:discrete}. We describe the proposed algorithms in Section~\ref{sec:lowrankdp}. We analyze their convergence properties and their computational costs in Section~\ref{sec:analysis}. We discuss a wide range of numerical examples and experiments in Section~\ref{sec:numexamples}. Section~\ref{sec:conclusion} offers some concluding remarks.

\section{Stochastic optimal control}\label{sec:soc}

In this section, we formulate a class of continuous-time continuous-space stochastic optimal control problems. Background is provided in Section~\ref{section:stochastic_optimal_control}. 
Under some mild technical assumptions, the optimal control is a Markov policy, \ie a mapping from the state space to the control space, that satisfies the Hamilton-Jacobi-Bellman (HJB) equation. We introduce the notion of Markov policies and the HJB equation in Sections~\ref{section:markov_policies} and \ref{section:hjb}, respectively. 

\subsection{Stochastic optimal control}\label{section:stochastic_optimal_control}

Denote the set of integers and the set of reals by $\integers$ and $\reals$, respectively. We denote the set of all positive real numbers by $\reals_{+}$. Similarly, the set of positive integers is denoted by $\integers_{+}$.
Let $\dx, \du, \dw \in \integers_{+}$,
$\ss \subset \reals^{\dx}$ and $\cs \subset \reals^{d_u}$ be compact sets with smooth boundaries and non-empty interiors, $\ts \subset \reals_{+}$,
and $\{w(t) : t \ge 0\}$ be a $\dw$-dimensional Brownian motion defined on some probability space $(\Omega, {\cal F}, \PP)$, where $\Omega$ is a sample space, ${\cal F}$ is a $\sigma$-algebra, and $\PP$ is a probability measure.

Consider a dynamical system described by the following stochastic differential equation in the differential form: 
\rev{
 \begin{align} \label{eqn:system}
dx(t) = \drift(x(t), u(t)) dt + \diffusion(x(t)) dw(t),
\end{align}
for all $t \in \ts$, where $\drift: \ss \times \cs \to \reals^{\dx}$ is a vector-valued function, called the {\em drift}, and $\diffusion : \statespace \to \reals^{\dx \times \dw}$ is a matrix-valued function, called the {\em diffusion}. 
}
Strictly speaking, for any admissible control process\footnote{Suppose the control process $\{u(t) : t \ge 0\}$ is defined on the same probability space $(\Omega, {\cal F}, \PP)$ which the Wiener process $\{w(t) : t \ge 0\}$ is also defined on. Then, $\{u(t) : t \ge 0\}$ is said to be \textit{admissible} with respect to $\{w(t) : t\ge0\}$, if there exists a filtration $\{{\cal F}_t : t \ge 0\}$ defined on $(\Omega, {\cal F}, \PP)$ such that $u(t)$ is ${\cal F}_t$-adapted and $w(t)$ is an ${\cal F}_t$-Wiener process. Kushner et al.~\cite{Kushner2001} provide the precise measure theoretic definitions.} $\{ u(t) : t \ge 0\}$, the solution to this differential form is a stochastic process $\{x(t) : t \ge 0\}$ satisfying the following integral equation: For all  $t \in \ts$, 
\rev{
\begin{align}
x(t) = x(0) + \int_0^t & \drift(x(\tau), u(\tau)) \, d\tau \nonumber \\
              \quad \quad \quad \quad \quad \quad \quad & + \int_{0}^t \diffusion(x(\tau), u(\tau)) \, dw(\tau), \label{eq:integral}
\end{align}
}
where the last term on the right hand side is the usual It\^o integral~\cite{oksendal:2003ug}. We assume that the drift and diffusion are measurable, continuous, and bounded functions. These conditions guarantee existence and uniqueness of the solution to Equation~\eqref{eq:integral}~\cite{oksendal:2003ug}. \rev{Finally, we consider only time-invariant dynamical systems, however, our algorithms can be extended to systems with time varying dynamics through state augmentation~\cite{Bertsekas:2012uq}.

\rev{
In this paper, we focus on a discounted-cost infinite-horizon problem, although our methodology and framework can be extended finite-horizon problems as well. Our description of the problem and the corresponding notation closely follows that of Fleming and Soner~\cite{Fleming2006}. %

 Let $\ssos \subset \ss$ denote an open subset. If $\ssos \neq \reals^d$, then let its boundary $\partial \ssos$  be a compact $(d-1)$-dimensional manifold of class $C^3$, \ie the set of 3-times differentiable functions.
Let $\stagecost{},\termcost{}$ denote continuous stage and terminal cost functions, respectively, that satisfy polynomial growth conditions:
\begin{align*} 
|\stagecost{}(x,u)| &\leq C(1 + |x|^k + |u|^k), \\
|\termcost{}(x)| &\leq C(1 + |x|^k),
\end{align*}
for some constants $C\in \reals,k \in \naturals$. 

Define the {\em exit time} $\tau$ as either the first time that the state $x(s)$ exits from $\ssos$, or we set $\tau = \infty$ if the state remains forever within $\ssos$, \ie $x(s) \in \ssos$ for all $s \geq 0$. Within this formulation, we can still use a terminal cost $\psi$ for the cases when $\tau < \infty$. To accommodate finite exit times, we use the indicator function $\chi_{\tau < \infty}$ that evaluates to one if the state exits $\ssos$ and to zero otherwise. The cost functional is defined as: 
\begin{align*}
\costfuncc(\xo;u) &= \mathbb{E}\Big[ \int_{0}^{\tau} e^{-\beta s}\stagecost{}(s,x(s),u(s))ds \\
                 &           \quad \quad \quad \quad  + \chi_{\tau < \infty} e^{-\beta \tau}\termcost{}(\tau,x(\tau)) \Big], \quad x(0) = \xo,
\end{align*}
where $\beta>0$ is a discount factor. 

The {\em discounted-cost infinite-horizon stochastic optimal control problem} is to find a control $u(t)$
such that $\costfuncc(z; u)$ is minimized for all $\xo \in \ssos$, subject to Equation~\eqref{eqn:system}. 
We require that the cost until exit is bounded 
\begin{equation*}
\mathbb{E}\left[ \int_{0}^{\tau} \exp^{-\beta s} |\stagecost{}(s,x(s),u(s))|ds \right] \leq \infty,
\end{equation*}
for this problem to be well defined~\cite{Fleming2006}.

}

}

\subsection{Markovian policies}\label{section:markov_policies}
A \textit{Markov policy} is a mapping $\mu : \ss \to \cs$ that assigns a control input to each state. Under a Markov policy $\mu$, an admissible control is obtained according to $u(t) = \mu(x(t)).$
\rev{
For the discounted-cost infinite-horizon problem, the cost functional associated with a specific Markov policy $\mu$ is denoted by}
\begin{align*}
\costfuncc_{\mu}(\xo) &= \mathbb{E}\Big[ \int_{0}^{\tau} e^{-\beta s}\stagecost{}(s,x(s),\mu(s,x(s))ds \\
                    & \quad \quad \quad \quad + \chi_{\tau < \infty} e^{-\beta \tau}\termcost{}(x(\tau)) \Big], \quad x(0) = \xo,
\end{align*}

Under certain conditions, one can show that a Markov policy is at least as good as any other arbitrary $\mathcal{F}_t$-adapted policy; see for example Theorem 11.2.3 by {\O}ksendal~\cite{oksendal:2003ug}. In this work we assume these conditions hold, and only work with Markov control policies. Storing Markov control policies allows us to avoid storing trajectories of the system when considering what action to apply. Instead, Markov policies only require knowledge of the current time and state and are computationally efficient to use in practice.

\rev{
The stochastic control problem is to find an \textit{optimal} cost $\costfuncc_{\mu^*}$ with the following property
\begin{equation*}
\costfuncc_{\mu^*}(\xo) = \inf_{\mu} \costfuncc_{\mu}(\xo), \qquad \mbox{for all } \xo,
\end{equation*}
subject to Equation~\eqref{eqn:system}. 
}

\subsection{Dynamic programming}\label{section:hjb}

We can formulate the stochastic optimal control problem as a dynamic programming problem.
In the dynamic programming formulation, we seek an optimal value function $\val{}(t,\xo)$ defined as
\begin{equation*} \label{eq:dpprob}
\val{}(t,\xo) = \inf_{\mu} \costfunc_{\mu}(t,\xo) \textrm{ for all } \xo \in \ssos.
\end{equation*}

For continuous-time continuous-space stochastic optimal control problems, the optimal value function satisfies a partial differential equation (PDE), called the Hamilton-Jacobi-Bellman (HJB) PDE~\cite{Fleming2006}. \rev{The HJB PDE is a continuous analogue of the Bellman equation~\cite{Bellman1962}, which we will be solving in compressed format. In Section~\ref{sec:discrete} we will describe how a Bellman equation arises from a \textit{discretization} of the SDE. As the discretization is refined, however, this approach converges to the solution of the HJB PDE.} 

To define the HJB PDE, we first introduce some notation. Let $\mathcal{S}_{+}^{\dx}$ denote the set of symmetric, nonnegative definite matrices. Let $\mat{A} \in \mathcal{S}_{+}^{\dx}$ and $\mvf{A} = \diffusion \diffusion^{T}$, then the trace $\text{tr } \mvf{A}\mat{A}$ is defined as
\begin{equation*}
\text{tr }\mvf{A} \mat{A} = \sum_{i,j}^{\dx} \mvf{A}[i,j]\mat{A}[i,j].
\end{equation*}

\rev{
For $\xo \in \ssos$, $p \in \statespace $, $\mat{A} \in \mathcal{S}_{+}^d$, define the \textit{Hamiltonian}  as
\begin{align*}
\hamiltonian(\xo,p,\mat{A}) &= \sup_{\ue \in \controlspace}\Big[ - \drift(\xo,\ue) \cdot p - \frac{1}{2}\text{tr } \mvf{A}(\xo,\ue) \mat{A} \\
                             & \quad \quad \quad \quad \quad \quad \quad \quad - \stagecost{}(\xo,\ue)\Big].
\end{align*}
}

\rev{
For discounted-cost infinite-horizon problems, the HJB PDE is then defined as  
\begin{equation*}\label{eq:hjb2}
\beta \val{} + \hamiltonian(\xo,\nabla \val{}, D_x^2 \val{}) = 0, \quad \xo \in \ssos,
\end{equation*}
with boundary conditions
\begin{equation*}
\val{}(\xo) = \termcost{}(\xo), \quad \xo \in \partial \ssos.
\end{equation*}
}
%

\section{The Markov chain approximation method}\label{sec:discrete}
In this section, we provide background for a solution method based on discretization that forms the basis of our computational framework.
The Markov chain approximation (MCA)~\cite{Kushner2001} and similar methods, \eg the method prescribed by Tsitsiklis~\cite{Tsitsiklis1995}, for solving the stochastic optimal control problem rely on first discretizing the state space and dynamics described by Equation~\eqref{eqn:system} and then solving the resulting discrete-time and discrete-space Markov Decision Process (MDP). The discrete MDP can be solved using standard techniques such as Value Iteration (VI) or Policy Iteration (PI) or other approximate dynamic programming techniques~\cite{Bertsekas1996,Bertsekas2007,Kushner2001,Powell2007,Bertsekas2013}. 

In Section~\ref{sec:mdp}, we provide a brief overview of discrete MDPs. In Section~\ref{sec:mca} we describe the Markov chain approximation method for discretizing continuous stochastic optimal control problems. Finally, in Section~\ref{sec:vipiml}, we describe three standard algorithms to solve discrete MDPs, namely value iteration, policy iteration, as well as multilevel methods.

\subsection{Discrete-time discrete-space Markov decision processes}\label{sec:mdp}

The Markov chain approximation method relies on discretizing the state space of the underlying stochastic dynamical system, for instance, using a grid. The discretization is parametrized by the discretization step, denoted by $h \in \reals_+$, of the grid. The discretization is finer for smaller values of $h$.

The MDP resulting from the Markov chain approximation method with a discretization step $h$ is a tuple, denoted by $\mdp^{h} = (\statespace^h,\controlspace, \ptrans{h},\stagecost{h},\termcost{h})$, where $\statespace^h$ is the set of discrete states, $\ptrans{h}(\cdot,\cdot \vert \cdot): \statespace^h \times \statespace^h \times \controlspace \to [0,1]$ is a function that denotes the transition probabilities satisfying $\sum_{\xoo \in \statespace^h}\ptrans{h}(\xo,\xoo|\ue) = 1$  for all $\xo \in \statespace^h$ and all $\ue \in \controlspace$, $\stagecost{h}$ is the stage cost of the discrete system, and $\termcost{h}$ is the terminal cost of the discrete system.

The transition probabilities replace the drift and diffusion terms of the stochastic dynamical system as the description for the evolution of the state. For example, when the process is at state $\xo \in \statespace^h$ and action $\ue \in \controlspace$ is applied, the next state of the process becomes $\xoo \in \statespace^h$ with probability $\ptrans{h}(\xo,\xoo \vert \ue)$. 

In this discrete setting, Markov policies are now mappings 
$\mu^h: \statespace^h \to \controlspace$ defined from a discrete state space rather than from the continuous space $\statespace$. Furthermore, the cost functional becomes a multidimensional array 
\rev{$\costfunc_{\mu^h} : \statespace^h \to \reals$. }
The cost associated with a particular trajectory and policy for a discrete time system, \ie $t_0 = 0,t_1 = 1,t_2=2,\ldots,$  can be written as
\rev{
\begin{align*}\label{eq:discretecost}
\costfunc_{\mu^h}\left(\xo\right) &= \mathbb{E}\Big[\sum_{i=1}^N\gamma^{i}\stagecost{h}\left(x\left(t_k\right),\mu^h\left(x\left(t_k\right)\right)\right) \\
                                     & \quad \quad \quad \quad \quad + \termcost{h}\big(x\left(t_N\right)\big)\Big], \quad x(t_0) = \xo
\end{align*}
}
for $\xo \in \statespace^h$, where $0 < \gamma < 1$ is the discount factor, and $N$ is the first time the system exists $\ssos$. The \textit{Bellman equation}, a discrete analogue of the HJB equation, describing the optimality of this discretized problem can then be written as
\rev{
\begin{equation}\label{eq:dcih}
\val{h}(\xo) = \displaystyle{\min_{\ue \in \controlspace}}\Big[\stagecost{h}(\xo,\ue)
  + \gamma \displaystyle{\sum_{\xoo \in \statespace^h}}\ptrans{h}(\xo,\xoo|\ue)\val{h}(\xoo)\Big],
\end{equation}
} where $\val{h}$ is the optimal discretized value function and satisfies the Bellman equation
\rev{
\begin{equation*}
\val{h}(\xo) = \inf_{\mu^h} \costfunc_{\mu^h}(\xo).
\end{equation*}}Therefore, it also solves the discrete MDP~\cite{Bertsekas2013}.
\subsection{Markov chain approximation method}\label{sec:mca}
The MCA method, developed by Kushner and co-workers~\cite{Kushner1977,Kushner1990,Kushner2001}, constructs a sequence of discrete MDPs such that the solution of the MDPs converge to the solution of the original continuous-time continuous-space problem. 

Let $\{\mdp^{h_\ell} :  \ell \in \naturals\}$ be a sequence of MDPs, where each $\mdp^{h_{\ell}} = (\statespace^{h_{\ell}},\controlspace, \ptrans{h_{\ell}},\stagecost{h_{\ell}},\termcost{h_{\ell}})$ is defined as before.
Define $\partial \statespace^{h_{\ell}}$ as the subset of $\statespace^{h_{\ell}}$ that falls on the boundary of $\statespace$, \ie $\partial \statespace^{h_{\ell}} = \partial \statespace \cap \statespace^{h_{\ell}}$. 
Let $\{\Delta t^{\ell} : \ell \in \naturals\}$, where $\Delta t^{\ell}: \statespace^{h_{\ell}} \to \reals_{+}$, be a sequence of {\em holding times}~\cite{Kushner2001}.
Let $\{\xi_i^{\ell} : i \in \naturals\}$, where $\xi_i^{\ell} \in \statespace^{h_{\ell}}$, be a (random) sequence of states that describe the trajectory of $\mdp^{h_{\ell}}$. We use holding times as interpolation intervals to generate a continuous-time trajectory from this discrete trajectory as follows. With a slight abuse of notation, let $\xi^{\ell} : \reals_{\ge 0} \to \statespace^{h_{\ell}}$ denote the continuous-time function defined as follows: $\xi^{\ell}(\tau) = \xi_i^{\ell}$ for all $\tau \in [t_i^{\ell}, t_{i+1}^{\ell})$, where $t_i^{\ell} = \sum_{k=0}^{i-1} \Delta t^{\ell}(\xi_k)$. 
Let $\{u^{\ell}_i : i \in \naturals\}$, where $u^{\ell}_i \in \cs$, be a sequence of control inputs defined for all ${\ell} \in \naturals$. Then, we define the continuous time interpolation of $\{u^{\ell}_i : i \in \naturals\}$ as $u^{\ell}(\tau) = u^{\ell}_i$ for all $\tau \in [t_i^{\ell}, t_{i+1}^{\ell})$. An illustration of this interpolation is provided in Figure~\ref{fig:visualization}.
\begin{figure}
\begin{center}
\includegraphics[width=0.4\textwidth]{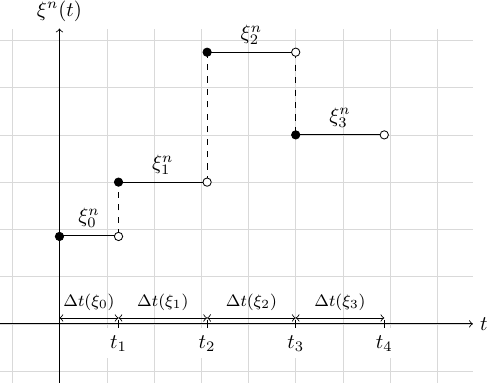}
\caption{An illustration of a continuous-time interpolation of a discrete process arising from the Markov chain approximation.}
\label{fig:visualization}
\end{center}
\end{figure}

The following result by Kushner and co-workers characterizes the conditions under which the trajectories and value functions of the discrete MDPs converge to those of the original continuous-time continuous-space stochastic system. 

\begin{theorem}[See Theorem 10.4.1 by Kushner and Dupuis~\cite{Kushner2001}] \label{theorem:kushner}
\rev{Let $\drift$ and $\diffusion$ denote the drift and diffusion terms of the stochastic differential equation~\eqref{eqn:system}.} Suppose a sequence $\{ \mdp^{h_{\ell}} : \ell \in \naturals\}$ of MDPs and the sequence $\{\Delta t^{\ell} : {\ell} \in \naturals\}$ holding times satisfy the following conditions: For any sequence of inputs $\{u_i^{\ell} : i \in \naturals\}$ and the resulting sequence of trajectories $\{\xi_i^{\ell} : i \in \naturals\}$ if 
$$
\lim_{\ell \to \infty } \Delta t^{\ell} (\xo) = 0, \textrm{ for all } \xo \in \statespace, 
$$
and
\begin{align*}
\lim_{\ell \to \infty} \frac{\EE[\xi_{i+1}^{\ell} - \xi_i^{\ell} \,\vert\, \xi_i^{\ell} = z , u_i^{\ell} = \ue ]}{\Delta t^{\ell}(z)} = \drift(z,\ue), \\
\lim_{\ell \to \infty}\frac{\mathrm{Cov}[\xi_{i+1}^{\ell} - \xi_i^{\ell} \,\vert\, \xi_i^{\ell} = z, u_i^{\ell} = \ue]}{\Delta t^{\ell}(z)} = \diffusion(z,\ue), 
%
\end{align*}
for all $z \in \statespace$ and $\ue \in \cs$.
Then, the sequence $\{(\xi^{\ell}, u^{\ell}) : \ell \in \naturals\}$ of interpolations converges in distribution to $(x, u)$ that solves the integral equation with differential form given by~\eqref{eqn:system}.
Let $\val{h_{\ell}}$ denote the optimal value function for the MDP $\mdp^{h_{\ell}}$. Then, for all $z \in \statespace^{h_{\ell}}$, 
$$
\lim_{\ell \to \infty} \vert \val{h_{\ell}}(z) - \val{}(z) \vert  = 0.
$$
\end{theorem}
The conditions of this theorem are called {\em local consistency conditions}. Roughly speaking, the theorem states that the trajectories of the discrete MDPs will converge to the trajectories of the original continuous-time stochastic dynamical system if the local consistency conditions are satisfied. Furthermore, in that case, the value function of the discrete MDPs also converge to that of the original stochastic optimal control problem. 
A discretization that satisfies the local consistency conditions is called a {\em consistent discretization}. Once a consistent discretization is obtained, standard dynamic programming algorithms such as value iteration or policy iteration~\cite{Bertsekas2007} can be used for its solution. 
\subsubsection{Discretization procedures}\label{sec:discretized}

In this section, we provide a general discretization framework described by Kushner and Dupuis~\cite{Kushner2001} along with a specific example. For the rest of the paper, we will drop the subscript $\ell$ from $h_{\ell}$ for the sake of brevity, and simply refer to the discretization step with $h$.

Let $\statespace^h \subset \statespace$ denote a discrete set of states. For each state $\xo \in \statespace^h$ define a finite set of vectors $M(z) = \{\bvec{v}_{i,\xo} \allowbreak : i < m(\xo)\},$ where $\bvec{v}_{i,\xo} \in \reals^{\dx}$ and $m(\xo): \statespace^h\to \naturals$ is uniformly bounded. These vectors denote directions from a state $\xo$ to a neighboring set of states $\{y: y = \xo + h\bvec{v}_{i,\xo}, i \leq m(\xo)\} \subset \statespace^h$. A valid discretization is described by the functions $q_i^1(\xo): \statespace^h \to \reals$ and $q_i^0(\xo,\ue): \statespace^h \times \cs \to \reals$ that satisfy
\begin{align*}
 \drift(\xo,\ue) &= \sum_{\bvec{v}_{i,\xo} \in M(\xo)}q_i^0(\xo,\ue)\bvec{v}_{i,\xo}, \textrm{ for all } \ue, \\
 \diffusion(\xo) &= \sum_{\bvec{v}_{i,\xo} \in M(\xo)}q_i^1(\xo)\bvec{v}_{i,\xo}\bvec{v}^{\prime}_{i,\xo}, \\
 \sum_{\bvec{v}_{i,\xo} \in M(\xo)}q_i^1(\xo)\bvec{v}_{i,\xo} &= 0, \\
 hq_i^0(\xo,\ue) + q_i^1(\xo) &\geq 0, \\
 q_i^1(\xo) &> 0,
\end{align*}
where the third condition guarantees that $q_i^1$ only contribute to the variance of the chain and not the mean, and the fourth and fifth conditions guarantee non-negative transition probabilities, which we verify momentarily. 

After finding $q_i^1$ and $q_i^0$ that satisfy these conditions, we are ready to define the approximating MDP $\mdp^{h} = (\statespace^{h},\controlspace, \allowbreak\ptrans{h},\allowbreak\stagecost{h},\termcost{h})$. 
First, define the normalizing constant
\begin{equation*}
\normP^h(\xo,\ue) = \sum_{\bvec{v}_{i,\xo} \in M(\xo)}\left[ hq_i^0(\xo,\ue) + q_i^1(\xo)\right]
\end{equation*}
Then, the discrete MDP is defined by the state space $\statespace^{h_{\ell}}$, the control space $\controlspace$, the transition probabilities
\begin{equation*}
\ptrans{h}(\xo,\xo+h\bvec{v}_{i,\xo}|\ue) = \frac{hq_i^0(\xo,\ue) + q_i^1(\xo)}{\normP^h(\xo,\ue)},
\end{equation*}
and the stage costs
\begin{equation*}
\stagecost{h}(\xo,\ue) = \frac{\Delta t^h(\xo,\ue)}{\normP^h(\xo,\ue)}\stagecost{}(\xo,\ue).
\end{equation*}
The discount factor is
\begin{equation*}
\gamma = \exp(-\beta \Delta t^h ).
\end{equation*}
Finally, define the interpolation interval
\begin{equation*}
\Delta t^h(\xo,\ue) = \frac{h^2}{\normP^h(\xo,\ue)}.
\end{equation*}
These conditions satisfy local consistency~\cite{Kushner2001}.

\subsubsection{Upwind differencing}
One realization of the framework described above is generated based on upwind differencing.  This procedure tries to ``push'' the current state of the system in the direction of the drift dynamics on average, and we describe this method here and use it for all of the numerical examples in Section~\ref{sec:numexamples}. 

\rev{The upwind discretization, for a two-dimensional state space, is given by
\begin{align*}
q_1^0(\xo,\ue) &= \frac{h}{h_1}\drift[1](\xo,\ue)^{-}, \quad q_0^{1} = \left(\frac{h}{h_1}\right)^2\frac{\mvf{A}[1,1](\xo)^2}{2}, \\
q_2^0(\xo,\ue) &= \frac{h}{h_1}\drift[1](\xo,\ue)^{+}, \quad q_1^{1} = \left(\frac{h}{h_1}\right)^2\frac{\mvf{A}[1,1](\xo)^2}{2}, \\
q_3^0(\xo,\ue) &= \frac{h}{h_2}\drift[2](\xo,\ue)^{-}, \quad q_2^{1} = \left(\frac{h}{h_2}\right)^2\frac{\mvf{A}[2,2](\xo)^2}{2},  \\
q_4^0(\xo,\ue) &= \frac{h}{h_2}\drift[2](\xo,\ue)^{+}, \quad q_3^{1} = \left(\frac{h}{h_2}\right)^2\frac{\mvf{A}[2,2](\xo)^2}{2}, 
\end{align*}
where the state is discretized with a spacing of $h_1$ in the first dimension and $h_2$ in the second dimension, and $h = \min(h_1,h_2)$. The sample discretization is shown in Figure~\ref{fig:upwind}, where the transition directions are aligned with the coordinate axis. This alignment results in $2d$ neighbors for every node, thus not incurring an exponential growth with dimension. Verification that such a probability assignment satisfies local consistancy can be found in~\cite{Kushner2001}.

}

\begin{figure}
\begin{center}
\includegraphics[scale=1.0,clip=true,trim=0 0 0 0]{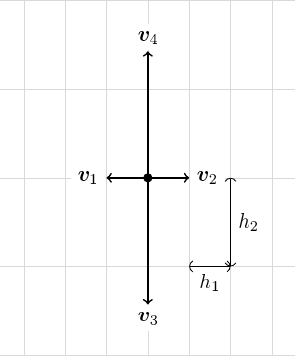}
\caption{Sample discretization of a two-dimensional state space.}
\label{fig:upwind}
\end{center}
\end{figure}

We can also analyze the computational cost of this upwind differencing procedure. The computation of the transition probabilities, for some state $\xo$ and control $\ue$, requires the evaluation of the drift and diffusion. Suppose that this evaluation requires $n_{\textrm{op}}$ operations\footnote{\rev{The number of operations $n_{\textrm{op}}$ to evaluate the drift and diffusion is usually quadratic (and at most polynomial) in the dimension of the state space, and this complexity applies to all the examples in Section~\ref{sec:numexamples}. More specifically, the number of operations required for evaluating each output of the drift typically scales as $\mathcal{O}(d)$, and therefore the full drift vector requires $\mathcal{O}(d^2)$ operations.} }. Assembling each $q_i^0(\xo,\ue)$ and $q_i^1$ requires two operations: multiplication and division. Since there are $2d$ neighbors for each $\xo$, the evaluation of all of them requires $4d$ operations. Next, the computation of the normalization $\normP^h$ involves summing all of the $q_i^0$ and $q_i^1$, a procedure requiring $4d$ operations. Computing the interpolation interval requires a single division, computing the discrete stage cost requires a division and multiplication, and computing the discount factor requires exponentiation. 
Together, these operations mean that the computational complexity of discretizing the SOC for some state $\xo$ and control $\ue$ using upwind differencing is linear with dimension
\begin{equation}
\mathcal{O}(n_{\textrm{op}} + d). \label{eq:transassemble}
\end{equation}

\subsubsection{Boundary conditions}\label{sec:mcabound}
The discretization methods described in the previous section apply to the interior nodes of the state space. In order to numerically solve optimal stochastic control problems, however, one typically needs to limit the state space to a particular region. In order to utilize low-rank tensor based methods in high dimensions, we design $\ssos$ to be a hypercube. Due to this state truncation we are required to assign boundary conditions for the discrete Markov process. Three boundary conditions are commonly used: periodic, absorbing, and reflecting boundary conditions.

A {\em periodic boundary condition} maps one side of the domain to the other. For example consider $\ssos = (-1,1)^2$. Then, if we define a periodic boundary condition for the first dimension, we mean that $\xo=(-1,\cdot)$ and $\xoo = (1,\cdot)$ are equivalent states. 

An {\em absorbing boundary condition} dictates that if the Markov process enters $\partial \ssos$ at the exit time $\tau$, then the process terminates and terminal costs are incurred.

A {\em reflecting boundary condition} is often imposed when one does not want to end the process at the boundary and periodic boundaries are not appropriate. In this case, the stochastic process is modeled with a jump diffusion. The jump diffusion term is responsible for keeping the process within $\ssos$. In our case, we will assume that the jump diffusion term instantaneously ``reflects'' the process using an orthogonal projection back into the state space $\ssos$. For example, if the system state is $\xo \in \ssos$ and the Markov process transitions to $\xoo = \xo + h \bvec{e}_k$ such that $\xoo \in \partial \ssos$, then the system immediately returns to the state $\xo$. Therefore, we can eliminate $\xoo$ from the discretized state space and adjust the self transition probability to be
\begin{equation*}
\ptrans{h}(\xo,\xo|\ue) \leftarrow \ptrans{h}(\xo,\xo|\ue) + \ptrans{h}(\xo,\xoo|\ue).
\end{equation*}
In other words, the probability of self transitioning is increased by the probability of transitioning to the boundary.

\subsection{Value iteration, policy iteration, and multilevel methods}\label{sec:vipiml}

In this section, we describe algorithms for solving the discounted-cost infinite-horizon MDP given by Equation~\eqref{eq:dcih}. In particular, we describe the value iteration (VI) algorithm and the policy iteration (PI) algorithm. Then, we describe a multi-level algorithm that is able to use coarse-grid solutions to generate solutions of fine-grid problems. FT-based versions of these algorithms will then be described in Section~\ref{sec:lowrankdp}.

\subsubsection{DP equations}
Let $\rspace{h}$ be the set of real-valued functions $\polval{h}:\ss^h \to \reals.$  Define the functional $\polrhs{h}:\ss^h \times \cs \times \rspace{h}$ as
\begin{equation}\label{eq:polrhs}
\polrhs{h}(\xo,\ue,\polval{h}) := \stagecost{h}(\xo,\ue) + \gamma \sum_{\xoo \in \ssd{}{h}} \ptrans{h}(\xo,\xoo|\ue) \polval{h}(\xoo).
\end{equation}
For a given policy $\mu$, define operator $\polfp{h}:\rspace{h} \to \rspace{h}$ as
\begin{equation*}
\polfp{h}(\polval{h})(\xo) :=  \polrhs{h}(\xo,\mu(\xo),\polval{h}), \quad \forall \xo \in \ssd{}{h} , \polval{h} \in \rspace{h} \label{eq:op}
\end{equation*}
Define the mapping $\polfpopt{h}:\rspace{h} \to \rspace{h}$, which corresponds to the Bellman equation given by Equation~\eqref{eq:dcih}, as %
\begin{equation*}
\polfpopt{h}(\polval{h})(\xo) := \min_{\ue \in \cs} \polrhs{h}(\xo,\ue,\polval{h}), \quad \forall \xo \in \ssd{}{h}, \polval{h} \in \rspace{h}. \label{eq:optop}
\end{equation*}
Using these operators we can denote two important fixed-point equations. The first describes the value function $\polval{h}$ that corresponds to a fixed policy $\mu$
\begin{equation}\label{eq:polval}
\polval{h} = \polfp{h}(\polval{h}).
\end{equation}
The second equation describes the optimal value function $\val{h}$
\begin{equation}\label{eq:optval}
\val{h} = \polfpopt{h}(\val{h})
\end{equation}
These equations are known as the dynamic programming equations.

\subsubsection{Assumptions for convergence}
Three assumptions are required to guarantee existence and uniqueness of the solution to the dynamic programming equations and to validate the convergence of their associated solution algorithms.
\begin{assumption}[Assumption A1.1 by Kushner and Dupuis~\cite{Kushner2001}]\label{as:continuity}
The functions $\ptrans{h}(\xo,\xoo|\ue)$ and $\stagecost{h}(\xo,\ue)$ are continuous functions of $\ue$ for all $\xo,\xoo \in \ssd{}{h}$.
\end{assumption}
The second assumption involves contraction. 
\begin{definition}[Contraction]
Let $\yspace{}$ be a normed vector space with the norm $\lVert \cdot \rVert$. A function $f:\yspace{} \to \yspace{}$ is a \textit{contraction mapping} if for some $\gamma \in (0,1)$ we have
\begin{equation*}
\lVert f(y) - f(y^{\prime}) \rVert \leq \gamma \lVert y - y^{\prime} \rVert, \quad \forall y,y^{\prime} \in \yspace{}.
\end{equation*}
\end{definition}
\begin{assumption}[Assumption A1.2 by Kushner and Dupuis~\cite{Kushner2001}]\label{as:contraction}
(i) There is at least one admissible feedback policy $\mu$ such that $\polfp{h}$ is a contraction, and the infima of the costs over all admissible policies is bounded from below. (ii) $\polfp{h}$ is a contraction for any feedback policy for which the associated cost is bounded.
\end{assumption}
The third assumption involves the repeated application $\polfp{h}$.
\begin{assumption}[Assumption A1.3 by Kushner and Dupuis~\cite{Kushner2001}]\label{as:convergence}
Let $\pmat_{\mu} = \{ \ptrans{h}\left(\xo,\xoo| \mu\left(\xo\right)\right) : \xo,\xoo \in \ssd{}{h} \}$ be the matrix formed by the transition probabilities of the discrete-state MDP for a fixed policy $\mu$.
If the value functions associated with the use of policies $\mu_1, \ldots, \mu_n, \ldots $ in sequence, is bounded, then 
$$\lim_{n\to \infty} \pmat_{\mu_1} \pmat_{\mu_2} \cdots \pmat_{\mu_n}  = 0.$$
\end{assumption}

\subsubsection{Value iteration algorithm}\label{sec:vi}
The VI algorithm is a fixed-point (FP) iteration aimed at computing the optimal value function $\val{h}$. It works by starting with an initial guess $\vali{h}{0} \in \rspace{h}$ and defining a sequence of value functions $\{\vali{h}{k}\}$ through the iteration $\vali{h}{k+1} = \polfpopt{h}(\vali{h}{k})$. Theorem~\ref{th:jacobi} guarantees the convergence of this algorithm under certain conditions.

\begin{theorem}[Jacobi iteration, Theorem 6.2.2 by Kushner and Dupuis~\cite{Kushner2001}]\label{th:jacobi}
Let $\mu$ be an admissible policy such that $\polfp{h}$ is a contraction. Then for any initial vector $\polvali{h}{0} \in \rspace{h}$, the sequence $\polvali{h}{k}$ defined by 
\begin{equation}\label{eq:valuefp}
\polvali{h}{k+1} = \polfp{h}(\polvali{h}{k})
\end{equation}
converges to $\polval{h}$, the unique solution to Equation~\eqref{eq:polval}. Assume Assumptions~\ref{as:continuity},~\ref{as:contraction}, and~\ref{as:convergence}. Then for any vector $\vali{h}{0} \in \rspace{h}$, the sequence recursively defined by 
\begin{equation}\label{eq:vi}
 \vali{h}{k+1} = \polfpopt{h}(\vali{h}{k})
\end{equation}
converges to the optimal value function $\val{h}$, the unique solution to Equation~\eqref{eq:optval}.
\end{theorem}
Indeed, Equation~\eqref{eq:vi} is the fixed-point iteration that is the value iteration algorithm. In Section~\ref{sec:ftvi}, we will describe an FT-based version of this algorithm.

\subsubsection{Policy iteration algorithm}\label{sec:pi}

Policy iteration (PI) is another method for solving discrete MDPs. Roughly, it is analogous to a gradient descent method, and our experiments indicate that it generally converges faster than the VI algorithm. The basic idea is to start with a Markov policy $\mu_0$ and to generate a sequence of policies $\{\mu_k\}$ according to 
\begin{equation}\label{eq:pi}
\mu_{k} = \arg \min_{\mu} \left[ \polfp{h}(\polvali{h}{k-1})\right]
\end{equation}
The resulting value functions $\{\polvali{h}{k}\}$ are solutions of Equation~\eqref{eq:polval}, \ie
\begin{equation}
\polvali{h}{k} = \polfpi{h}{k}(\polvali{h}{k}).
\label{eq:pifp}
\end{equation}
Theorem~\ref{th:pi} provides the conditions under which this iteration converges.
\begin{theorem}[Policy iteration, Theorem 6.2.1 by Kushner and Dupuis~\cite{Kushner2001}]\label{th:pi}
Assume Assumptions~\ref{as:continuity} and~\ref{as:contraction}. Then there is a unique solution to Equation~\eqref{eq:optval}, and it is the \revv{infimum} of the value functions over all time independent feedback policies. Let $\mu_0$ be an admissible feedback policy such that the corresponding value function $\polvali{h}{0}$ is bounded. For $k \geq 1$, define the sequence of feedback policies $\mu_k$ and costs $\polvali{h}{k}$ recursively by Equation~\eqref{eq:pi} and Equation~\eqref{eq:pifp}. Then $\polvali{h}{k} \to \val{h}.$
Under the additional condition given by Assumption~\ref{as:convergence}, $\val{h}$ is the \revv{infimum} of the value functions over all admissible policies.
\end{theorem}

Note that policy iteration requires the solution of a linear system in Equation~\eqref{eq:pifp}. Furthermore, Theorem~\ref{th:jacobi} states that since $\polfp{h}$ is a contraction mapping that a fixed-point iteration can also be used to solve this system. This property leads to a modification of the policy iteration algorithm called {\em optimistic policy iteration}~\cite{Bertsekas2013}, which substitutes $n_{fp}$ steps of fixed-point iterations for solving~\eqref{eq:pifp} to create a computationally more efficient algorithm. The assumptions necessary for convergence of optimistic PI are the same as those for PI and VI. Details are given by Bertsekas~\cite{Bertsekas2013}. In Section~\ref{sec:ftpi}, we will describe an FT-based version of optimistic PI algorithm. 

\subsubsection{Multilevel algorithms}\label{sec:multigrid}

 Multigrid techniques~\cite{Briggs2000,Trottenberg2000} have been successful at obtaining solutions to many problems by exploiting multiscale structure of the problem. For example, they are able to leverage solutions of linear systems at coarse discretization levels for solving finely discretized systems.

We describe how to apply these ideas within dynamic programming framework for two purposes: the initialization of fine-grid solutions with coarse-grid solutions and for the solution of the linear system in Equation~\eqref{eq:pifp} within policy iteration. Our experiments indicate that fine-grid problems typically require more iterations to converge, and initialization with a coarse-grid solution offers dramatic computational gains. Since we expect the solution to converge to the continuous solution as the grid is refined, we expect the number of iterations required for convergence to decrease as the grid is refined.

The simplest multi-level algorithm is the one-way discretization algorithm that sequentially refines coarse-grid solutions of Equation~\eqref{eq:optval} by searching for solutions on a grid starting from an initial guess obtained from the solution of a coarser problem. This procedure was analyzed in detail for shortest path or MDP style problems by Chow and Tsitsiklis in~\cite{Chow1991}. The pseudocode for this algorithm provided in Algorithm~\ref{alg:coarsetofine}. In Algorithm~\ref{alg:coarsetofine}, a set of $\kappa$ discretization levels $\{h_{1},h_{2},\ldots, h_{\kappa}\}$ such that for $j > i$ we have $h_{j} > h_{i}$, are specified. Furthermore, the operator $\ctof{k}$ interpolates the solution of the $h_{k+1}$ grid onto the $h_{k}$ grid. Then a sequence of problems, starting with the coarsest, are solved until the fine-grid solution is obtained.

\begin{algorithm}
  \caption{One-way multigrid~\cite{Chow1991}}
  \label{alg:coarsetofine}                                     
  \begin{algorithmic}[1]                                                  
    \REQUIRE Set of discretization levels $\{h_{1},h_{2},\ldots,h_{\kappa}\}$; Initial cost function $\val{h_{\kappa}}_0$
    \STATE $\val{h_{\kappa}} \leftarrow $ Solve Equation~\eqref{eq:optval} starting from $\val{h_{\kappa}}_{0}$
    \FOR {$k = \kappa-1 \ldots 1$}
    \STATE $\val{h_{k}}_0 = \ctof{k}\val{h_{k+1}}$
    \STATE $\val{h_k}\leftarrow $ Solve Equation~\eqref{eq:optval} starting from $\val{h_k}_{0}$
    \ENDFOR
  \end{algorithmic}                                                        
\end{algorithm}

Multigrid techniques can also be used for solving the linear system in Equation~\eqref{eq:polval} within the context of policy iteration. \rev{For simplicity of presentation, in the rest of this section we consider two levels of discretization with $h_1 = h$ and $h_2 = 2h$.}  Recall that for a fixed policy $\mu$, this system can be equivalently written using a linear operator $\pollin{h}$ defined according to
\begin{equation*}
\resizebox{1\hsize}{!}
{$\pollin{h}(\polval{h})(\xo) \equiv \polval{h}(\xo) - \gamma \sum_{\xoo \in \ssd{}{h}}\ptrans{h}\left(\xo,\xoo|\mu\left(\xo\right)\right)\polval{h}(\xoo), \quad \forall \xo \in \ssd{}{h}.$}
\end{equation*}
Therefore, for a fixed policy $\mu$ the corresponding value function satisfies
\begin{equation}\label{eq:pollin}
\pollin{h}(\polval{h}) = \stagecost{h}.
\end{equation}
Typically, we do not expect to satisfy this equation exactly, rather we will have an approximation $\polvala{h}$ that yields a non-zero residual
\begin{equation}\label{eq:resid}
r^h = \stagecost{h} - \pollin{h}(\polvala{h}).
\end{equation}
In addition to the residual, we can define the difference between the approximation and the true minimum as $\polerr{h} = \polval{h} - \polvala{h}$. 
Since, $\pollin{h}$ is a linear operator, we can replace $\polvala{h}$ in Equation~\eqref{eq:resid} to obtain
\begin{align*}
r^h &= \stagecost{h} -\pollin{h}(\polval{h}-\polerr{h}) \\
    &= \stagecost{h} - \pollin{h}(\polval{h}) + \pollin{h}(\polerr{h})\\
    &= \pollin{h}(\polerr{h})
\end{align*}
Thus, if solve for $\polerr{h}$, then we can update $\polvala{h}$ to obtain the solution
\begin{equation*}
\polval{h} = \polerr{h} + \polvala{h}.
\end{equation*}
In order for multigrid to be a successful strategy, we typically assume that the residual $r^h$ is ``smooth,'' and therefore we can potentially solve for $\polerr{h}$ on a coarser grid. 
The coarse grid residual is
\begin{align*}
r^{2h} &= \pollin{2h}(\polerr{2h}) 
\end{align*}
where we now choose the residual to be the restriction, denoted by operator $I_{h}^{2h}$, of the fine-grid residual
\begin{equation*}
r^{2h} = I_{h}^{2h}r^h.
\end{equation*}
Combining these two equations we obtain an equation for $\polerr{2h}$
\begin{equation}\label{eq:coarsegrid}
\pollin{2h}(\polerr{2h}) = I_{h}^{2h}r^h
\end{equation}
Note that the relationship between the linear operators $\polfp{2h}$ and $\pollin{2h}$  displayed by Equations~\eqref{eq:polval} and~\eqref{eq:pollin} lead to an equivalent equation for the error given by
\begin{equation}\label{eq:mgfp}
\polerr{2h} = \polfp{2h}(\polerr{2h};r^{2h}),
\end{equation}
where we specifically denote that the stage cost is replaced by $r^{2h}$.
Since $\polfp{2h}$ is a contraction mapping we can use the fixed-point iteration in Equation~\eqref{eq:valuefp} to solve this equation.

After solving the system in Equation~\eqref{eq:coarsegrid} above we can obtain the correction at the fine grid 
\begin{equation}
\polvala{h} \leftarrow \polvala{h} + I_{2h}^h\polerr{2h}
\end{equation}
In order, to obtain smooth out the high frequency components of the residual $r^h$ one must perform ``smoothing'' iteration instead of the typical iteration $\polfp{h}$. These iterations are typically Gauss-Seidel relaxations or weighted Jacobi iterations. Suppose that we start with $\polvala{h}_{k}$, then using the weighted Jacobi iteration we obtain an update $\polvala{h}_{k+1}$ through the following two equations
\begin{align}
\polvalt{h} &= \polfp{h}(\polvala{h}_{k},\stagecost{h}), \nonumber \\
\polvala{h}_{k+1} &= \omega \polvalt{h} + (1-\omega) \polvala{h}_{k}, \nonumber
\end{align}
for $\omega > 1$. To shorten notation, we will denote these equations by the operator $\polfpw{h}$ such that
\begin{equation*}\label{eq:jacobiweighted}
\polvala{h}_{k+1} = \polfpw{h}(\polvala{h}_{k},\stagecost{h}).
\end{equation*}

We have chosen to use the weighted Jacobi iteration since it can be performed by treating the linear operator $\polfp{h}$ as a black-box FP iteration, \ie the algorithm takes as input a value function and outputs another value function. Thus, we can still wrap the low-rank approximation scheme around this operator. A relaxed Gauss-Seidel relaxation would require sequentially updating elements of $\polvala{h}_k$, and then using these updated elements for other element updates. Details on the reasons for these smoothing iterations within multigrid is available in the existing literature~\cite{Briggs2000,Trottenberg2000,Kushner2001}.

Combining all of these notions we can design many multigrid methods. We will introduce an FT-based version of the two-level V-grid in Algorithm~\ref{alg:Vgrid} in the next section. 

\section{Low-rank compression of functions}\label{sec:lowrankfunc}
The Markov chain approximation method is often computationally intractable for problems instances with state spaces embedded in more than a few dimensions. This curse of dimensionality, or exponential growth in storage and computation complexity, arises due to state-space discretization. To mitigate the curse of dimensionality, we believe that algorithms for solving general stochastic optimal control problems must be able to
\begin{enumerate}
  \item exploit function structure to perform compression with polynomial time complexity, and
  \item perform multilinear algebra with functions in compressed format in polynomial time.
\end{enumerate}
The first capability ensures that value functions can be \textit{represented} on computing hardware. \rev{The second ensures that the computational operations required by dynamic programming algorithms can be performed in polynomial time.}

In this section, we describe a method for ``low-rank'' representation of multivariate functions, and algorithms that allow us to perform multilinear algebra in this representation. 
The algorithms presented in this section were introduced in earlier work by the authors~\cite{Gorodetsky2015a}. In that paper, algorithms for approximating multivariate black box functions and performing computations with the resulting approximation are described. We review low-rank function representations in Section~\ref{sec:seprep} and the approximation algorithms in Section~\ref{sec:crossround}.

\subsection{Low-rank function representations}\label{sec:seprep}

Let $\xspace{}$ be a tensor product of closed intervals $\xspace{} := [a_1,b_1] \times [a_2,b_2] \times \cdots \times [a_d,b_d]$, with $a_i,b_i \in \reals$ and $a_i < b_i$ for $i = 1,\ldots,d$. 
\rev{A {\em low-rank} %
 function $f:\xspace{} \to \reals$ is, in a general sense, one that exhibits some degree of separability amongst input dimensions. This separability means it can be written in a factored form as small sum of products of univariate functions. While the definitions of rank and the types of factorizations change for different types of tensor network structures, they all generally exploit additive and multiplicative separability. 

One particular low-rank representation of a function uses a sum of the outer products of a set of univariate functions $f_j^{(i)}:[a_i,b_i] \to \reals$, i.e., }
\begin{equation*}
f(x_1,\ldots,x_d) \approx \sum_{j=1}^R f^{(1)}_j(x_1)\ldots f^{(d)}_j(x_d).
\end{equation*}
This approximation is called the canonical polyadic (CP) decomposition~\cite{Carroll1970}. The CP format is defined by sets of univariate functions  ${\cal F}_i^{CP} = \{ f_1^{(i)}(x_i),\ldots,f_R^{(i)}(x_i) \}$ for $i = 1, \ldots d$. The storage complexity of the CP format is clearly linear with dimension, but also depends on the storage complexity of each $f_j^{(i)}.$ For instance, if each input dimension is discretized into $n$ grid points, so that each univariate function is represented by $n$ values, then the storage complexity of the CP format is $\mathcal{O}(dnR)$. This complexity is linear with dimension, linear with discretization level, and linear with rank $R$. Thus, for a class of functions whose ranks grow polynomially with dimension, \ie $R(d) = \mathcal{O}(d^p)$ for some $p \in \naturals$, polynomial storage complexity is attained. 

Contrast this with the representation of $f(x_1, \dots, x_d)$ as a lookup table. If the lookup table is obtained by discretizing each input variable into $n$ points, the storage requirement is $O(n^d)$, which grows exponentially with dimension. 

Regardless of the representation of each $f_j^{(i)}$, the complexity of this representation is always \textit{linear} with dimension. Intuitively, for approximately separable functions, storing many univariate functions requires fewer resources than storing a multivariate function. In the context of stochastic optimal control, the CP format has been used for the special case when the control is unconstrained and the dynamics are affine with control input~\cite{Horowitz2014}.

Using the canonical decomposition can be problematic %
in practice because the problem of determining the canonical rank of a discretized function, or tensor, is NP complete~\cite{Kruskal1989,Hastad1990}, and the problem of finding the best approximation in Frobenius norm for a given rank can be ill-posed~\cite{DeSilva2008}. 
Instead of the canonical decomposition, we propose using a continuous variant of the \emph{tensor-train} (TT) decomposition~\cite{Oseledets2010,Oseledets2011,Gorodetsky2015a} called the functional tensor-train, or \emph{function-train} (FT)~\cite{Gorodetsky2015a}. In these formats, the best fixed-rank approximation problem is well posed, and the approximation can be computed using a sequence of matrix factorizations.
 
A multivariate function $f:\xspace{} \to \reals$ in FT format is represented as
\begin{equation}
  \resizebox{0.95\hsize}{!}{$f(x_1, \ldots, x_d) =   
	\displaystyle{\sum_{\alpha_0=1}^{r_0} 
	\ldots\sum_{\alpha_d=1}^{r_d}  \ffiber{\alpha_0,\alpha_1}{1}(x_1)  
	\ldots \ffiber{\alpha_{d-1},\alpha_d}{d}(x_d)},$} \label{eq:ft}
\end{equation}
where $r_i \in \integers_{+}$ are the FT \textit{ranks}, with $r_0=r_d=1$. For each input coordinate $i = 1, \ldots, d$, the set of univariate functions $\mathcal{F}_i =\{f^{(i)}_{(\alpha_{i-1}, \alpha_{i})}\! \! : [a_i, b_i] \to \mathbb{R}, \  \alpha_{i-1} \in \{ 1, \ldots, r_{i-1} \},\ \alpha_{i} \in \{ 1, \ldots, r_{i} \} \}$ are called \textit{cores}. %
Each set can be viewed as a matrix-valued function and visualized as an array of univariate functions:
\begin{equation*}
\core{F}_{i}(x_i) = \left[
  \begin{array}{ccc}
    \ffiber{1,1}{i}(x_i) & \cdots & \ffiber{1,r_{i}}{i}(x_i) \\
    \vdots          & \ddots & \vdots \\
    \ffiber{r_{i-1},r_{i}}{i}(x_i) & \cdots & \ffiber{r_{i-1},r_i}{k}(x_i)
  \end{array} 
\right].
\end{equation*}
Thus the evaluation of a function in FT format may equivalently be posed as a sequence of $d-1$ vector-matrix products:
\begin{equation}\label{eq:ftmat}
f(x_1,\ldots,x_d) = \core{F}_1(x_1)\ldots\core{F}_d(x_d).
\end{equation}

The tensor-train decomposition differs from the canonical tensor decomposition by allowing a greater variety of interactions between neighboring dimensions through products of univariate functions in neighboring cores $\mathcal{F}_i, \mathcal{F}_{i+1}$. Furthermore, each of the cores contains $r_{i-1} \times r_{i}$ univariate functions instead of a fixed number $R$ for each $\mathcal{F}_i^{CP}$ within the CP format. 

The ranks of the FT decomposition of a function $f$ can be bounded by the singular value decomposition (SVD) ranks of certain separated representations of $f$. 
Let $\xspace{\leq i} := [a_1,b_1] \times [a_2,b_2] \times \cdots \times [a_i,b_i]$ and  $\xspace{>i} := [a_{i+1},b_{i+1}] \times \cdots \times [a_d,b_d]$, such that $\xspace{} = \xspace{\leq i} \times \xspace{>i}$. We then let $f^i$ denote the $i$--separated representation of the function $f$, also called the $i$th \emph{unfolding} of $f$:
\begin{align*}
f^i:\xspace{\leq i} \times \xspace{>i} \to \reals,  & \textrm{ where } \\
 f^i(\{x_1,\ldots, x_i\},& \{x_{i+1},\ldots,x_{d}\}) = f(x_1,\ldots,x_d).  \nonumber
\end{align*}
The FT ranks of $f$ are related to the SVD ranks of $f^i$ via the following result.
\begin{theorem}[Ranks of approximately low-rank unfoldings, from Theorem 4.2 in~\cite{Gorodetsky2015a}]\label{th:approx}
Suppose that the unfoldings $\funfold{i}$ of the function $f$ satisfy\footnote{The rank condition on $g^i$ is based on the \textit{functional} SVD, or Schmidt decomposition, a continuous analogue of the discrete SVD.}, for all $i = 1,\ldots,d-1$,
\begin{equation*}
\funfold{i} = g^i + e^i, \quad \rank g^i = r_i, \quad \left| \left| e^i \right| \right|_{L^2} = \alr_i.
\end{equation*}
Then a rank $\bvec{r} = [r_0,r_1,\ldots,r_d]$  approximation $\FT$ of $f$ in FT format may obtained with bounded error\footnote{In this paper, integrals $\int f dx$ will always be with respect to the Lebesgue measure. For example, $\int f(x_i) dx_i$ should be understood as $\int f(x_i) \mu(dx_i)$, and similarly $\int f(x) dx = \int f(x) \mu(dx)$.}
\begin{equation*}
{\int \left(f-\FT\right)^2 dx} \leq { \sum_{i=1}^{d-1} \alr_i^2}.
\end{equation*}
\end{theorem}

The ranks of an FT approximation of $f$ can also be related to the Sobolev regularity of $f$, with smoother functions having faster approximation rates in FT format; precise results are given in \cite{Bigoni2016}. These regularity conditions are sufficient but not necessary, however; according to Theorem~\ref{th:approx} even discontinuous functions can have low FT ranks, if their associated unfoldings exhibit low rank structure.

\subsection{Cross approximation and rounding}\label{sec:crossround}

The proof of Theorem~\ref{th:approx} is constructive and results in an algorithm that allows one to decompose a function into its FT representation using a sequence of SVDs. While this algorithm, referred to as \texttt{TT-SVD}~\cite{Oseledets2011}, exhibits certain optimality properties, it encounters the curse of dimensionality associated with computing the SVD of each unfolding $f^i$. 

To remedy this problem, Oseledets~\cite{Oseledets2010} replaces the SVD with a cross approximation algorithm for computing CUR/skeleton decompositions~\cite{Goreinov1997,Tyrtyshnikov2000,Mahoney2009} of matrices, within the overall context of compressing a tensor into TT format. Similarly, our previous work~\cite{Gorodetsky2015a} employs a continuous version of cross approximation to compress a function into FT format. If each unfolding function $f^i$ has a finite SVD rank, then these cross approximation algorithms can yield exact reconstructions.

The resulting algorithm only requires the evaluation of univariate function fibers. Fibers are multidimensional analogues of matrix rows and columns, and they are obtained by fixing all dimensions except one~\cite{Kolda2009}. Each FT core can be viewed as a collection of fibers of the corresponding dimension, and the cross approximation algorithm only requires \rev{$\mathcal{O}(dnr^2)$} evaluations of $f$, where $n$ represents the number of parameters used to represent each univariate function in each core \rev{and $r \geq r_i$ for $i=0,\ldots,d$ is an upper bound on all the unfolding ranks.}

\rev{
  More specifically, continuous cross-approximation chooses a basis for each dimension of a multivariate function $f$ using certain univariate fibers. It does so by sweeping across each input dimension and approximating a set of fibers that are represented as univariate functions. Consider the first step of a left-right sweep. Let $x_{>1}^{(\alpha_1)} \in \xspace{>1}$ for $\alpha_1 = 1, \ldots, r_1$ be the fixed values that define univariate functions of $x_1$; then the fibers are defined as
\begin{equation*}
  \ffiber{\alpha_1}{1}(x_1) = f(x_1,x_{>1}^{(\alpha_1)}).
\end{equation*}
Recall that these functions then form the core $\fcore{1}$\footnote{In practice, stability is enhanced through a second step of orthonormalizing these functions using a \textit{continuous} QR decomposition using Householder reflectors to obtain an orthonormal basis; see~\cite{Gorodetsky2015a}}. Similarly, during step $k$, the fibers are obtained by choosing $x_{<k}^{(\alpha_{k-1})} \in \xspace{<k}$ for $\alpha_{k-1} = 1,\ldots,r_{k-1}$ and $x_{>k}^{(\alpha_{k})} \in \xspace{>k}$ for $\alpha_{k} = 1,\ldots,r_{k}$ to form
\begin{equation*}
  \ffiber{\alpha_{k-1},\alpha_{k}}{k}(x_k) = f(x_{<k}^{(\alpha_{k-1})},x_k,x_{>k}^{(\alpha_k)}).
\end{equation*}
The fixed values $x_{<k}^{(\alpha_{k-1})}$ and $x_{>k}^{(\alpha_{k})}$ are obtained using a continuous maximum volume scheme~\cite{Gorodetsky2015a}. These fibers are then adaptively approximated using any linear or nonlinear approximation scheme. Since they are only univariate functions, the computational cost is low.

Here we see a clear difference between the discrete tensor-train used in~\cite{Gorodetsky2015b} and the continuous analogue. In the TT, each of these fibers is actually a vector of function evaluations, and no additional information is available. In the continuous version, \textit{additional structure} can be exploited to obtain more compact representations of the function. For example, if the function is constant, then it can be stored using only one parameter in the FT format, rather than as a constant vector in the TT format. If the discretization from MCA is very fine, i.e., we can evaluate the value function on a fine grid, then we can obtain significant benefits from the FT format by compressing the representation and avoiding the storage of large vectors.
}

While low computational costs make this algorithm attractive, there are two downsides. First, there are no convergence guarantees for the cross approximation algorithm when the unfolding functions are not of finite rank. Second, even for finite-rank tensors, the algorithm requires the specification of upper bounds on the ranks. If the upper bounds are set too low, then errors occur in the approximation; if the upper bounds are set too high, however, then too many function evaluations are required.

We mitigate the second downside, the specification of ranks, using a rank adaptation scheme. The simplest adaptation scheme, and the one we use for the experiments in this paper, is based on TT-rounding. The idea of TT-rounding~\cite{Oseledets2010} is to approximate a tensor in TT format $\tensor{T}$ by another tensor in TT format $\hat{\tensor{T}}$ to a tolerance $\epsilon$, i.e.,
\[
\lVert \tensor{T} - \hat{\tensor{T}} \rVert \leq \epsilon \lVert \tensor{T} \rVert.
\]
The benefit of such an approximation is that a tensor in TT format can often be well \textit{approximated} by another tensor with lower ranks if one allows for a relative error $\epsilon$. Furthermore, performing the rounding operation requires $\mathcal{O}(nr^3)$ operations, where $r$ refers to the rank of $\tensor{T}$. 

\rev{
  Rounding, by construction, guarantees a relative error tolerance of $\epsilon$ in the Frobenius norm ($L_2$ for a flattening of the tensor). In other words, a user specifies a desired accuracy and then one reduces the ranks of a tensor such that the accuracy is maintained. This is a direct multivariate analogue of performing a truncated SVD, where the $\epsilon$ is used to specify the truncation tolerance. Furthermore, due to the equivalence of norms, this accuracy can be obtained in any norm for a flattened tensor. For example,

\begin{equation*}
\lVert \tensor{T} - \hat{\tensor{T}} \rVert_{\infty} \leq \lVert \tensor{T} - \hat{\tensor{T}} \rVert_{F} \leq \epsilon \lVert \tensor{T} \rVert_{F} \leq \sqrt{N} \epsilon \lVert \tensor{T} \rVert_{\infty},
\end{equation*}
where $N$ is the number of elements in the tensor. Thus to achieve an equivalent error in the maximum norm, one needs to divide the error tolerance by $\sqrt{N}$. The maximum norm is important for dynamic programming since the Bellman operator is often proved to be a contraction with respect to this norm.
}

In our continuous context, we use a continuous analogue of TT-rounding~\cite{Gorodetsky2015a} to aid in rank adaptation. The rank adaptation algorithm we use requires the following steps:
\begin{enumerate}
\item Estimate an upper bound on each rank $r_i$
\item Use cross approximation to obtain a corresponding FT approximation $\hat{f}$
\item Perform FT-rounding~\cite{Gorodetsky2015a} with tolerance $\epsilon$ to obtain a new $\tilde{f}$
\item If ranks of $\tilde{f}$ are not smaller than the ranks of $\hat{f},$ increase the ranks and go to step 2.
\end{enumerate}
If all ranks are not rounded down, we may have under-specified the proper ranks in Step 1. In that case, we increase the upper bound estimate and retry. 

 The pseudocode for this algorithm is given in Algorithm~\ref{alg:rankadapt}. Within this algorithm, there are calls to cross approximation ($\texttt{cross-approx}$) and to rounding ($\texttt{ft-round}$). A detailed description of these algorithms can be found in~\cite{Gorodetsky2015a}. Algorithm~\ref{alg:rankadapt} requires five inputs and produces one output. The inputs are: a black-box function that is to be approximated, a cross approximation tolerance $\crossdelta$, a rank increase parameter $\texttt{kickran}$, a rounding tolerance $\roundeps$, and an initial rank estimate. The output of $\ftcross$ is a separable approximation in FT format. We abuse notation by defining $\epsilon= (\crossdelta,\roundeps$), and refer to this procedure as $\hat{f} = \ftcross(f,\epsilon).$ 

\begin{algorithm}[ht!]
    \caption{$\texttt{ft-rankadapt}$: FT approximation with rank adaptation~\cite{Gorodetsky2015a}}
\begin{algorithmic}[1]
\label{alg:rankadapt}
\REQUIRE A $d$-dimensional function $f:\xspace{} \to \reals$; \newline 
         \hspace*{22pt} Cross-approximation tolerance $\crossdelta$; \newline 
         \hspace*{22pt} Size of rank increase $\texttt{kickrank}$; \newline 
         \hspace*{22pt} Rounding accuracy $\roundeps$; \newline
         \hspace*{22pt} Initial rank estimates $\bvec{r}$
\ENSURE Approximation $\hat{f}$ such that a rank increase and rounding does not change ranks.
\STATE $\hat{f} = \texttt{cross-approx}(f,\bvec{r},\crossdelta)$
\STATE $\hat{f}_r = \texttt{ft-round}(\hat{f},\roundeps)$
\STATE $\hat{\bvec{r}} = \rank(\hat{f}_r)$
\WHILE {$\exists i \textrm{ s.t. } \hat{r}_i = r_i$} 
    \FOR{$k=1 \to d-1$}
        \STATE $r_k = \hat{r}_k + \texttt{kickrank}$
    \ENDFOR
    \STATE $\hat{f} = \texttt{cross-approx}(f,\bvec{r},\crossdelta)$
    \STATE $\hat{f}_r = \texttt{ft-round}(\hat{f},\roundeps)$
    \STATE $\hat{\bvec{r}} = \rank(\hat{f}_r)$
\ENDWHILE
\STATE $\hat{f} = \hat{f}_r$
\end{algorithmic}
\end{algorithm}

This algorithm requires $\mathcal{O}(nr^2)$ evaluations of the function $f$, and it requires $\mathcal{O}(nr^3)$ operations in total.

\rev{
\subsection{Examples of low-rank functions}

Next we provide several examples of low-rank functions, both to demonstrate how the rank is related to \textit{separability} of the inputs and to provide intuition about ranks of certain functions. We begin with two simple canonical structures.

The first comprises \textit{additively separable} functions $f(x) = f_1(x_1) + f_2(x_2) + \cdots + f_d(x_d)$.  Additive functions are extremely common within high-dimensional modeling~\cite{Hastie1990,Lie2008,Meier2009}, and they are rank-2 as seen by the following decomposition:
\begin{align*}
f(x_1,&x_2,\ldots,x_d) = \\
&\left[f_1(x_1) \  1\right] \left[ 
  \begin{array}{cc}
    1 & 0 \\
    f_2(x_2) & 1 
  \end{array} \right] \cdots 
\left[ \begin{array}{c}
         1 \\ 
         f_d(x_d) 
       \end{array}\right].
\end{align*}

The second example are quadratic functions. This particular class of functions is important to optimal control as it contain the solutions of the classical linear-quadratic regulator (LQR) control problem. Quadratic functions have ranks bounded by the dimension of the state space, i.e., $r_i \leq d+1.$ One representation of a quadratic function $$f(x_1,x_2,\ldots,x_d) = x^T\mat{A}x$$ in FT format has the following core structure
\begin{align*}
\mvf{F}_1(x_1) &=  \left[ 
\begin{array}{c@{\hspace{2em}}cccc}
a_{1,1} x_1^2 & a_{1,2}x_1  & \ldots & a_{1,d} x_1  & 1
\end{array}
\right], \\
\mvf{F}_d(x_d) &= \left[
\begin{array}{ccc}
1 &
x_d &
a_{d,d}x_d^2 
\end{array}
\right]
\end{align*}
\begin{align*}
&\mvf{F}_i(x_i) = \\
& \left[
\begin{array}{c@{\hspace{0.5em}}c@{\hspace{0.5em}}c@{\hspace{0.5em}}c@{\hspace{0.5em}}c@{\hspace{0.5em}}c@{\hspace{0.5em}}c@{\hspace{0.5em}}}
1  	     & 0      & 0      & 0      & \cdots & 0      & 0        \\ 
x_i      & 0      & 0      & 0      & \cdots & 0      & 0       \\
0 		 & 1      & 0      & 0      & \cdots      & 0      & 0     \\
0 		 & 0      & 1      & 0      & \cdots & 0      & 0       \\ 
\vdots 	 &        & \ddots & \ddots &        &        & \vdots    \\ 
\vdots 	 &        &        & \ddots & \ddots &        & \vdots      \\ 
\vdots 	 & \vdots &        &        &        & \ddots & \ddots \\ 
0 		 &  0     & 0      & 0      & \cdots & 1      & 0       \\
a_{i,i}x_i^2  & a_{i,i+1}x_i  & a_{i,i+2}x_i & \cdots & a_{i,d-1} x_i & a_{i,d} x_i & 1
\end{array}
\right].
\end{align*}
The middle cores of the quadratic are sparse. Since the FT representation stores and represents each univariate function within each core independently, algorithms can accommodate and discover such sparse structure in a routine and automatic manner.

General polynomial functions, however, have exponentially scaling rank $p^d$. This is clear because a polynomial can be written as
\begin{align*}
  f(x_1,x_2,\ldots,x_d) &= \sum_{i_1=1}^p\sum_{i_2=1}^p\ldots\sum_{i_d=1}^pa_{i_1i_2\ldots i_d} x_1^{i_1}x_2^{i_2}\ldots x_d^{i_d}.
\end{align*}
Thus, without any additional structure, $p^d$ coefficients $a_{i_1i_2\ldots i_d}$ must be stored, and low-rank approximation would not be suitable. Without exploiting some other structure of the coefficient tensor, there could be little hope for this class of problems. One important structure that often leads to low-rank representations (i.e., low-rank coefficients) in this setting is \textit{anisotropy} of the functions. For a further discussion of this structure, along with philosophy and intuition as to why certain functions are low rank, we refer the reader to~\cite{Trefethen2016}.

Finally, we emphasize that in many applications, tensors or functions of interest can be \textit{numerically} low-rank. A particularly relevant body of literature is that which seeks numerically low-rank solutions of partial differential equations (PDEs). Since the Markov chain approximation algorithm is closely related to the solution of Hamilton-Jacobi-Bellman PDEs~\cite{Kushner2001}, this literature is applicable. Depending on the drift and diffusion terms, the HJB PDE may be linear or nonlinear  and of elliptic, parabolic, hyperbolic, or some other type. Investigating low-rank solutions of elliptic~\cite{Khoromskij2011}, linear parabolic~\cite{Tobler2012}, and other PDE types~\cite{Bachmayr2016} is an active area of research where significant compression rates have been achieved.

}

\section{Low-rank compressed dynamic programming}\label{sec:lowrankdp}
In this section, we propose novel dynamic programming algorithms based on the compressed continuous computation framework described in the previous section. Specifically, we describe how to represent value functions in FT format in Section~\ref{sec:tensmca}, and how to perform FT-based versions of the value iteration, policy iteration, and multigrid algorithms in Sections~\ref{sec:ftvi}, \ref{sec:ftpi}, and \ref{sec:ftmultigrid}, respectively.

\subsection{FT representation of value functions}\label{sec:tensmca}
The Markov chain approximation method approximates a continous-space stochastic control problem with a discrete-state Markov decision process. In this framework, the continuous value functions $\polval{}$ are approximated by their discrete counterparts $\polval{h}.$ To leverage low-rank decompositions, we focus our attention on situations where the discrete value functions represent the cost of discrete MDPs defined through a tensor-product discretization of the state space, and therefore, $\polval{h}$ can be interpreted as a $d$-way array. In order to combat the curse of dimensionality associated with storing and computing $\polval{h}$, we propose using the FT decomposition.

Let $\ssd{}{} = \ssd{1}{} \times \ssd{2}{} \times \cdots \times \ssd{d}{}$ denote a tensor-product of intervals as described in Section~\ref{sec:seprep}. Recall that the space of the $i$th state variable is $\ssd{i}{} = [a_i,b_i],$ for $a_i,b_i \in \reals$ that have the property $a_i < b_i$. A \textit{tensor-product discretization} $\ssd{}{h}$ of $\ssd{}{}$ involves discretizing each dimension into $n$ nodes to form $\ssd{i}{h} \subset \ssd{i}{}$ where
\[
\ssd{i}{h} = \{ \dnode{1}{i}, \dnode{2}{i}, \ldots \dnode{n}{i} \}, \quad \textrm{ where } \dnode{k}{i} \in \ssd{i}{} \textrm{ for all $k$}.
\]
A discretized value function $\polval{h}$ can therefore be viewed as a vector with $n^d$ elements. 

The Markov chain approximation method guarantees that the solution to the discrete MDP approximates the solution to the original stochastic optimal control problem, \ie $\val{h} \approx \val{}$, for small enough discretizations. This approximation, however, is ill-defined since $\val{h}$ is a multidimensional array and $\val{}$ is a multivariate function. Furthermore, a continuous control law requires the ability to determine the optimal control for a system when it is in a state that is not included in the discretization. Therefore, it will be necessary to use the discrete value function $\polval{h} \in \rspace{h}$ to develop a value function that is a mapping from the continuous space $\ssd{}{}$ to the reals.

We will slightly abuse notation and interpret $\polval{h}$ as both a value function of the discrete MDP and as an \textit{approximation} to the value function of the continuous system. Furthermore, when generating this continuous-space approximation, we are restricted to evaluations located only within the tensor-product discretization. In this sense, we can think of $\polval{h}$ both as an array $\polval{h}:\ssd{}{h} \to \reals$ in the sense that it has elements
$$
\polval{h}[i_1,\ldots,i_d] \approx \polval{} \big(\,\dnode{1}{i_1},\ldots, \dnode{d}{i_d}\,\big),
$$
and simultaneously as a function $\polval{h}: \ssd{}{} \to \reals$ where
\begin{equation*}
\polval{h} \big(\,\dnode{1}{i_1},\ldots, \dnode{d}{i_d} \,\big) \approx \polval{} \big(\,\dnode{1}{i_1},\ldots, \dnode{d}{i_d}\,\big).
\end{equation*}

Now, recall that the FT representation of a function $f$ is given by Equation~\eqref{eq:ft} and defined by the set of FT cores $\{\mvf{F}_i\}$ for $i = 1, \ldots, d$. Each of these cores is a matrix-valued function $\core{F}_{i}: \ssd{i}{} \to \reals^{r_{i-1} \times r_{i}}$ that can be represented as a two-dimensional array of scalar-valued univariate functions:
\begin{equation*}
\core{F}_{i}(x_i) = \left[
  \begin{array}{ccc}
    \ffiber{1,1}{i}(x_i) & \cdots & \ffiber{1,r_{i}}{i}(x_i) \\
    \vdots          & \ddots & \vdots \\
    \ffiber{r_{i-1},r_{i}}{i}(x_i) & \cdots & \ffiber{r_{i-1},r_i}{i}(x_i)
  \end{array} 
\right].
\end{equation*}
Using this matrix-valued function representation for the cores, the evaluation of a function in the FT format can be expressed as 
$
f(x_1,\ldots,x_d) = \core{F}_1(x_1)\ldots\core{F}_d(x_d).
$

Since we can effectively compute $\polval{h}$ through evaluations at uniformly discretized tensor-product grids, we employ a \textit{nodal} representation of each scalar-valued univariate function %
\begin{equation}\label{eq:hatfunc}
\ffiber{\alpha_1,\alpha_2}{i,h}(x_i) = \sum_{\ell=1}^n a_{\alpha_1,\alpha_2,\ell}^{(i,h)}\phi^h_{\ell}(x_i), 
\end{equation}
where $a_{\alpha_1,\alpha_2,\ell}^{(i,h)}$ are the coefficients of the expansion, and the basis functions $\phi^h_{\ell}:\xspace{i} \to \reals$ are hat functions:
\begin{equation*}
\phi_{\ell,h}(x_i) = \left\{
              \begin{array}{cc}
                0 & \textrm{ if } x_i < \dnode{\ell-1}{i}\textrm{ or } x_i > \dnode{\ell+1}{i} \\
                1 &\textrm{ if } x_i = \dnode{\ell}{i} \\
                \frac{x_i-\dnode{\ell-1}{i}}{\dnode{\ell}{i}-\dnode{\ell-1}{i}} & \textrm{ if } \dnode{\ell-1}{i} \leq x_i < \dnode{\ell}{i} \\
                \frac{\dnode{\ell+1}{i}-x_i}{\dnode{\ell+1}{i}-\dnode{\ell}{i}} & \textrm{ if } \dnode{\ell}{i} < x_i \leq \dnode{\ell+1}{i} \\
              \end{array}
  \right. .
\end{equation*}
Note that these basis functions yield a \textit{linear element} interpolation\footnote{If we would have chosen a piecewise constant reconstruction, then for all intents and purposes the FT would be equivalent to the TT. Indeed we have previously performed such an approximation for the value function~\cite{Gorodetsky2015}.} of the function when evaluating it for a state not contained within $\ssd{}{h}$.
We will denote the FT cores of the value functions for discretization level $h$ as $\core{F}_{i}^{h}$. Finally, evaluating $\polval{h}$ anywhere within $\ssd{}{}$ requires evaluating a sequence of matrix-vector products
\begin{align*}
\polval{h}(x_1,x_2,\ldots,x_d) &= \core{F}_{1}^h(x_1)\core{F}_{2}^h(x_2)\ldots \core{F}_{d}^{h}(x_d), 
\end{align*}
for all $x_i \in \ssd{i}{}$.

\subsection{FT-based value iteration algorithm}\label{sec:ftvi}

In prior work~\cite{Gorodetsky2015}, we introduced a version of \rev{value iteration (VI)} in which each update given by Equation~\eqref{eq:vi} was performed by using a low-rank tensor interpolation scheme to selectively choose a small number of states $\xo \in \ssd{}{h}$. Here we follow a similar strategy, except we use the continuous space approximation algorithm described in Algorithm~\ref{alg:rankadapt}, and denoted by $\ftcross$, to accomodate our continuous space approximation. By seeking a low-rank representation of the value function, we are able to avoid visiting every state in $\ssd{}{h}$ and achieve significant computational savings. The pseudocode for low-rank VI is provided by Algorithm~\ref{alg:ttvi}. In this algorithm, $\vali{h}{k}$ denotes the value function approximation during the $k$-th iteration.

\begin{algorithm}
  \caption{FT-based Value Iteration (FTVI)
  }
  \begin{algorithmic}[1]                                                  
    \label{alg:ttvi}                                     
    \REQUIRE \hspace*{23pt} \ftcross tolerances $\epsilon = (\crossdelta,\roundeps)$; \newline 
             \hspace*{23pt} Initial cost function in FT format $\vali{h}{0}$; \newline
             \hspace*{23pt} Convergence tolerance $\delta_{\max}$
    \ENSURE Residual $\delta = \Vert \vali{h}{k}\ - \vali{h}{k-1} \Vert^2 < \delta_{\text{max}}$
    \STATE $k=0$                                                            
    \REPEAT 
    \STATE $ \displaystyle{ \vali{h}{k+1} = \ftcross \left(\polfpopt{h}(\vali{h}{k}),\epsilon\right)} $ \label{alg:vi:update}
    \STATE $k \leftarrow k+1$                                                    
    \STATE $\delta = \Vert \vali{h}{k} - \vali{h}{k-1} \Vert^2 $
    \UNTIL{$\delta < \delta_{\text{max}}$} 
  \end{algorithmic}                                                        
\end{algorithm}     
The update step \ref{alg:vi:update} treats the VI update as a black box function into which one feeds a state and obtains an updated cost. After $k$ steps of FT-based VI we can obtain a policy as the minimizer of $\polfpopt{h}(\vali{h}{k})(\xo)$, for any state $\xo \in \ssd{}{}$.

\subsection{FT-based policy iteration algorithm}\label{sec:ftpi}

As part of the policy iteration algorithm, for each $\mu_k$, one needs to solve Equation~\eqref{eq:pifp} for the corresponding value function $\polvali{h}{k}$. This system of equations has an equivalent number of unknowns as states in $\ssd{}{h}$. Hence, the number of unknowns grows exponentially with dimension, for tensor-product discretizations. In order to efficiently solve this system in high dimensions we seek \textit{low-rank} solutions. A wide variety of low-rank linear system solvers have recently been developed that can potentially be leveraged for this task~\cite{Oseledets2012,Dolgov2013}. %

We focus on optimistic policy iteration, where we utilize the contractive property of $\polfpi{h}{k}$ to solve Equation~\eqref{eq:pifp} approximately, using $n_{fp}$ fixed point iterations. See Section~\ref{sec:pi}. We leverage the low-rank nature of each intermediate value $\polvali{h}{k}$ by interpolating a new value function for each of these iterations. Notice that this iteration in Equation~\eqref{eq:valuefp} is much less expensive than the value iteration in Equation~\eqref{eq:vi}, because it does not involve any minimization. 

The pseudocode for the FT-based optimistic policy iteration is provided in Algorithm~\ref{alg:ttpi}.
\begin{algorithm}[b]
  \caption{FT-based Optimistic Policy Iteration (FTPI)}
  \begin{algorithmic}[1]                                                  
    \label{alg:ttpi}                                     
    \REQUIRE Termination criterion $\delta_{\text{max}}$; \newline
             \hspace*{23pt} \ftcross tolerances $\epsilon = (\crossdelta,\roundeps)$; \newline 
             \hspace*{23pt} Initial value function in FT format $\polvali{h}{0}$; \newline 
             \hspace*{23pt} Number of FP sub-iterations iterations $n_{fp}$
    \STATE $k=1$                                                            
    \REPEAT 
    \STATE $\mu_{k} = \texttt{ImplicitPolicy}\left(\displaystyle{\arg \min_{\mu}} \left[ \polfp{h}(\polvali{h}{k-1})\right]\right)$ \label{alg:pi:implicitpol}
    \STATE $\polvali{h}{k} = \polvali{h}{k-1}$
    \FOR{$\ell=1$ \TO $n_{fp}$} \label{alg:pi:polfunc}
    \STATE $\polvali{h}{k} = \ftcross \left(\polfp{h}\left(\polvali{h}{k}\right),\epsilon\right)$ 
    \ENDFOR
    \STATE $\delta = \Vert \polvali{h}{k} - \polvali{h}{k-1} \Vert^2 $
    \STATE $k \leftarrow k+1$                                                    
    \UNTIL {$\delta < \delta_{\text{max}}$} 
  \end{algorithmic}                                                        
\end{algorithm}
In Line~\ref{alg:pi:implicitpol}, we represent a policy $\mu_k$ \textit{implicitly} through a value function. We make this choice, instead of developing a low-rank representation of $\mu_k$, because the policies are generally not low-rank, in our experience. %
Indeed, discontinuities can arise due to regions of uncontrollability, and these discontinuities increase the rank of the policy. Instead, to evaluate an implicit policy $\mu_{k}$ at a location $\xo$, one is required to solve the optimization problem in Equation~\eqref{eq:pi} using the fixed value function $\polvali{h}{k-1}$. However, one can store the policy evaluated at nodes visited by the cross approximation algorithm to avoid recomputing them during each iteration of the loop in Line~\ref{alg:pi:polfunc}. Since the number of nodes visited during the approximation stage should be relatively low, this does not pose an excessive algorithmic burden.

\subsection{FT-based multigrid algorithm}\label{sec:ftmultigrid}

In this section, we propose a novel FT-based multigrid algorithm. Recall that the first ingredient of multigrid is a prolongation operator $I_{h}^{2h}$ which takes functions defined on the grid $\ssd{}{h}$ into a coarser grid $\ssd{}{2h}$. The second ingredient is an interpolation operator $I_{2h}^h$ which interpolates functions defined on $\ssd{}{2h}$ onto the functions defined on $\ssd{}{h}$.

Many of the common operators used for $I_{h}^{2h}$ and $I_{2h}^{h}$ can take advantage of the low rank structure of any functions on which they are operating. In particular, performing these operations on function in low-rank format often simply requires performing their one-dimensional variants onto each univariate scalar-valued function of its FT core. 
Let us describe the FT-based prolongation and interpolation operators. 

The \textit{prolongation operator} that we use picks out values common to both $\ssd{}{h}$ and $\ssd{}{2h}$ according to 
\begin{equation*}
(I_h^{2h}\polval{h})(\xo) = \polval{2h}(\xo), \quad \forall \xo \in \ssd{}{2h}
\end{equation*}
This operator requires a constant number of computational operations, only access to memory. Furthermore, the coefficients of $\ffiber{\alpha_1,\alpha_2}{i,h}$ are reused to form the coefficients of $\ffiber{\alpha_1,\alpha_2}{i,2h}$. See Equation~\eqref{eq:hatfunc}. Suppose that fine grid is $\ssd{i}{h} \allowbreak= \allowbreak\{\dnode{1}{i},\allowbreak\dnode{2}{i}, \allowbreak\ldots,\dnode{n}{i}\}$ and the coarse grid $\ssd{i}{2h} = \{\dnode{2l-1}{i}\}_{l=1,\ldots,n/2}$ consists of half the nodes ($|\ssd{i}{2h}| = n/2$). Then the univariate functions making up the cores of $\polval{2h}$ are
\begin{align*}
\ffiber{j,k}{i,2h}(x_i) &= \sum_{\ell=1}^{n/2} a_{j,k,\ell}^{(i,2h)}\phi^{2h}_{\ell}(x_i), \\
                   &= \sum_{\ell=1}^{n/2} a_{j,k,2\ell-1}^{(i,h)}\phi^{2h}_{\ell}(x_i),
\end{align*}
where $\phi_{\ell}^{2h}$ are the functions defined on the coarser grid. In the second line, we use every other coefficient from the finer grid as coefficients of the corresponding coarser grid function.\footnote{Alternatively, one can choose more regular nodal basis functions, such as splines. For other basis functions, there may be more natural prolongation and interpolation operators. We choose hat functions in this paper, because we observe that they are well behaved in the face of discontinuities or extreme nonlinearies often encountered in the solution of the HJB equation.}

The \textit{interpolation operator} arises from the interpolation that the FT performs, and in our case, the use of hat functions leads to a linear interpolation scheme. This means that if the scalar-valued univariate functions making up the cores of $\polval{2h}$ are represented in a nodal basis obtained at the tensor product grid $\ssd{}{2h} = \ssd{1}{2h} \times \ldots \times \ssd{d}{2h},$ then we obtain a nodal basis with twice the resolution defined on $\ssd{}{h} = \ssd{1}{h} \times \ldots \times \ssd{d}{h}$ through interpolation of each core. This operation requires interpolating of each of the univariate functions making up the cores of $\polval{2h}$ onto the fine grid. Thus $\polval{h}$ becomes an FT with cores consisting of the univariate functions
\begin{align*}
\ffiber{j,k}{i,h}(x_i) &= \sum_{\ell=1}^n a_{j,k,\ell}^{(i,h)}\phi_{\ell}^h(x_i), \\
                   &= \sum_{\ell=1}^n \ffiber{j,k}{i,2h}(\dnode{\ell}{i}) \phi_{\ell}^h(x_i),
\end{align*}
where we note that in the second equation we use evaluations of the coarser basis functions to obtain the coefficients of the new basis functions.
In summary, both operators can be applied to each FT core of the value function separately. Both of these operations, therefore, scale linearly with dimension.

The FT-based two-level V-grid algorithm is provided in Algorithm~\ref{alg:Vgrid}. In this algorithm, an approximate solution for each equation is obtained by $\ell_i$ fixed point iterations at grid level $h_i$. An extension to other grid cycles and multiple levels of grids is straightforward and can be performed with all of the same operations. %
\begin{algorithm}
  \caption{Two-level FT-based V-grid 
  }
  \begin{algorithmic}[1]                                                  
    \label{alg:Vgrid}                                     
    \REQUIRE Discretization levels $\{h_1,h_2\}$ such that $h_1 < h_2$; \newline
             \hspace*{23pt} Number of iterations at each level $\{\ell_1,\ell_2\}$; \newline
             \hspace*{23pt} Initial value function $\polvala{h}_0$; \newline 
             \hspace*{23pt} Policy $\mu$; \newline 
             \hspace*{23pt} \ftcross~tolerances $\epsilon = (\crossdelta,\roundeps)$
    \STATE $k = 0$
    \REPEAT
        \FOR {$\ell = 1, \ldots, \ell_1$}
            \STATE $\polvala{h_1}_k \leftarrow \ftcross(\polfp{h_1}(\polval{h_1}_{k}),\epsilon)$
        \ENDFOR
        \STATE $r^h = \ftcross(\polfpw{h}(\polvala{h_1}_{k}),\epsilon) - \polvala{h_1}_k$
        \STATE $\polerr{h_2} = 0$
        \FOR {$\ell = 1, \ldots, \ell_2$}
            \STATE $\polerr{h_2}_{n+1} \leftarrow \ftcross(\polfp{h_2}(\polerr{h_2}_{n};I_{h_1}^{h_2}r^h),\epsilon)$ \COMMENT{See~\eqref{eq:mgfp}}
        \ENDFOR
        \STATE $\polvala{h_1}_{k+1} = \polvala{h_1}_{k} + I_{h_2}^{h_1} \polerr{h_2}$
        \STATE $k \leftarrow k+1$
    \UNTIL convergence
  \end{algorithmic}                                                        
\end{algorithm}

\section{Analysis}\label{sec:analysis}
In this section, we analyze the convergence properties and the computational complexity of the FT-based algorithms proposed in the previous section. First, we prove the convergence and accuracy of approximate fixed-point iteration methods in Section~\ref{sec:ftapp}. Then, in Section~\ref{sec:convlowrank}, we apply these result to prove the convergence of the proposed FT-based algorithms we discuss computational complexity in Section~\ref{sec:complexity}.

The algorithms discussed in the previous section all rely on performing cross approximation of a function with a relative accuracy tolerance $\epsilon$. Recall, from Section~\ref{sec:crossround} that cross approximation yields an exact reconstruction only when the value function has finite rank unfoldings. In this case, rounding to a relative error $\epsilon$ also guarantees that we achieve an $\epsilon$-accurate solution. %
In what follows, we provide conditions under which FT-based dynamic programming algorithms converge, assuming the rank-adaptive cross approximation algoritm $\ftcross$ algorithm provides an approximation with at most $\epsilon$ error.

\subsection{Convergence of approximate fixed-point iterations}\label{sec:ftapp}

We start by showing that a small relative error made during each step of the relevant fixed-point iterations result in a bounded overall approximation error. 

We begin recalling the contraction mapping theorem.
\begin{theorem}[Contraction mapping theorem; Proposition B.1 by Bertsekas~\cite{Bertsekas2013}]\label{th:contmap} %
Let $\rspace{}$ be a complete vector space and $\rspacec{}$ be a closed subset. Then if $\ct:\rspacec{} \to \rspacec{}$ is a contraction mapping with modulus $\gamma \in (0,1)$, there exists a unique $\polval{} \in \rspacec{}$ such that 
$
\polval{} = \ct(\polval{}).
$
Furthermore, the sequence defined by $\polvali{}{0} \in \rspacec{}$ and the iteration $\polvali{}{k} = \ct(\polvali{}{k-1})$ converges to $\polval{}$ for any $\polval{} \in \rspacec{}$ according to
\begin{equation*}
\lVert \polvali{}{k} - \polval{} \rVert \leq \gamma^{k} \lVert \polvali{}{0} - \polval{} \rVert, \quad k=1,2,\ldots
\end{equation*}
\end{theorem}
Theorem~\ref{th:contmap} can be used, for example, to prove the convergence of the value iteration algorithm, when the operator $\polfpopt{h}$ is a contraction mapping. The proposed FT-based algorithms are based on {\em approximations} to contraction mappings. To prove their convergence properties, it is important to understand when approximate fixed-point iterations converge and the accuracy which they attain. Lemma~\ref{th:afp} below addresses the convergence of approximate fixed-point iterations.
\begin{lemma}[Convergence of approximate fixed-point iterations]\label{th:afp}
Let $\rspace{h}$ be a closed subset of a complete vector space. Let $\ct:\rspace{h} \to \rspace{h}$ be a contractive mapping with modulus $\gamma \in (0,1)$ and fixed point $\polval{h}$. Let $\ctt:\rspace{h} \to \rspace{h}$ be an approximate mapping such that 
\begin{equation}\label{eq:apcontract}
\lVert \ctt\left(\polval{\prime}\right) - \ct\left(\polval{\prime}\right) \rVert \leq \epsilon \lVert \ct\left(\polval{\prime}\right) \rVert, \quad \forall \polval{\prime} \in \rspace{h},
\end{equation}
for $\epsilon > 0$. Then, the sequence defined by $\polvali{h}{0} \in \rspace{h}$ and the iteration $\polvali{h}{k} = \ctt \left( \polvali{h}{k-1}\right)$ for $k=1,2,\ldots$ satisfies
\begin{align*}
\lVert \polvali{h}{k} - \polval{h} \rVert & \leq \epsilon \frac{1 - \left(\gamma\epsilon + \gamma\right)^{k}}{1 - \left(\gamma\epsilon+\gamma\right)} \lVert \polval{h} \rVert \ + \\
    & \quad \quad \quad \qquad \qquad \left(\gamma \epsilon + \gamma\right)^{k} \lVert \polvali{h}{0} - \polval{h} \rVert.
\end{align*}
\end{lemma}
\begin{proof}
The proof is a standard contraction argument. 
We begin by bounding the difference between the $k$-th iterate and the fixed point $\polval{h}$:
\begin{align*}
\lVert \polvali{h}{k} - \polval{h} \rVert &= \lVert \polvali{h}{k} - \ct(\polvali{h}{k-1}) + \ct(\polvali{h}{k-1}) - \polval{h} \rVert \\
& \leq \lVert \polvali{h}{k} - \ct(\polvali{h}{k-1}) \rVert  + \lVert \ct(\polvali{h}{k-1}) - \polval{h} \rVert \\
& \leq \epsilon \lVert \ct(\polvali{h}{k-1}) \rVert  + \gamma \lVert \polvalai{h}{k-1} - \polval{h} \rVert \\
& \leq \epsilon \lVert \ct(\polvali{h}{k-1}) - \polval{h} \rVert + \epsilon \lVert \polval{h} \rVert \ +\\
& \qquad \qquad  \gamma \lVert \polvali{h}{k-1} - \polval{h} \rVert \\
& \leq \left(\gamma \epsilon + \gamma\right) \lVert \polvali{h}{k-1} - \polval{h} \rVert + \epsilon \lVert \polval{h} \rVert,
\end{align*}
where the second inequality comes from the triangle inequality, the third comes from Equation~\eqref{eq:apcontract} and contraction, the fourth inequality arises again from the triangle inequality, the final inequality arises from contraction.
Using recursion results in
\begin{align*}
\lVert \polvali{h}{k} - \polval{h} \rVert & \leq \left(\gamma \epsilon + \gamma\right)\Big[\left(\gamma \epsilon + \gamma\right) \left[ \left(\gamma \epsilon + \gamma\right) \ldots + \epsilon \lVert \polval{h} \rVert \right] \\
                                          & \quad \quad \quad \quad \quad \quad \quad \quad + \epsilon \lVert \polval{h} \rVert \Big] + \epsilon \lVert \polval{h} \rVert  \\ 
                                          &=  \epsilon \lVert \polval{h} \rVert \left[ \sum_{\ell=0}^{k-1} \left(\gamma \epsilon + \gamma\right)^{\ell} \right] + \\
                                          & \quad \quad \quad \quad \quad \quad \quad \quad \left(\gamma \epsilon + \gamma\right)^{k} \lVert \polvalai{h}{0} - \polval{h} \rVert.
\end{align*}
Evaluating the sum of a geometric series,
\begin{align*}
\lVert \polvalai{h}{k} - \polval{h} \rVert & \leq \epsilon \frac{1 - \left(\gamma\epsilon + \gamma\right)^{k}}{1 - \left(\gamma\epsilon+\gamma\right)} \lVert \polval{h} \rVert  + \\
                                           & \quad \quad \quad \quad \quad
                                           \left(\gamma \epsilon + \gamma\right)^{k} \lVert \polvalai{h}{0} - \polval{h} \rVert,
\end{align*}
we reach the desired result.
\end{proof}

Notice that, as expected, when $\epsilon=0$ and $k \to \infty$ this result yields that the iterates $\polvali{h}{k}$ converge to the fixed point $\polval{h}$. Second, the condition $\gamma\epsilon  + \gamma < 1$ is required to avoid divergence. This condition effectively states that, if the contraction modulus is small enough, a larger approximation error may be incurred. On the other hand, if the contraction modulus is large, then the approximation error must be small. In other words, this requirement can be thought as a condition for which the approximation can remain a contraction mapping; larger errors can be tolerated when the original contraction mapping has a small modulus, and vice-versa.

The following result is an alternative to the previous lemma. It is more flexible in the size of error that it allows, so long as each iterate is bounded.

\begin{lemma}[Convergence of approximate fixed-point iterations with a boundedness assumption]\label{th:afp2}
Let $\rspace{h}$ be a closed subset of a complete vector space. Let $\ct:\rspace{h} \to \rspace{h}$ be a contractive mapping with modulus $\gamma \in (0,1)$ and fixed point $\polval{h}$. Let $\ctt:\rspace{h} \to \rspace{h}$ be an approximate mapping such that 
\begin{equation}\label{eq:apcontract2}
\lVert \ctt\left(\polval{\prime}\right) - \ct\left(\polval{\prime}\right) \rVert \leq \epsilon \lVert \ct\left(\polval{\prime}\right) \rVert, \quad \forall \polval{\prime} \in \rspace{h},
\end{equation}
for $\epsilon > 0$. Let $\polvali{h}{0} \in \rspace{h}$. Define a sequence the sequence $\{\polvali{h}{k}\}$ according to the iteration $\polvali{h}{k} = \ctt \left( \polvali{h}{k-1}\right)$ for $k=1,2,\ldots$. Assume that $\lVert \polvali{h}{k} \rVert \leq \rho_1 < \infty $ so that $\lVert \ct\left(\polvali{h}{k}\right) \rVert \leq \rho < \infty.$ Then $\{\polvali{h}{k}\}$ satisfies
\begin{equation}\label{eq:c2res1}
\lVert \polvali{h}{k} - \ct^{[k]}(\polvali{h}{0}) \lVert \leq \frac{1-\gamma^k}{1-\gamma} \epsilon \rho, \quad \forall k,
\end{equation}
so that,
\begin{equation}\label{eq:c2res}
\lim_{k \to \infty} \lVert \polvali{h}{k} - \polval{h} \lVert \leq \frac{\epsilon \rho}{1 - \gamma},
\end{equation}
where $\ct^{[k]}$ denotes $k$ applications of the mapping $\ct$.
\end{lemma}
\begin{proof}
The strategy for this proof again relies on standard contraction and triangle inequality arguments. Furthermore, it follows closely the proof of Proposition 2.3.2 (Error Bounds for Approximate VI) work by Bertsekas~\cite{Bertsekas2013}. In that work, the proposition provides error bounds for approximate VI when an absolute error (rather than a relative error) is made during each approximate FP iteration. The assumption of boundedness that we use here will allow us to use the same argument as the one made by Bertsekas.

Equation~\eqref{eq:apcontract2} implies 
\begin{equation}\label{eq:rrr}
\lVert \polvali{h}{k} - \ct(\polvali{h}{k-1}) \rVert = \lVert \ctt(\polvali{h}{k-1}) - \ct(\polvali{h}{k-1}) \rVert \leq \epsilon \lVert\ct(\polvali{h}{k-1}) \rVert.
\end{equation}
Recall $\ct^{[k]}(\polval{\prime})$ denotes $k$ applications of the operator $\ct$, \ie
\begin{align*}
\ct^{[k]}\left(\polval{\prime}\right) &= \ct^{[k-1]}\left(\ct\left(\polval{\prime}\right)\right) = \ct^{[k-2]}\left(\ct\left(\ct\left(\polval{\prime}\right)\right)\right) = \cdots \\
                                     &= \ct\left(\ct\left(\ct\left( \cdots \ct\left(\polval{\prime} \right) \right) \right)\right).
\end{align*}
Using the triangle inquality, contraction, and Equation~\eqref{eq:rrr}, 
\begin{align*}
\lVert \polvali{h}{k} - \ct^{[k]}(\polvali{h}{0}) \lVert &\leq \lVert \polvali{h}{k} - \ct(\polvali{h}{k-1}) \rVert \\ 
         & \quad \quad + \lVert \ct(\polvali{h}{k-1}) - \ct^{[2]}(\polvali{h}{k-2}) \rVert + \cdots \\
         & \quad \quad + \lVert \ct^{[k-1]}(\polvali{h}{1}) - \ct^{[k]}(\polvali{h}{0}) \rVert \\
         & \leq \epsilon \lVert \ct(\polvali{h}{k-1}) \rVert +  \gamma \epsilon \lVert \ct(\polvali{h}{k-2}) \rVert + \cdots \\
         & \quad \quad + \gamma^{k-1} \epsilon \lVert \ct(\polvali{h}{0}) \rVert.
\end{align*}
The boundedness assumption yields
\begin{equation*}
\lVert \polvali{h}{k} - \ct^{[k]}(\polvali{h}{0}) \lVert \leq \epsilon \rho +  \gamma \epsilon \rho + \cdots + \gamma^{k-1} \epsilon \rho.
\end{equation*}
We then evaluate the sum of a geometric series, 
\begin{equation*}
\lVert \polvali{h}{k} - \ct^{[k]}(\polvali{h}{0}) \lVert \leq \frac{1-\gamma^k}{1-\gamma} \epsilon \rho.
\end{equation*}
Taking the limit $k \to \infty$ and using Theorem~\ref{th:contmap}, where $\lim_{k\to \infty}\ct^{[k]}(\polvali{h}{0}) = \polval{h}$, yields Equation~\eqref{eq:c2res}.
\hfill
\end{proof}
Equation~\eqref{eq:c2res1} shows the difference between iterates of exact fixed-point iteration and approximate fixed-point iteration, and indicates that this difference can not grow larger than $\epsilon \rho$. Furthermore, this result does not require any assumption on the relative error $\epsilon$. 

In the next section, we use these intermediate results to prove that the proposed FT-based dynamic programming algorithms converge under certain conditions.

\subsection{Convergence of FT-based fixed-point iterations}\label{sec:convlowrank}

In order for the above theorems to be applicable to the case of low-rank approximation using the \ftcross algorithm, we need to be able to explicitly bound the error committed during each iteration of FT-based value iteration, and each subiteration within the optimistic policy iteration algorithm. 
In order to strictly adhere to the intermediate results of the previous section, we focus our attention to functions with finite FT rank. This assumption is formalized below.
\begin{assumption}[Bounded ranks of FT-based FP iteration]\label{as:boundedrank}
Let $\rspace{h}$ be a closed subset of a complete vector space. Let $\ct:\rspace{h} \to \rspace{h}$ be a contractive mapping with modulus $\gamma \in (0,1)$ and fixed point $\polval{h}$.
The sequence defined by the initial condition $\polvali{h}{0} \in \rspace{h}$ and the iteration
$\polvali{h}{k} = \ftcross\left(f\left(\polvali{h}{k-1}\right),\epsilon\right)$ has the property that the functions $\{ \ct(\polvali{h}{k}) \}$ have finite FT ranks \rev{and that Algorithm~\ref{alg:rankadapt} successfully finds upper bounds to these ranks.} In other words, each of the unfoldings of $\ct(\polvali{h}{k})$ have rank bounded by some $r < \infty$,
\begin{align*}
\rank \left[ \ct(\polvali{h}{k}) (\{x_1,\ldots,x_l\}; \{x_{l+1} \ldots x_{d}\}) \right] &< r < \infty,
\end{align*}
for $l = 1, \ldots d-1$.
\end{assumption}

Since we are approximating a function with finite FT ranks at each step of the fixed-point iteration under Assumption~\ref{as:boundedrank}, the \ftcross algorithm converges. \rev{Furthermore in order to guarantee an accuracy of approximation at each step, we need to assume that the rounding procedure successfully finds an upper bound to the true ranks. Such an upper bound is necessary to guarantee that the cross-approximation algorithm can exactly represent the function. If the function is represented exactly after cross approximation, then rounding can generate an approximation with arbitrary accuracy. Indeed rounding helps control the growth in ranks that might be required by an exact solver.}
Thus, this assumption leads to the satisfaction of the conditions required by Lemmas~\ref{th:afp} and~\ref{th:afp2}. Then, the following result is immediate.
\begin{theorem}[Convergence of the FT-based value iteration algorithm]
Let $\rspace{h}$ be a closed subset of a complete vector space. Let the operator $\polfpopt{h}$ of~\eqref{eq:optval} be a contraction mapping with modulus $\gamma$ and fixed point $\val{h}$. Define a sequence of functions according the an initial function $\vali{h}{0} \in \rspace{h}$ and the iteration $\vali{h}{k} = \ftcross\left(\polfpopt{h}\left(\vali{h}{k-1}\right),\epsilon\right)$. Assume Assumption~\ref{as:boundedrank}. Furthermore, assume that $\lVert \polvali{h}{k} \rVert \leq \rho_1 < \infty $ so that $\lVert \ct\left(\polvali{h}{k}\right) \rVert \leq \rho < \infty.$ Then, FT-based VI converges according to
\begin{equation*}
\lim_{k \to \infty} \lVert \vali{h}{k} - \val{h} \lVert \leq \frac{\epsilon \rho}{1 - \gamma}.
\end{equation*}
\end{theorem}
\rev{ Note that in practice, the functions $\ct(\polvali{h}{k})$ may not have finite rank, but rather a decaying spectrum, and can therefore be well approximated numerically by low-rank functions. In such cases, the results still hold for an error $\epsilon$ incurred at each step; however, we cannot guarantee or indeed check the value of $\epsilon$ obtained by the compression routine. Still, numerical examples suggest that the rank adaptation scheme is able to find sufficiently large ranks to obtain good performance.} 

\subsection{Complexity of FT-based dynamic programming algorithms}\label{sec:complexity}
Suppose that each dimension is discretized into $n$ nodes,
then recall that Algorithm~\ref{alg:rankadapt} consists of a single interpolation of black box function having all FT ranks equal to $r$ that requires $\mathcal{O}(nr^2)$ evaluations of the black box function for cross approximation and a rounding step that requires  $\mathcal{O}(nr^3)$ operations. %

Suppose we use the upwind differencing scheme described in Section~\ref{sec:discretized}.
Then, the complexity of one step of an approximate fixed-point iteration can be characterized as follows.

\begin{proposition}[Complexity of the evaluation of Equation~\eqref{eq:polrhs}]\label{prop:bellmaneval}
Let the evaluation of stage cost $\stagecost{h}$, drift $\drift$, and diffusion $\diffusion$ require $n_{\textrm{op}}$ operations. Let the discretization $\ssd{}{h}$ of the MCA method arise from a tensor product of $n$-node discretizations of each dimension of the state space. Let the resulting transition probabilities $\ptrans{h}(\xo,\xoo|\ue)$ be computed according to the upwind scheme described in Section~\ref{sec:discretized}. Furthermore, let the value function $\polvali{h}{k}$ have ranks $\bvec{r} = [r_0,r_1,\ldots,r_d]$ where $r_0=r_d=1$ and $r_k = r$ for $k = 1 \ldots d-1$. Then, evaluating 
\begin{align*}
\stagecost{h}(\xo,\ue) + \gamma \sum_{\xoo \in \ssd{}{h}} \ptrans{h}(\xo,\xoo|\ue) \polvali{h}{k}(\xoo)
\end{align*}
for a fixed $\xo \in \ssd{}{h}$ and $\ue \in \cs$ requires $\mathcal{O}(n_{\textrm{op}} + d^2nr^2)$ operations.
\end{proposition}
\begin{proof}
In the specified upwind discretization scheme, there exist $2d$ neighbors to which the transition probabilities are non-zero. Furthermore, in Section~\ref{sec:discretized}, we showed that computing all of these transition probabilities requires $\mathcal{O}(n_{\textrm{op}} + d)$ operations. The evaluation of the cost of each neighbor, $\polvali{h}{k}(\xoo)$, requires $\mathcal{O}(dnr^2)$ evaluations. Since this evaluation is required at $2d$ neighbors, a conservative estimate for this total cost is $\mathcal{O}(d^2nr^2).$ Thus, the result is obtained by observing that the cost is dominated by the computation of the transition probabilities and the evaluation of value function at the neighboring grid points.
\end{proof}

Using Proposition~\ref{prop:bellmaneval} and assuming that for each $\xo$ the minimization over control $\ue$ requires $\mathcal{\kappa}$ evaluations of Equation~\eqref{eq:polrhs}, the following result is immediate.
\begin{theorem}[Computational complexity of the FT-based value iteration algorithm]
Let $\rspace{h}$ be a closed subset of a complete vector space. Let the operator $\polfpopt{h}$ of Equation~\eqref{eq:optval} be a contraction mapping with modulus $\gamma$ and fixed point $\val{h}$. Let the evaluation of stage cost $\stagecost{h}$, drift $\drift$, and diffusion $\diffusion$ corresponding to $\polfpopt{h}$ require $n_{\textrm{op}}$ operations. Let the discretization $\ssd{}{h}$ of the MCA method arise from a tensor product of $n$-node discretizations of each dimension of the state space.

 Define a sequence of functions according the an initial function $\vali{h}{0} \in \rspace{h}$ and the iteration $\vali{h}{k} = \ftcross\left(\polfpopt{h}\left(\vali{h}{k-1}\right),\epsilon\right)$. Assume that for each $\xo \in \xspace{h}$, the minimization over control $\ue \in \cs$ requires $\mathcal{\kappa}$ evaluations of Equation~\eqref{eq:polrhs}. Then, each iteration of FT-based VI involves cross approximation and rounding and requires 
$$
\mathcal{O}\left(dnr^2\kappa \left( n_{\textrm{op}} + d^2nr^2\right) + dnr^3\right)
$$
operations.
\end{theorem}

We remark that the the computational cost of the proposed FT-based value iteration algorithm grows polynomially with increasing dimensionality. Therefore, this algorithm mitigates the curse of dimensionality as long as the rank $r$ of the problem does not grow exponentially with dimensionality.

\section{Computational results and\\ quadcopter experiments}\label{sec:numexamples}

In this section, we demonstrate the proposed algorithm in a number of challenging examples. First, we demonstrate the algorithm on simple control problems involving linear dynamics, quadratic cost and Gaussian process noise in Section~\ref{sec:lqg}. While these problems are not high dimensional, we experiment with various parameters, such as the amount the noise, to see their effect on various problem variables, such as rank. In this manner, we gain insight into various problem dependent properties including the FT rank. Next, we consider four different dynamics with increasing dimensionality. In Section~\ref{sec:cars}, we consider two models of car-like robots with dimension three and four. These models are nonlinear. In Section~\ref{sec:perch}, we consider a widely-studied perching problem that features nonlinear underactuated dynamics with a seven-dimensional state space that is not affine in control. Finally, in Section~\ref{sec:quad}, we consider a problem instance involving a quadcopter maneuvering through a tight window, using a nonlinear quadcopter model with a six-dimensional state space. We compute a full-state feedback controller, and demonstrate it on a quadcopter flying through a tight window using a motion-capture system for full state estimation. 

The simulation results in this section are obtained using multiple threads of an Intel Xeon CPU clocked at 2.4GHz. We use the Compressed Continuous computation ($C^3$) library~\cite{c3} for FT-based compressed computation, and this library is BSD licensed and available through GitHub. 

A multistart BFGS optimization algorithm, available within $C^3$, is used for generating a control for a particular state within simulation and real-time system operation in Section~\ref{sec:experiments}. In particular for any state $\xoo$ encountered in our simulation/experiment, we use multistart BFGS to obtain the minimizer of Equation~\eqref{eq:dcih}. Within the objective the compressed cost function is evaluated at the corresponding neighboring states, i.e., $\val{h}(\xoo)$ is evaluated such that $\xoo$ are neighbors determined by the MCA discretization of $\xo$.

The low-rank dynamic programming algorithms are provided in a stochastic control module that is released separately, also on GitHub~\cite{c3sc}.

\subsection{Linear-quadratic-Gaussian problems with bounded control}\label{sec:lqg}
In this section, we investigate the effect of state boundary conditions and control bounds on a prototypical control system. The system has a bounded state space, linear dynamics and quadratic cost. However, it also has a bounded control and state space, thereby making analytical solutions difficult to obtain. 

Consider the following stochastic dynamical system
\begin{align*}
    dx_1 &= x_2dt  + \sigma_1 dw_1(t)\\
    dx_2 &= u(t)dt + \sigma_2 dw_2(t).
\end{align*}
This stochastic differential equations represent a physical system with position $x_1$ and velocity $x_2$ with $\ssos = (-2,2)^2.$ The control input to this system is the acceleration $u$, and we consider different lower and upper bounds on the control space $\controlspace = [u_{lb},u_{ub}]$. 
We consider a discounted-cost infinite-horizon problem, with discount $e^{-\frac{t}{10}}$.
The stage cost is
\begin{equation}
\stagecost{}(x,u) = x_1^2 + x_2^2 + u^2,
\end{equation}
and the terminal cost is
\begin{equation}
\termcost{}(x) = 100, \quad x \in \partial \ssos.
\end{equation}

We first solve the problem for various parameter values to gain insight to the problem. Next, we discuss numerical results summarizing the convergence of the algorithm.

\subsubsection{Parameter studies}
We present computational experiments that assess when low-rank cost functions arise and what factors affect ranks. 
Our first computational experiment studies the effects of the control boundary conditions on a problem with absorbing boundaries. This computational experiment is performed over several different $u_{lb},u_{ub}$ combinations and the resulting optimal value functions are shown in Figure~\ref{fig:lqr_absorb_ubound}. We utilized FT-based policy iteration. The Markov chain approximation was obtained using 60 points in each dimension. 
\newcolumntype{C}{>{\centering\arraybackslash} m{4cm} }  
\newcolumntype{D}{>{\centering\arraybackslash} m{3cm} }  
\begin{figure*}
\begin{center}
\setlength{\extrarowheight}{0.1cm}
\begin{tabular}{c}
  \begin{tabular}{m{0.9cm}|m{3.5cm}m{3.5cm}m{3.5cm}}
    $u_{lb} \backslash u_{ub}$ & \centerline{$10$} & \centerline{$4$} & \centerline{$1/2$} \\
    \hline
    $-10$ &
    \vspace{5pt}
    \begin{subfigure}[b]{0.2\textwidth}
      \includegraphics[width=\textwidth]
      {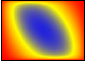}
    \end{subfigure}
    &
    \vspace{5pt}
    \begin{subfigure}[b]{0.2\textwidth}
      \includegraphics[width=\textwidth]{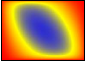}
    \end{subfigure}
    &
    \vspace{5pt}
    \begin{subfigure}[b]{0.2\textwidth}
      \includegraphics[width=\textwidth]{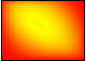}
    \end{subfigure}
    \\
    $-4$
    &
    \begin{subfigure}[b]{0.2\textwidth}
      \includegraphics[width=\textwidth]{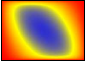}
    \end{subfigure}
    &
    \begin{subfigure}[b]{0.2\textwidth}
      \includegraphics[width=\textwidth]{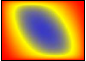}
    \end{subfigure}
    &
    \begin{subfigure}[b]{0.2\textwidth}
      \includegraphics[width=\textwidth]{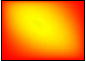}
    \end{subfigure}
    \\
    $-\frac{1}{2}$
    &
    \begin{subfigure}[b]{0.2\textwidth}
      \includegraphics[width=\textwidth]{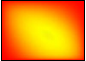}
    \end{subfigure}
    &
    \begin{subfigure}[b]{0.2\textwidth}
      \includegraphics[width=\textwidth]{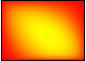}
    \end{subfigure}
    &
    \begin{subfigure}[b]{0.2\textwidth}
      \includegraphics[width=\textwidth]{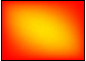}
    \end{subfigure}
  \end{tabular}
  \\
  \begin{subfigure}[b]{0.8\textwidth}
    \quad \quad \quad \quad \quad \quad \quad \quad \includegraphics{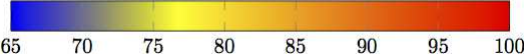}
  \end{subfigure}
\end{tabular}
\caption[Cost functions for the LQG problem with different control bounds and absorbing boundaries.]{Cost functions and ranks of the solution to the LQG problem for varying control bounds $[u_{lb},u_{ub}].$ Diffusion magnitude is $\sigma_1=\sigma_2=1$, and absorbing boundary conditions are used. The FT ranks found through the cross approximation algorithm with rounding tolerance $\roundeps=10^{-7}$ were either 7 or 8. The x-axis of each plot denotes $x_1$ and the y-axis denotes $x_2$.}
\label{fig:lqr_absorb_ubound}
\end{center}
\end{figure*}

Several phenomena are evident in Figure~\ref{fig:lqr_absorb_ubound}. When the range of the valid controls is wide, the value function is able to achieve smaller values, i.e., the blue region (indicating small costs) is larger with a wider control range. This characteristic is expected since the region from which the state can avoid the boundary is larger when more control can be exercised. 
However, the alignment of the value function, with the diagonal stretching from the upper left to the upper right, is the same for all of the test cases. Only the magnitude of the value function changes, and therefore the ranks of the value functions are all either rank 7 or 8. Changing the bounds of the control space does not greatly affect the ranks of this problem.

\begin{figure*}
\begin{center}
\setlength{\extrarowheight}{0.1cm}
  \begin{tabular}{m{0.9cm}|m{3.5cm}m{3.5cm}m{3.5cm}}
    $\sigma_1 \backslash \sigma_2$& \centerline{$10$} & \centerline{$1$} & \centerline{$0.1$} \\
    \hline
    $10$ &
    \vspace{5pt}
    \begin{subfigure}[b]{0.2\textwidth}
      \includegraphics[clip=true,trim=0 0 0 0,width=\textwidth]{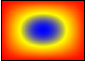}
      \caption{$r=2$}
    \end{subfigure}
    &
    \vspace{5pt}
    \begin{subfigure}[b]{0.2\textwidth}
      \includegraphics[clip=true,trim=0 0 0 0,width=\textwidth]{{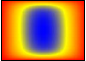}}
      \caption{$r=3$}
    \end{subfigure}
    &
    \vspace{5pt}
    \begin{subfigure}[b]{0.2\textwidth}
      \includegraphics[clip=true,trim=0 0 0 0,width=\textwidth]{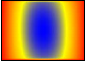}
      \caption{$r=3$}
    \end{subfigure} \\
    $1$ &
    \begin{subfigure}[b]{0.2\textwidth}
      \includegraphics[clip=true,trim=0 0 0 0,width=\textwidth]{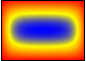}
      \caption{$r=3$}
    \end{subfigure}
    &
    \begin{subfigure}[b]{0.2\textwidth}
      \includegraphics[clip=true,trim=0 0 0 0,width=\textwidth]{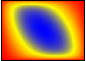}
      \caption{$r=8$}
    \end{subfigure}
    &
    \begin{subfigure}[b]{0.2\textwidth}
      \includegraphics[clip=true,trim=0 0 0 0,width=\textwidth]{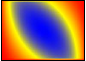}
      \caption{$r=9$}
    \end{subfigure} \\
    $0.1$ &
    \begin{subfigure}[b]{0.2\textwidth}
      \includegraphics[clip=true,trim=0 0 0 0,width=\textwidth]{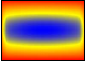}
      \caption{$r=3$}
    \end{subfigure}
    &
    \begin{subfigure}[b]{0.2\textwidth}
      \includegraphics[clip=true,trim=0 0 0 0,width=\textwidth]{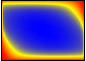}
      \caption{$r=11$}
    \end{subfigure}
    &
    \begin{subfigure}[b]{0.2\textwidth}
      \includegraphics[clip=true,trim=0 0 0 0,width=\textwidth]{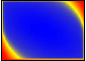}
      \caption{$r=16$}
    \end{subfigure}
  \end{tabular}
\caption[Cost functions and ranks for the LQG problem with different diffusion magnitudes and absorbing boundaries.]{Cost functions and ranks of the solution to the LQG problem for different $\sigma_1,\sigma_2$. We fix $u_{lb} = -3$, $u_{ub} = 3,$ and use absorbing boundary conditions. The FT ranks found through the cross approximation algorithm with rounding tolerance $\roundeps=10^{-7}$ are indicated by the caption for each frame. The x-axis of each plot denotes $x_1$ and the y-axis denotes $x_2$. The color scale is different in each plot to highlight each function's shape.}
\label{fig:lqr_absorb_diff}
\end{center}
\end{figure*}

Next, we consider the effect of the diffusion magnitude on the optimal value function and its rank. The results of this experiment are shown in Figure~\ref{fig:lqr_absorb_diff}. We observe that the diffusion is influential for determining the rank of the problem.
One striking pattern seen in Figure~\ref{fig:lqr_absorb_diff} is that as the diffusion decreases, the blue region grows in size. The blue region indicates low cost, and it intuitively represents the area from which a system will not enter the boundary. In other words, the system is more controllable when there is less noise. When the noise is very large, for example in the upper left panel, there is a large chance that the Brownian motion can push the state into the boundary. This effect causes the terminal cost to propagate further into the interior of the domain. As $\sigma_2$ decreases, there is less noise affecting the acceleration of the state, and the system remains controllable for a wide range of velocities. However, since the diffusion magnitude is large in the equation for velocity, the value function is only small when the position is close to the origin. In the area close to the origin there is less of a chance for the state to be randomly pushed into the absorbing region. Finally, when both diffusions are small, the system is controllable from a far greater range of states, as indicated by the lower right panel.

The ranks of the value functions follow the same pattern as controllability. The ranks are small when the diffusion terms are large, and high then the diffusion terms are small. When the features of the function are aligned with the coordinate axes, the ranks remain low. As the function becomes more complex due to small diffusion terms, the boundaries between guaranteed absorption and nonabsorption begin to have more complex shapes, resulting in increased ranks. 

\begin{figure*}
\begin{center}
\setlength{\extrarowheight}{0.1cm}
\begin{tabular}{c}
  \begin{tabular}{m{0.9cm}|m{3.5cm}m{3.5cm}m{3.5cm}}
    $u_{lb} \backslash u_{ub}$ & \centerline{$10$} & \centerline{$4$} & \centerline{$1/2$} \\
    \hline
    $-10$ &
    \vspace{5pt}
    \begin{subfigure}[b]{0.2\textwidth}
      \includegraphics[clip=true,trim=0 0 0 0,width=\textwidth]{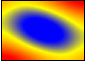}
    \end{subfigure}
    &
    \vspace{5pt}
    \begin{subfigure}[b]{0.2\textwidth}
      \includegraphics[clip=true,trim=0 0 0 0,width=\textwidth]{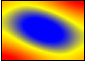}
    \end{subfigure}
    &
    \vspace{5pt}
    \begin{subfigure}[b]{0.2\textwidth}
      \includegraphics[clip=true,trim=0 0 0 0,width=\textwidth]{lqg_rb_12.pdf}
    \end{subfigure}
    \\
    $-4$ &
    \vspace{5pt}
    \begin{subfigure}[b]{0.2\textwidth}
      \includegraphics[clip=true,trim=0 0 0 0,width=\textwidth]{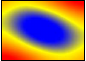}
    \end{subfigure}
    &
    \vspace{5pt}
    \begin{subfigure}[b]{0.2\textwidth}
      \includegraphics[clip=true,trim=0 0 0 0,width=\textwidth]{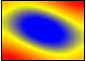}
    \end{subfigure}
    &
    \vspace{5pt}
    \begin{subfigure}[b]{0.2\textwidth}
      \includegraphics[clip=true,trim=0 0 0 0,width=\textwidth]{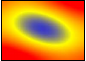}
    \end{subfigure}
    \\
    $-\frac{1}{2}$ &
    \vspace{5pt}
    \begin{subfigure}[b]{0.2\textwidth}
      \includegraphics[clip=true,trim=0 0 0 0,width=\textwidth]{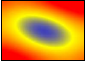}
    \end{subfigure}
    &
    \vspace{5pt}
    \begin{subfigure}[b]{0.2\textwidth}
      \includegraphics[clip=true,trim=0 0 0 0,width=\textwidth]{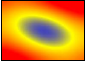}
    \end{subfigure}
    &
    \vspace{5pt}
    \begin{subfigure}[b]{0.2\textwidth}
      \includegraphics[clip=true,trim=0 0 0 0,width=\textwidth]{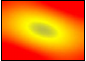}
    \end{subfigure}
  \end{tabular}
 \\
  \begin{subfigure}[b]{0.8\textwidth}
    \quad \quad \quad \quad \quad \quad \quad \quad \includegraphics{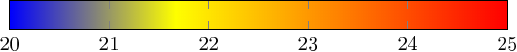}
  \end{subfigure}
\end{tabular}
\caption[Cost functions for the LQG problem with different control bounds and reflecting boundaries.]{Reflecting boundary conditions with cost functions and ranks for different $[u_{lb},u_{ub}]$, and the diffusion is set to $\sigma_1=\sigma_2=1$. The FT ranks found through the cross approximation algorithm with rounding tolerance $\roundeps=10^{-7}$ were 3 or \rev{4} for all cases.}
\label{fig:lqr_reflect_ubound}
\end{center}
\end{figure*}
Next, we investigate the value functions associated with reflecting boundary conditions. The results of this experiment are shown in Figure~\ref{fig:lqr_reflect_ubound}. The results indicate that neither that value functions nor their ranks are much affected by the bounds of the control space.  Notably, the value functions in all examples are of rank 3, and they all appear to be close to a quadratic function. We note that, for a classical LQR problem the solution is quadratic, and therefore also has a value function with rank 3. Thus, in some sense, reflecting boundary conditions more accurately represent a the classical problem with linear dynamics, quadratic cost, and Gaussian process noise, but with no state or input constraints.

\begin{figure*}
\begin{center}
\setlength{\extrarowheight}{0.1cm}
  \begin{tabular}{m{0.9cm}|m{3.5cm}m{3.5cm}m{3.5cm}}
    $\sigma_1 \backslash \sigma_2$ & \centerline{$10$} & \centerline{$1$} & \centerline{$0.1$} \\
    \hline
    $10$ &
    \vspace{5pt}
    \begin{subfigure}[b]{0.2\textwidth}
      \includegraphics[clip=true,trim=0 0 0 0,width=\textwidth]{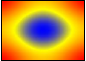}
      \caption{$r=1$}
    \end{subfigure}
    &
    \vspace{5pt}
    \begin{subfigure}[b]{0.2\textwidth}
      \includegraphics[clip=true,trim=0 0 0 0,width=\textwidth]{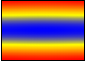}
      \caption{$r=1$}
    \end{subfigure}
    &
    \vspace{5pt}
    \begin{subfigure}[b]{0.2\textwidth}
      \includegraphics[clip=true,trim=0 0 0 0,width=\textwidth]{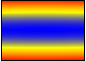}
      \caption{$r=1$}
    \end{subfigure} \\
    $1$ &
    \begin{subfigure}[b]{0.2\textwidth}
      \includegraphics[clip=true,trim=0 0 0 0,width=\textwidth]{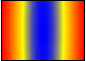}
      \caption{$r=1$}
    \end{subfigure}
    &
    \begin{subfigure}[b]{0.2\textwidth}
      \includegraphics[clip=true,trim=0 0 0 0,width=\textwidth]{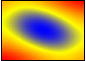}
      \caption{$r=3$}
    \end{subfigure}
    &
    \begin{subfigure}[b]{0.2\textwidth}
      \includegraphics[clip=true,trim=0 0 0 0,width=\textwidth]{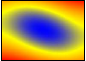}
      \caption{$r=4$}
    \end{subfigure} \\
    $0.1$ &
    \begin{subfigure}[b]{0.2\textwidth}
      \includegraphics[clip=true,trim=0 0 0 0,width=\textwidth]{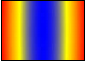}
      \caption{$r=1$}
    \end{subfigure}
    &
    \begin{subfigure}[b]{0.2\textwidth}
      \includegraphics[clip=true,trim=0 0 0 0,width=\textwidth]{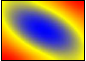}
      \caption{$r=6$}
    \end{subfigure}
    &
    \begin{subfigure}[b]{0.2\textwidth}
      \includegraphics[clip=true,trim=0 0 0 0,width=\textwidth]{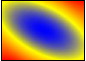}
      \caption{$r=11$}
    \end{subfigure}
  \end{tabular}
\caption[Cost functions and ranks for the LQG problem with different diffusion magnitudes and reflecting boundaries.]{Cost functions and ranks of the solution to the LQG problem for different $\sigma_1,\sigma_2$. We fix $u_{lb} = -3$, $u_{ub} = 3,$ and use reflecting boundary conditions. The FT ranks found through the cross approximation algorithm with rounding tolerance $\roundeps=10^{-7}$ are indicated by the caption for each frame. The x-axis of each plot denotes $x_1$ and the y-axis denotes $x_2$. The color scale is different in each plot to highlight each function's shape.}
\label{fig:lqr_reflect_diff}
\end{center}
\end{figure*}

Next, we consider the effect of the magnitude of the diffusion terms on the optimal value function and its rank, this time in problem instances involving reflecting boundary conditions. The results of this experiment are shown in Figure~\ref{fig:lqr_reflect_diff}, where it is evident that diffusion magnitude is again influential for determining the rank of this problem.
The shapes and ranks of these value functions are similar to those in the case of absorbing boundary conditions. They follow the same pattern of increasing rank as the diffusion magnitude decreases.

\subsubsection{Convergence}
In this section, we focus on the convergence properties of the proposed algorithms in computational experiments. We discuss convergence in terms of {\em (i)} the norm of the value function, {\em (ii)} the difference between iterates, and {\em (iii)} the fraction of states visited during each iteration. We consider both the FT-based value iteration and the FT-based policy iteration algorithms. 

Let us first consider the FT-based value iteration given in Algorithm~\ref{alg:ttvi}. We use FT tolerances of $\crossdelta=\roundeps=10^{-7}$ as input to the algorithm. In Figure~\ref{fig:diagvi}, we compare the convergence and the computational cost for solving the stochastic optimal control problem for varying discretization of the Markov chain approximation method, specifically discretization sizes of $n=25$, $n=50$, and $n=100$ points along each dimension. Note that this corresponds to $25^2=625$, $50^2=2,500$, and $100^2=10,000$ total number of discrete states, respectively. 
\begin{figure*}
\begin{center}
\begin{subfigure}[t]{0.32\textwidth}
  \includegraphics[height=0.80\textwidth,clip=true,trim=0 0 0 0,width=\textwidth]{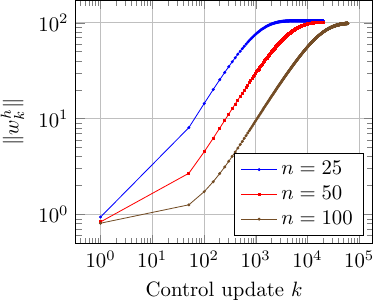}
\end{subfigure}
\begin{subfigure}[t]{0.32\textwidth}
  \includegraphics[height=0.80\textwidth,clip=true,trim=0 0 0 0,width=\textwidth]{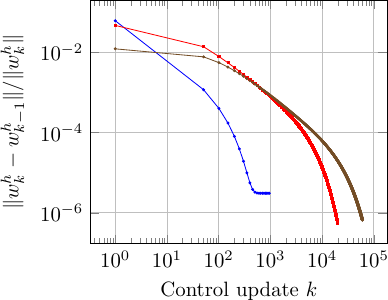}
\end{subfigure}
\begin{subfigure}[t]{0.32\textwidth}
  \includegraphics[clip=true,trim=0 0 0 0,width=\textwidth]{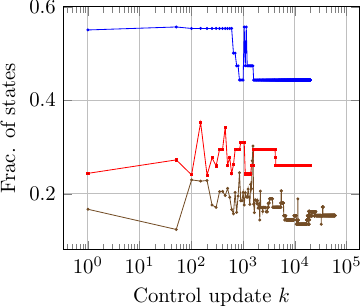}
\end{subfigure}
\caption[FT-based value iteration diagnostic plots for the LQG problem with reflecting boundaries]{FT-based value iteration diagnostic plots for the linear quadratic problem with reflecting boundaries, $u(t) \in [-1,1]$, and $\sigma_1=\sigma_2=1.$ The left panel shows that for all discretization levels the value function norm converges to approximately the same value. The middle panel shows the relative difference between value functions of sequential iterations. The right panel shows that the fraction of states evaluated within cross approximation decreases with increasing discretization resolution.}
\label{fig:diagvi}
\end{center}
\end{figure*}

The results demonstrate that approximately the same value function norm is obtained regardless of discretization. However, convergence is much faster for coarse discretizations, hence the the one-way multigrid algorithm may be useful. Furthermore, we observe that the low-rank nature of the problem emerges because the fraction of discretized states evaluated during cross approximation for each iteration decreases with increasing grid resolution. For reference, the standard value iteration algorithm would have the fraction of state space evaluated be $1$ for each iteration, as the standard value iteration would evaluate all discrete states. Thus, even in this two-dimensional example, the low-rank algorithm achieves between two to five times computational gains for each iteration, when compared to the standard value iteration algorithm.

Next, we repeat this experiment for FT-based policy iteration algorithm given by Algorithm~\ref{alg:ttpi}. We use 10 sub-iterations to solve for the value function for every policy, and we use FT tolerances of $\crossdelta=\roundeps=10^{-7}$ as input to the algorithm.  Figure~\ref{fig:diagpi} shows the results for the value function associated with each control update. The sub-iteration cost functions are not plotted. Note that, as expected, far fewer iterations are required for convergence. We observe similar solution quality when compared to the FT-based value iteration algorithm; however, we observe that convergence occurs using almost an order of magnitude fewer number of iterations. \rev{The timings using 8 threads were $0.75$s, $0.96$s and $1.32$s per control update for 25 $\times$ 25, 50 $\times$ 50, and 100 $\times$ 100 grids, respectively. These timings include 10 sub-iterations in PI along with a optimization update. So for example,  1000 iterations of the medium grid takes approximately 960 seconds or 16 minutes, while for the coarse grid it would take 12 minutes. These timings motivate the need for multilevel schemes since the coarse grid converges in fewer iterations to almost the exact solution.}
\begin{figure*}
\begin{center}
\begin{subfigure}[t]{0.30\textwidth}
  \includegraphics[height=0.80\textwidth,clip=true,trim=0 0 0 0,width=\textwidth]{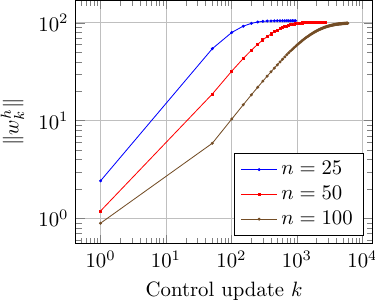}
\end{subfigure}
\begin{subfigure}[t]{0.30\textwidth}
  \includegraphics[height=0.80\textwidth,clip=true,trim=0 0 0 0,width=\textwidth]{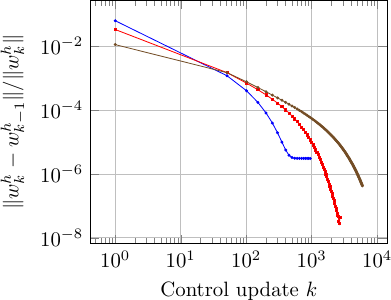}
\end{subfigure}
\begin{subfigure}[t]{0.30\textwidth}
  \includegraphics[clip=true,trim=0 0 0 0,width=\textwidth]{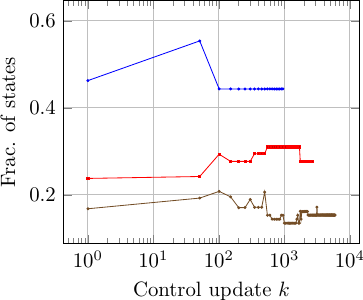}
\end{subfigure}
\caption[FT-based policy iteration diagnostic plots for the LQG problem with reflecting boundaries]{FT-based policy iteration diagnostic plots for the linear quadratic problem with reflecting boundaries, $u(t) \in [-1,1]$, and $\sigma_1=\sigma_2=1.$ The left panel shows that for all discretization levels the value function norm converges to same value. The middle panel shows the relative difference between value functions of sequential iterations. The right panel shows that the fraction of states evaluated decreases with increasing discretization.}
\label{fig:diagpi}
\end{center}
\end{figure*}

To summarize, from these experiments we observe the following: 
{\em (i)} the rank depends on the intrinsic complexity of the value function, which seems to increase with decreasing process noise when the state constraints are active; 
\rev{{\em (ii)} the rank does not seem to change with varying control bounds;}
{\em (iii)} even in these two dimensional problems we obtain substantial computational gains: the proposed algorithms evaluate two to five times less number of states to reach a high quality solution, when compared to standard dynamic programming algorithms; 
{\em (iv)} we observe that the FT-based policy iteration algorithm converges almost an order of magnitude faster than the FT-based value iteration algorithm in these examples. 

\rev{
\subsubsection{Discrete vs.\ continuous tensors  and underspecified ranks}

As described in Section~\ref{sec:crossround} the TT-based algorithm of~\cite{Gorodetsky2015b} can be interpreted as a specific realization of the continuous framework described in this paper that uses piecewise-constant (rather than piecewise-linear) reconstructions of each univariate fiber. For the case when the rank-adaptation scheme finds an upper bound to the exact ranks, we have found that the performance of these two approaches is similar.

However, for the more realistic case where we utilize low-rank \textit{approximations} of some high-rank function, we have noticed an accuracy benefit by using the more accurate continuous representation. As an example, we again consider the LQG problem with reflecting boundary conditions, $\sigma_1 = \sigma_2 = 1$, and $|u| \leq 1$. For this problem, we computed a reference solution on a 100 $\times$ 100 grid and found the rank to be four. Then we ran cross-approximation with fixed ranks of 2 and 3 for grids of size 25 $\times$ 25, 50 $\times$ 50, and 100 $\times$ 100 using both piecewise constant and piecewise linear reconstruction of the fibers. We compare the ratios of the root-mean-squared error of the resulting approximations at the grid's \textit{nodal locations}. Table~\ref{tab:discrete_vs_continuous} summarizes the results.

\begin{table}
  \centering
  \caption{Nodal RMSE ratios ($\text{RMSE}_{FT} / \text{RMSE}_{TT}$) between the piecewise-linear functional fiber reconstructions and the discrete TT approach of~\cite{Gorodetsky2015b}.}\label{tab:discrete_vs_continuous}
  \begin{tabular}{|cccc|}
    \hline
    Rank / Grid     & 25 $\times$ 25 & 50 $\times$ 50 & 100 $\times$ 100 \\ 
    \hline
    \hline
    2              & 0.687   &  0.517  & 0.504  \\
    3              & 2.222   &  0.071  & 0.226 \\
    \hline
  \end{tabular}
\end{table}

In this table we see that, for all but one case, the error of the continuous approach was smaller than the error of the discrete approach. The case where the error of the discrete approach was lower occurred in the regime of small rank underestimation and a coarse grid. This discrepancy may be attributed to the fact that imposing even a piecewise-linear reconstruction on a coarse grid may lead to overfitting. For all other cases, the continuous error was smaller. The difference in nodal values of these reconstructions is entirely due to a different inner product used within cross-approximation, and these results suggest that using the inner product resulting from piecewise-linear functions leads to higher accuracy when ranks have been underestimated.

}

\rev{
  \subsubsection{Scaling with dimension}


Next we test the ranks of value functions 
functions with increasing dimension.
 In particular we consider a set of double integrators 
\begin{align*}
dx_{2i-1} &= x_{2i} + \sigma_{2i-1}dw_{2i-1}(t) \\
dx_{2i} &= u_i + \sigma_{2i}dw_{2i}(t)
\end{align*}
for $i = 1,\ldots, \du$. Note that the state space has dimension $\dx = 2\du$. For the cost function we use $g(x,u) = \sum_{i=1}^{\dx}x_i^2 + \sum_{i=1}^{\du}u_i^2$, and we discretize each dimension into $n=50$ nodes.  

For an unbounded state space the problem would completely decouple into $\du$ separate double integrators of double integrators, and thus the value function would also decouple into a sum of quadratics of neighboring variables
\begin{equation*}
\val{}(x_1,\ldots,x_d) = \sum_{i=1}^{d/2} q_i(x_{2i-1},x_{2i}), 
\end{equation*}
where $q_i$ are bivariate quadratic functions,\footnote{We assume that the coefficients of the terms are one for simplicity of presentation} i.e., $q_i(z,z^{\prime}) = z^2 + zz^{\prime} + z^{\prime^2}$. In this case, we can show that the maximum rank of a low-rank decomposition will be $3$. This fact can be verified from the following decomposition,
\begin{align*}
\val{}(x_1,& \ldots,x_d) = \left[ x_1^2 \ \ x_1 \ \ 1 \right]
\left[ 
\begin{array}{cc}
1 & 0 \\
x_2 & 0 \\
x_2^2 & 1 \\
\end{array}
\right]
\left[
\begin{array}{ccc}
1 & 0 & 0 \\
x_3^2 & x_3 & 1 
\end{array}
\right] \times \\
& \left[ 
\begin{array}{cc}
1 & 0 \\
x_4 & 0 \\
x_4^2 & 1 \\
\end{array}
\right]
\left[
\begin{array}{ccc}
1 & 0 & 0 \\
x_5^2 & x_5 & 1 
\end{array}
\right] 
\times \cdots \times
\left[ 
\begin{array}{c}
1  \\
x_d \\
x_d^2 \\
\end{array}
\right] .
\end{align*}
Thus, the FT ranks alternate between 2 and 3, i.e., $\bvec{r} = (1, 3, 2, 3, 2, 3, 2, 3, \ldots, 1)$, and therefore, for unbounded domains, a low-rank decomposition can efficiently capture decoupling between variables. 

It is less clear, however, how the imposition of reflecting boundary conditions should affect the degree of coupling in this problem. We explore this question numerically, by imposing a bounded state space $x_i \in [-2,2]^{\dx}$ with reflecting boundary conditions. We also use the tolerance $\crossdelta=10^{-5}$ and $\roundeps=10^{-7}$. 
Table~\ref{tab:lqg_decoupled} shows that we discover the same alternating rank pattern as we expected in the unbounded domain problem. Furthermore, we see that the rank does not grow quickly with dimension. The maximum rank increases by at most one with each increment in the dimension of the state space. 

\begin{table}
\centering
\caption{Ranks of decoupled double integrators in $d$ dimensions}
\label{tab:lqg_decoupled}
\begin{tabular}{|cc|}
\hline
Dimension & Ranks \\
\hline
\hline
2 & (1,3,1) \\
4 & (1, 4, 3, 3, 1) \\ 
6 & (1, 4, 4, 6, 4, 4, 1) \\ 
8 & (1, 4, 4, 6, 4, 6, 4, 4, 1) \\ 
10 & (1, 5, 4, 7, 4, 7, 4, 7, 4, 6, 1)  \\
12 & (1, 4, 4, 7, 4, 7, 4, 7, 4, 7, 4, 5, 1) \\
\hline
\end{tabular}
\end{table}

}

\subsection{Car-like robots maneuvering in minimum time}\label{sec:cars}
Next, we consider two examples, each modeling car-like robots. The first example is a standard Dubins vehicle. The Dubins vehicle dynamics is nonlinear, nonholonomic, and has a three-dimensional state space. The second example extends the Dubins vehicle dynamics to vary speed; and at high speeds the vehicle understeers. The resulting system is also nonlinear and nonholonomic, and the state space is four dimensional. 

For all of the examples in this section, we use the FT-based policy iteration algorithm given by Algorithm~\ref{alg:ttpi} with $n_{fp}=10$, cross approximation and rounding tolerances set to $10^{-5}$.

\subsubsection{Dubins vehicle}\label{sec:dubin}

Consider the following dynamics:
\begin{align*}
dx &= \cos(\theta)dt + dw_1(t)\\
dy &= \sin(\theta)dt + dw_2(t)\\
d\theta &= u(t)dt + 10^{-2}dw_3(t),
\end{align*}
which is the standard Dubins vehicle dynamics with process noise added to each state. 
Consider bounded state space: $x \in (-4,4)$, $y \in (-4,4)$, and $\theta \in [-\pi, \pi]$. The control space consists of three points $\cs = \{-1,0,1\}$. Boundary conditions are absorbing for the position dimensions ($x,y$) and periodic for the angle $\theta$. An absorbing region is specified at the origin with a width of $0.5$, and the terminal costs are
\rev{
\[
\termcost{}(x,y,\theta) = \left\{
\begin{array}{cc}
0  &  \textrm{for } (x,y) \in [-0.25,0.25]^2 \\
10 & \textrm{for } |x| \geq 4 \textrm{ or } |y| \geq 4
\end{array}
\right.
\]
}

To compute a time-optimal control, the stage cost is set to \[\stagecost{}(x,u) = 1.\] The solution is obtained using one-way multigrid as discussed in Section~\ref{sec:multigrid}. First, one hundred steps of FT-based policy iteration are used with each dimension discretized into $n=25$ points. 
Then the result is interpolated onto a grid discretized into $n=50$ points per dimension,
and the problem is solved with fifty steps of FT-based policy iteration. This multigrid procedure is repeated for $n=100$ and $n=200$. 

A simulation of the resulting feedback controller for several initial conditions is shown in Figure~\ref{fig:dubinsim}.
\begin{figure*}
\begin{center}
\begin{subfigure}[t]{0.3\textwidth}
  \includegraphics[clip=true,trim=0 0 0 0,width=\textwidth]{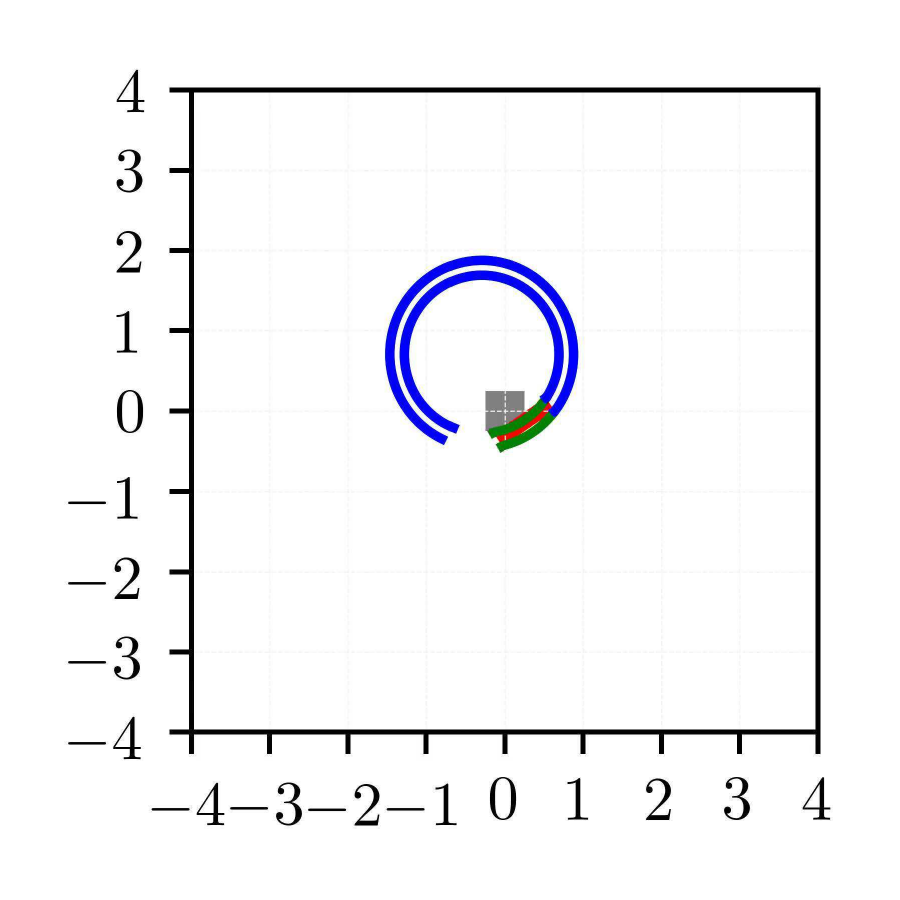}
\end{subfigure}\qquad
\begin{subfigure}[t]{0.3\textwidth}
  \includegraphics[clip=true,trim=0 0 0 0,width=\textwidth]{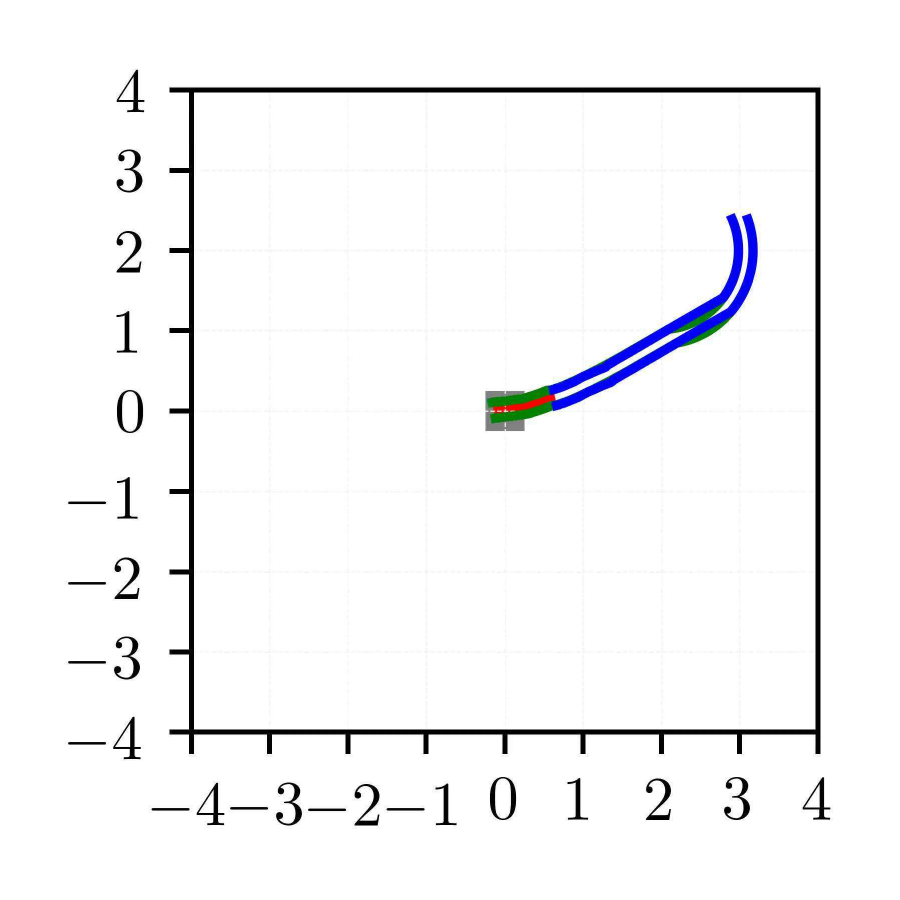}
\end{subfigure}\qquad
\begin{subfigure}[t]{0.3\textwidth}
  \includegraphics[clip=true,trim=0 0 0 0,width=\textwidth]{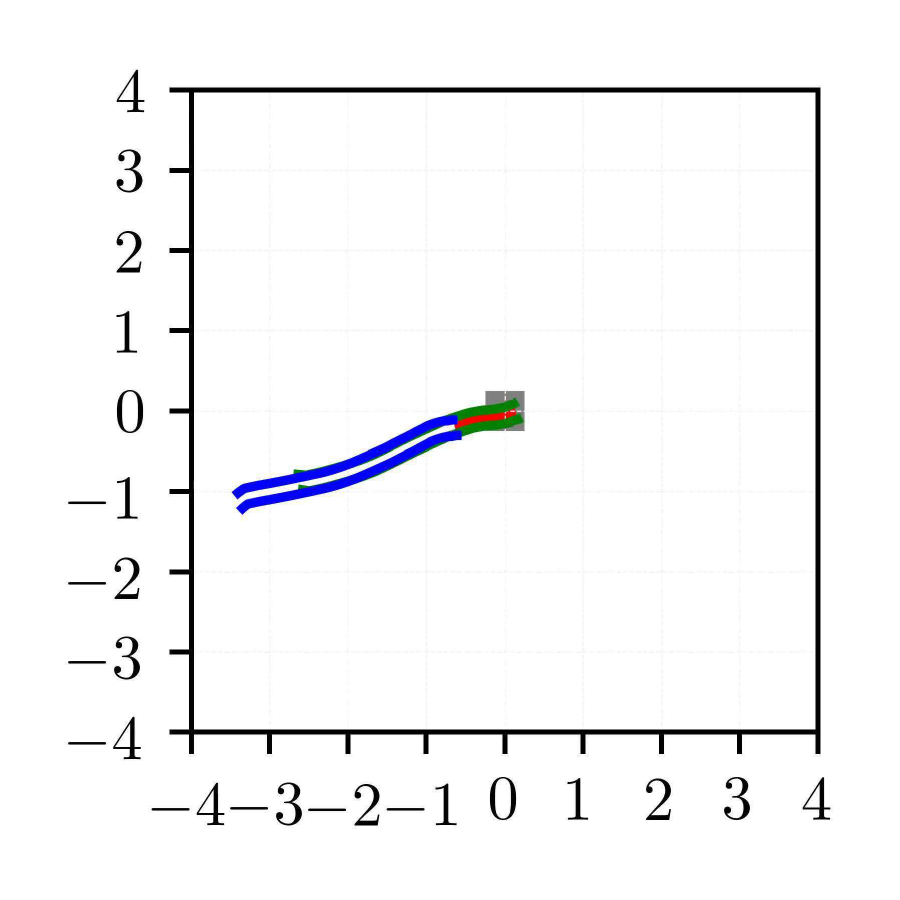}
\end{subfigure}
\caption[Trajectories of the Dubins car for three initial conditions.]{Trajectories of the Dubins car for three initial conditions. Panels show the car when it arrives in the absorbing region. \rev{The grey box is the target region, the red rectangle represents the final position of the car when it enters the absorbing region, and the parallel blue lines represent the trajectory of the car from each of the initial conditions.}}
\label{fig:dubinsim}
\end{center}
\end{figure*}

Convergence plots are shown in Figure~\ref{fig:dubinconverge}. It is worth noting at this point that these plots demonstrate one of the advantages of the FT framework: since the value function is represented as function rather than an array, we can compare the norms of functions represented with different discretizations. In fact, the upper left panel of Figure~\ref{fig:dubinconverge} demonstrates that the norm \textit{continuously} decreases when the discretization is refined.

Furthermore, these results suggest that the solution to this problem, with an accuracy due to rounding of $\roundeps=10^{-5}$, indeed has FT rank 12. This suggestion is supported by the fact that once the grid is refined enough, the fraction of states evaluated decreases, but the maximum rank levels off at 12. 
Note that, the total number of states changes throughout the iterations depending on the discretization level, \ie for $n=25, 50, 100, 200$, we have that the total number of states is $25^3=15,625$, $50^3=125,000$, $100^3 = 1,000,000$ and $200^3 = 8,000,000$. That is, at $n=200$, the total number of states reaches 8 million. 

The lower right panel shows the fraction of states evaluated at various iterations. We observe that when $n=100$, we evaluate only about 1\% of the states in each iteration. This number improves when $n = 200$. Hence, the FT-based algorithm provides an order of magnitude computational savings in terms of number of states evaluated, when compared to standard dynamic programming algorithms, even in this three-dimensional problem. 

\rev{Solving the Dubin's car using 8 threads required 4 minutes for $n=25$, 2.5 minutes for $n=50$, $4$ minutes for $n=100$, and $3$ minutes for $n=200$.}

\begin{figure*}
\begin{center}
\begin{subfigure}[t]{0.43\textwidth}
  \includegraphics[clip=true,trim=0 0 0 0,width=\textwidth]{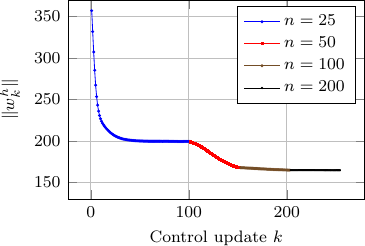}
\end{subfigure}\qquad
\begin{subfigure}[t]{0.43\textwidth}
  \includegraphics[clip=true,trim=0 0 0 0,width=\textwidth]{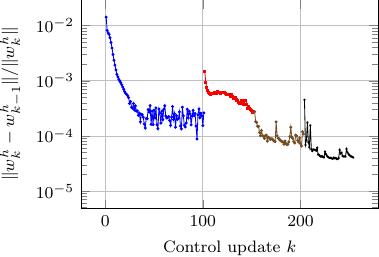}
\end{subfigure}
\begin{subfigure}[t]{0.43\textwidth}
  \includegraphics[clip=true,trim=0 0 0 0,width=\textwidth]{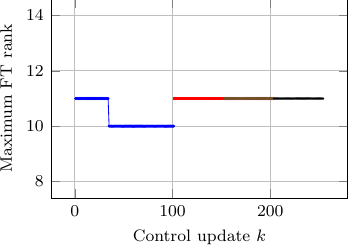}
\end{subfigure}\qquad
\begin{subfigure}[t]{0.43\textwidth}
  \includegraphics[clip=true,trim=0 0 0 0,width=\textwidth]{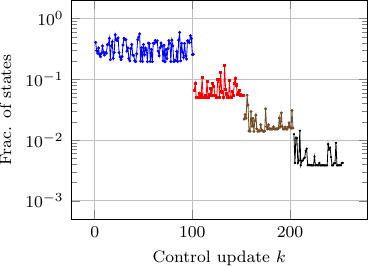}
\end{subfigure}
\caption{One-way multigrid for solving the Dubins car control problem.}
\label{fig:dubinconverge}
\end{center}
\end{figure*}
\subsubsection{Understeered car}
In this section, we consider a model that extends the Dubins vehicle. In this model, the speed of the vehicle is another state that can be controlled. If the speed is larger the vehicle starts to understeer, modeling skidding behavior. 
The states $(x,y,\theta,v)$ are now the $x$-position, the $y$-position, orientation, and velocity, respectively. The control for steering angle is $u_1(t) \in [-15\frac{\pi}{180},15\frac{\pi}{180}]$ and for acceleration is $u_2(t) \in [-1,1]$. \rev{We optimize over all controls in the tensor-product set $[-15\frac{\pi}{180}, 0, 15\frac{\pi}{180}] \times [-1, 0, 1]$.} The dynamical system is described by 
\begin{align*}
dx &= v \cos(\theta) dt + dw_1(t) \\
dy &= v \sin(\theta)dt + dw_2(t) \\ 
d\theta &= \frac{1}{1 + (v/v_c)}\frac{v}{L}\tan(u_1(t))dt + 10^{-2}dw_3(t) \\
dv &= \alpha u_2(t)dt + 10^{-2}dw_4(t)
\end{align*}
where $v_c = 8$ m/s is the characteristic speed, $L = 0.2$ m is the length of the car, and $\alpha=2$ is a speed control constant. The boundary conditions are absorbing for the the positions $x$ and $y$, periodic for $\theta$, and reflecting for $v$. The space for the positions and orientations are identical to that of the Dubins vehicle. The velocity is restricted to forward with $v \in [3,5].$ The stage cost is altered to push the car to the center
\begin{equation*}
\stagecost{}(x,y,\theta,v) = 1 + x^2 + y^2,
\end{equation*}
and we have expanded the absorbing region to have width 1. The terminal cost is the same as for the Dubins vehicle with an absorbing region at the origin of the $x,y$ plane. 

We remark that the dynamics of this example are {\em not} affine in control input. \rev{While they can be made affine by a simple transformation of variables on $u_1$, this problem still would not fit into many standard frameworks that solve a corresponding linear HJB equation~\cite{Horowitz2014,Yang2014} because of both the bounded state and arbitrary noise specifications.}

The trajectories for various starting locations are shown in Figure~\ref{fig:undersim}.
\begin{figure*}
\begin{center}
\begin{subfigure}[t]{0.3\textwidth}
  \includegraphics[clip=true,trim=0 0 0 0,width=\textwidth]{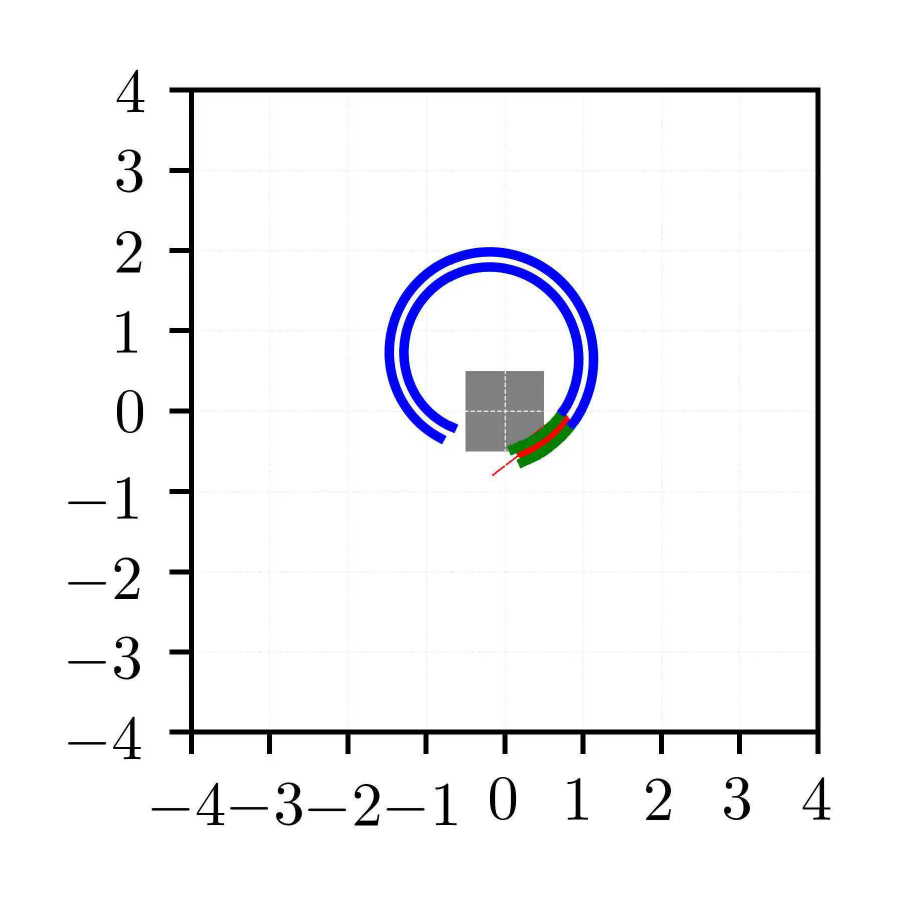}
\end{subfigure}\qquad
\begin{subfigure}[t]{0.3\textwidth}
  \includegraphics[clip=true,trim=0 0 0 0,width=\textwidth]{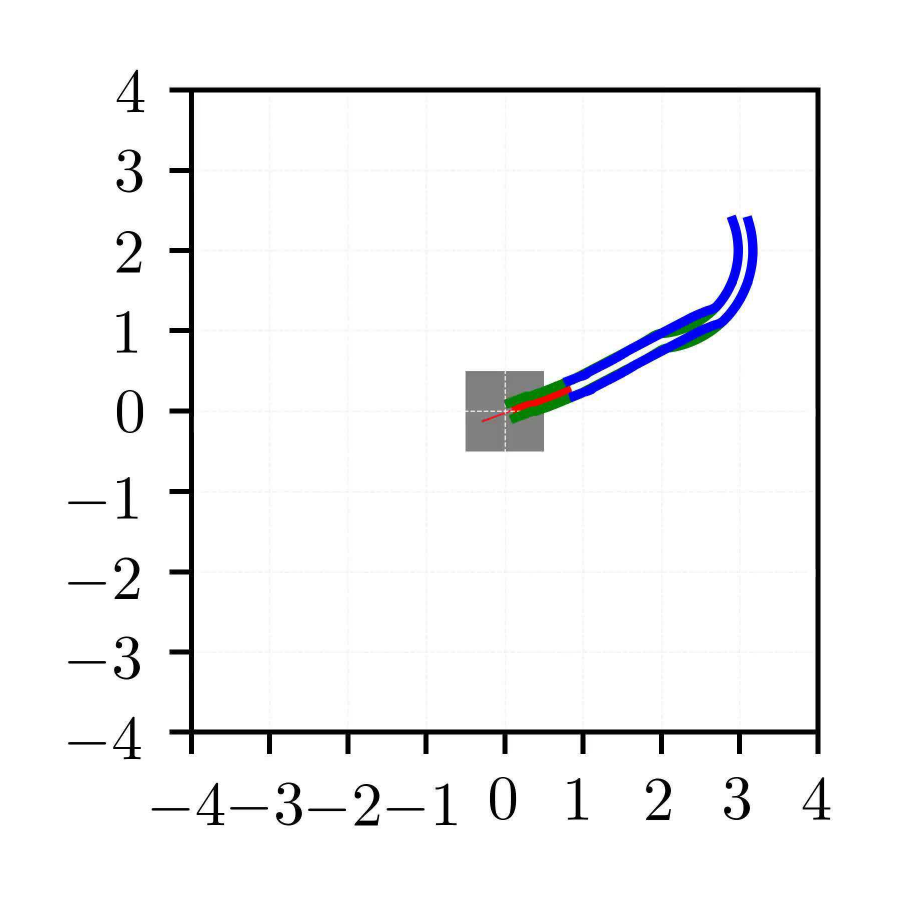}
\end{subfigure}\qquad
\begin{subfigure}[t]{0.3\textwidth}
  \includegraphics[clip=true,trim=0 0 0 0,width=\textwidth]{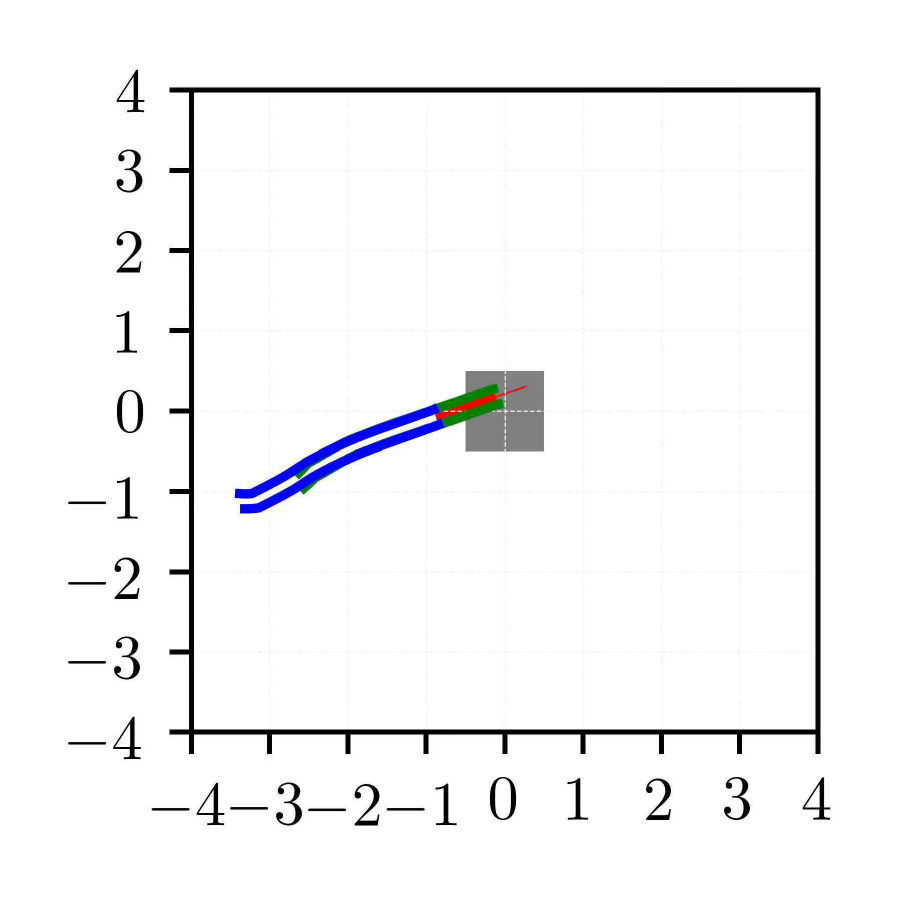}
\end{subfigure}
\caption[Trajectories of the understeered car for three initial conditions.]{Trajectories of the understeered car for three initial conditions. Panels show the car when it arrives in the absorbing region.\rev{The grey box is the target region, the red rectangle represents the final position of the car when it enters the absorbing region, and the parallel blue lines represent the trajectory of the rear wheels and the green lines indicate the trajectories of the front wheels of the car from each of the initial conditions.}}
\label{fig:undersim}
\end{center}
\end{figure*}
Figure~\ref{fig:underconverge} shows the convergence plots. Due to computational considerations, we fixed a maximum FT adaptation rank to 20. Therefore, instead of plotting the maximum rank, we plot the \textit{average} FT rank in the lower left panel. The average rank varies more widely for coarse discretizations, than for fine discretizations. But, when $n=100$, the average rank becomes smaller and more consistent. 

Let us note that, when $n = 100$, the number of discrete states is $100^4 = 100,000,000$, \ie 100 million. At this discretization level, the proposed FT-based algorithm evaluates less than 1\% of the states, leading to more than two orders of magnitude computational savings when compared to the standard dynamic programming algorithms. Even storing the optimal control as a standard lookup table in memory would require a large storage space, if the controller was computed using standard dynamic programming algorithms. 
At $n= 100$, the storage required is is around $0.8 \times 10^9$ bytes, or 800 MB, assuming we are storing floating point values. The proposed algorithm naturally stores the controller in a compressed format that leads to two orders of magnitude savings in storage as expected from the fraction of states evaluated. In fact, storing the final value function in the FT format required less than 1 MB of storage for this experiment. These savings can be tremendously beneficial in constrained computing environments. For example, these controllers can potentially be used on embedded systems that have serious memory constraints.

\rev{Finally, we note that the additional complexity of this problem over the Dubin's car is exhibited by larger average ranks. In particular, while the \textit{maximum} rank of the Dubin's car was 11, the \textit{average} rank of the understeered car is over 14.}
\begin{figure*}
\begin{center}
\begin{subfigure}[t]{0.43\textwidth}
  \includegraphics[clip=true,trim=0 0 0 0,width=\textwidth]{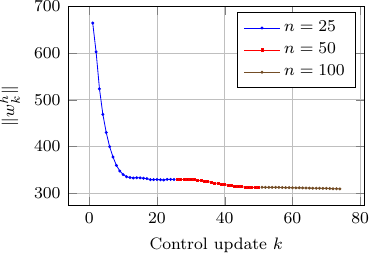}
\end{subfigure}\qquad
\begin{subfigure}[t]{0.43\textwidth}
  \includegraphics[clip=true,trim=0 0 0 0,width=\textwidth]{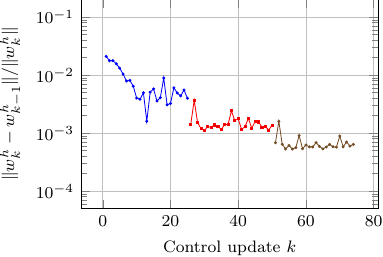}
\end{subfigure}
\begin{subfigure}[t]{0.43\textwidth}
  \includegraphics[clip=true,trim=0 0 0 0,width=\textwidth]{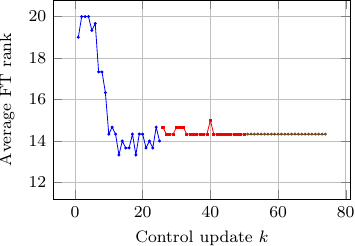}
\end{subfigure}\qquad
\begin{subfigure}[t]{0.43\textwidth}
  \includegraphics[clip=true,trim=0 0 0 0,width=\textwidth]{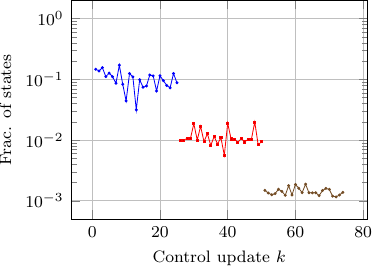}
\end{subfigure}
\caption[One-way multigrid for solving the understeered car control problem.]{One-way multigrid for solving the understeered car control problem. Maximum rank FT rank is restricted to 20.}
\label{fig:underconverge}
\end{center}
\end{figure*}

\rev{Solving the understeered car using 8 threads required 9 minutes for $n=25$, 19 minutes for $n=50$, and $35$ minutes for $n=100$.}

\subsection{Glider perching on a string}\label{sec:perch}
We now consider a problem involving a glider with a seven-dimensional state space. The mission is to control the glider to perch on a horizontal string. This problem has been widely studied in the literature~\cite{Cory2008,Roberts2009,Moore2012}. To the best of our knowledge, no optimal controller is known. We compute a controller using the FT-based dynamic programing algorithms. 

The glider is described by flat-plate model in the two-dimensional plane involving seven state variables, namely $(x,y,\theta,\phi,v_x,v_y, \dot{\theta})$, specifying its $x$-position, $y$-position, angle of attack, elevator angle, horizontal speed, vertical speed, and the rate of change of the angle of attack, respectively. The input control is the rate of change of the elevator angle $u = \dot{\phi}$. 

A successful perch is defined by a horizontal velocity between 0 and 2 m/s, a vertical velocity between -1 and -3 m/s, and the $x$ and $y$ positions of the glider within a \revv{5cm} radius of the perch. Under these conditions, \revv{Moore and Tedrake~\cite[Sec. III]{Moore2012} have indicated that the experimental aircraft can attach to the string}. For more information on this experimental platform, the reader is referred to to either Roberts et al.~\cite{Roberts2009} or Moore and Tedrake~\cite{Moore2012}. \rev{Here we use these specifications to help specify stage costs, boundary conditions, and terminal costs. In particular, we use them to place an absorbing region with zero cost to motivate the system to achieve these objectives.

While this example was previously solved using LQR Trees~\cite{Tedrake:2010jr}, we note that our approach here does not involve trajectory generation or linearization. Our formulation is one of an optimal stochastic control problem where we seek an optimal feedback control during an \textit{offline} procedure. As such our approach can also be used to determine controllability of a particular problem and to enable the robot to learn how to exploit its own dynamics.}

The dynamics of the glider are given by 
\begin{align*}
	\mathbf{x}_w &= [x-l_{w} c_{\theta}, y-l_ws_{\theta}], \\ 
	\dot{\mathbf{x}}_w &= [\dot{x} + l_w\dot{\theta}s_{\theta}, \dot{y}-l_w\dot{\theta}c_{\theta}] \\
	\mathbf{x}_{e}& = [x-lc_\theta-l_e,c_{\theta+\phi}, y- ls_\theta-l_es_{\theta+\phi}] \\
	\dot{\mathbf{x}}_{e} &= [\dot{x}+l\dot{\theta}s_{\theta} + l_e(\dot{\theta}+u)s_{\theta+\phi}, 
					     \dot{y}-l\dot{\theta}c_{\theta} - l_e(\dot{\theta}+u)c_{\theta+\phi}]\\
	\alpha_w &= \theta-\tan^{-1}(\dot{y}_{w},\dot{x}_w), \quad \alpha_e = \theta + \phi - \tan^{-1}(\dot{y}_{e},\dot{x}_e) \\
	f_{w} &= \rho S_{w}|\dot{\mathbf{x}}_w|^2\sin(\alpha_{w}), \quad f_{e} = \rho S_{e}|\dot{\mathbf{x}}_e|^2\sin(\alpha_{e}) 
\end{align*}
\begin{align*}
    dx &= v_xdt + 10^{-9}dw_1(t) \\
    dy &= v_ydt + 10^{-9}dw_2(t) \\
    d\theta &= \dot{\theta}dt + 10^{-9}dw_3(t) \\
    d\phi  &= udt + 10^{-9}dw_4(t) \\
    dv_x &= \frac{1}{m} \left(-f_w s_\theta - f_{e}s_{\theta+\phi}\right) dt + 10^{-9}dw_5(t) \\
    dv_y &= \frac{1}{m} \left( f_wc_\theta + f_ec_{\theta+\phi} -mg\right)dt + 10^{-9}dw_6(t)  \\
    d\dot{\theta} &= \frac{1}{I} \left( -f_wl_w - f_e(lc_\phi + l_e \right)dt + 10^{-9}dw_7(t) 
\end{align*}
where $\rho$ is the density of air, $m$ is the mass of the glider, $I$ is the moment of inertia of the glider, $S_w$ and $S_e$ are the surface areas of the wing and tail control surfaces, $l$ is the length from the center of gravity to the elevator, $l_w$ is the half chord of the wing, $l_e$ is the half chord of the elevator, $c_{\gamma}$ denotes $\cos(\gamma)$, and  $s_{\gamma}$ denotes $\sin(\gamma)$. The values of these parameters are chosen to be the same as those proposed by Roberts et al.~\cite{Roberts2009}. 

Absorbing boundary conditions are used, and an absorbing region is defined to encourage a successful perch. Let this region be defined for $x \in [-0.05,0.05], y \in [-0.05,0.05], v_x \in [-0.25,0.25],$ and $v_y \in [-2.25,2.25]$. 

The stage cost is specified as 
\begin{equation*}
\resizebox{0.95\hsize}{!}{$
\stagecost{}(x,y,\theta,\phi,v_x,v_y,\dot{\theta}) =20x^2 + 50y^2 + \phi^2 + 11v_x^2 + v_y^2 + \dot{\theta}^2.
$}
\end{equation*}

The terminal cost for both of these regions is
\begin{align*}
\termcost{}(x,y,\theta,\phi,v_x,v_y, \dot{\theta}) \qquad\qquad\qquad\qquad\qquad\qquad\qquad\\
= \left\{
\begin{array}{cl}
0 &  \textrm{ if perched} \\[0.5em]
600x^2 + 400y^2 + \frac{1}{9}\theta^2 + \frac{1}{9}\phi^2 + v_x^2 \\+ \left(v_y+1.5\right)^2 + \frac{1}{9}\left(\dot{\theta}+0.5^2\right) & \textrm{ otherwise. }
\end{array}
\right.
\end{align*}

A controller is computed using the FT-based one-way multi-grid algorithm.
Several trajectories of the controlled system under this controller are shown in Figure~\ref{fig:perchsim}. These trajectories are generated by starting the system from different initial states. Notice that they all follow the same pattern: first dive, then climb, and finally drop into the perch. This behavior is similar to that demonstrated in experiments by Cory and Tedrake~\cite{Cory2008}, Roberts et al.~\cite{Roberts2009}, and Moore and Tedrake~\cite{Moore2012}. 
\begin{figure*}
\begin{center}
\includegraphics[clip=true,trim=0 0 0 0, width=0.9\textwidth]{{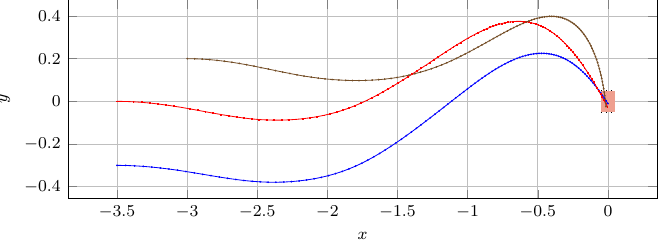}}
\caption[Trajectories of the perching glider for three initial conditions.]{\rev{Trajectories of the perching glider for three initial conditions given by: $(-3.5,-0.3,0,0,6.2,0,0)$ (blue), $(-3.5,0,0,0,5.8,0,0)$ (red), and $(-3,0.2,0,0,5.3,0,0)$ (brown). Target region is shown by the shaded rectangle centered at $(0,0)$.}}
\label{fig:perchsim}
\end{center}
\end{figure*}

Figure~\ref{fig:perchconverge} shows the convergence diagnostic plots for the one-way multigrid algorithm used for solving this problem. We used discretization levels of $n=20,40,80,160$ points along each dimension. Thus, the number of discretized states at the finest discretization is $160^7 \approx 2.6 \times 10^{15}$, \ie 2.6 quadrillion or 2,600,000 billion. Furthermore, we limited the rank to a maximum of \rev{15}, and the panel in lower left panel shows that the average ranks \rev{reach this maximum threshold.} In terms of the number of states evaluated, the final ranks of the value function translate into savings of approximately four orders of magnitude \textit{per iteration} when $n=20$, and a corresponding savings of ten orders of magnitude \textit{per iteration} when $n=160$. However, in this example it seems that the rank truncation is indeed affecting the convergence of the problem. A noisy and volatile decay of the value function norm is seen in the upper left panel. Furthermore, the difference between iterates, in the upper right panel, indicates a relative error of approximation $10^{-2}$ which is above our FT rounding threshold of $10^{-5}.$ Nonetheless, the resulting controller seems to be successful at achieving the desired behavior. The storage size for the controller is stored in 940 kB, or approximately 1MB. This means that for an $n=80$ controller, the resulting storage cost would require $\mathcal{O}(10^8)$ MB or $\mathcal{O}(100)$ TB. \rev{The solution times were 21 minutes for $n=20$, 110 minutes for $n=40$, 296 minutes for $n=80$ and 364 minutes for $n=160$.}

\begin{figure*}
\begin{center}
\begin{subfigure}[t]{0.43\textwidth}
  \includegraphics[clip=true,trim=0 0 0 0,width=\textwidth]{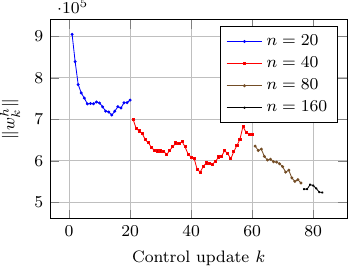}
\end{subfigure}\qquad
\begin{subfigure}[t]{0.43\textwidth}
  \includegraphics[clip=true,trim=0 0 0 0,width=\textwidth]{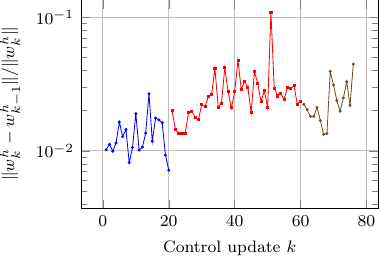}
\end{subfigure}
\begin{subfigure}[t]{0.43\textwidth}
  \includegraphics[clip=true,trim=0 0 0 0,width=\textwidth]{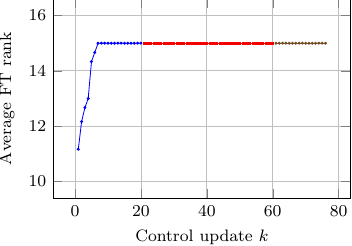}
\end{subfigure} \qquad
\begin{subfigure}[t]{0.43\textwidth}
  \includegraphics[clip=true,trim=0 0 0 0,width=\textwidth]{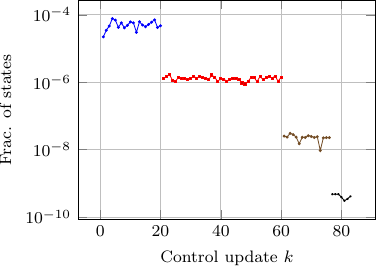}
\end{subfigure}
\caption[One-way multigrid for solving the perching glider control problem.]{One-way multigrid for solving perching glider. Maximum rank FT rank is restricted to 15.}
\label{fig:perchconverge}
\end{center}
\end{figure*}

\subsection{Quadcopter flying through a window}\label{sec:quad}
Finally, we consider the problem of maneuvering a quadcopter through a small target region, such as a window. The resulting dynamical system modeling the quadcopter has six states and three controls. The states $(x,y,z,v_x,v_y,v_z)$ are the $x$-position in meters $x \in [-3.5,3.5]$, $y$-position in meters $y \in [-3.5,3.5]$, $z$-position in meters $z \in [-2,2]$, $x$-velocity in meters per second $v_x \in [-5,5]$, $y$-velocity in meters per second $v_y \in [-5,5]$, and $z$-velocity in meters per second $v_z \in [-5,5]$. The controls are the thrust (offset by gravity) $u_1 \in [-1.5,1.5],$ the roll angle $u_2 \in [-0.4,0.4]$, and the pitch angle $u_3 \in [-0.4,0.4].$ Then, the dynamical system is described by the following stochastic differential equation:
\begin{align*}
dx &= v_xdt + 10^{-1}dw_1(t) \\
dy &= v_ydt + 10^{-1}dw_2(t) \\
dz &= v_zdt + 10^{-1}dw_3(t) \\
dv_x &= \frac{u_1-mg}{m}\cos(u_2)\sin(u_3)dt + 1.2 dw_4(t)\\
dv_y &= -\frac{u_1-mg}{m}\sin(u_2)dt + 1.2dw_5(t))\\
dv_z &= \left( \cos(u_2)\cos(u_3)\frac{u_1-mg}{m} + g \right)dt + 1.2dw_6(t),
\end{align*}
and reflecting boundary conditions are used for every state. A similar model was used by Carrillo et al.~\cite{Carrillo2012}. 

A target region is specified as a cube centered at the origin, and a successful maneuver is one which enters the cube with a forward velocity of one meter per second with less than $0.15$m/s speed in the $y$ and $z$ directions. 

The terminal cost is assigned to be zero for this region, \ie 
\begin{equation*}
\termcost{}(x,y,z,v_x,v_y,v_z) = 0, 
\end{equation*}
for
\[(x,y,z,v_x,v_y,v_z) \in [-0.2,0.2]^3 \times [0.5, 1.5] \times [-0.2,0.2]^2.\]
The stage cost is set to 
\begin{align*}
\stagecost{}(x,y,z,v_x,&v_y,v_z,u_1,u_2,u_3) = \\
                       &60 + 8x^2 + 6y^2 + 8z^2 + 2u_1^2 + u_2^2 + 6u_3^2.
\end{align*}
The maximum FT-rank is restricted to 10. The tolerances were set as $\crossdelta=\roundeps=10^{-5}$.
While, theoretically, underestimating the ranks can potentially cause significant approximation errors, in this case we are still able to achieve a well performing controller. Investigating the effect of rank underestimation is an important area of future work. 

Trajectories of the optimal controller for various initial conditions are shown in Figures~\ref{fig:quadsim} and~\ref{fig:quadtraj}. In Figure~\ref{fig:quadsim}, the position and velocities each approach their respective absorbing conditions for each simulation. In the first simulation, the quadcopter quickly accelerates and decelerates into the goal region. The final positions do not lie exactly within the absorption region. This absorption region is, in a way, unstable since it requires the quadcopter enter with a forward velocity. The forward velocity requirement virtually guarantees that the quadcopter must eventually exit the absorption region. The second simulation starts with positive $y$ and $z$ velocities. The third simulation starts with a negative $y$ velocity and the quadcopter far away from the origin.

Note that the controls are all fairly smooth except when the quadcopter state is near or in the absorption region. At this point, the roll angle (and pitch angle in the first and third simulations) oscillates rapidly around zero. We conjecture this behavior is due to ``flatness'' of the value function. Since the quadcopter is so close to the absorption region the control inputs can change rapidly to ensure it stays there. 
\begin{figure*}
\begin{center}
\begin{subfigure}[t]{0.58\textwidth}
  \includegraphics[clip=true,trim=0 0 0 0,width=\textwidth]{{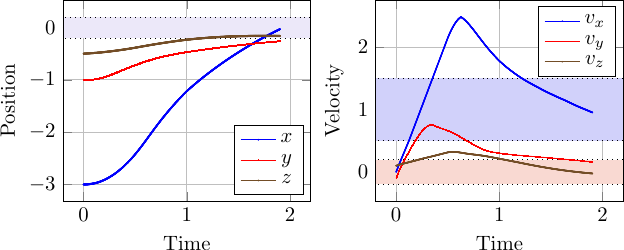}}
\end{subfigure}
\begin{subfigure}[t]{0.3\textwidth}
  \includegraphics[clip=true,trim=0 0 0 0,width=\textwidth]{{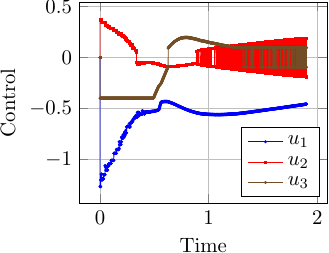}}
\end{subfigure}
\begin{subfigure}[t]{0.58\textwidth}
  \includegraphics[clip=true,trim=0 0 0 0,width=\textwidth]{{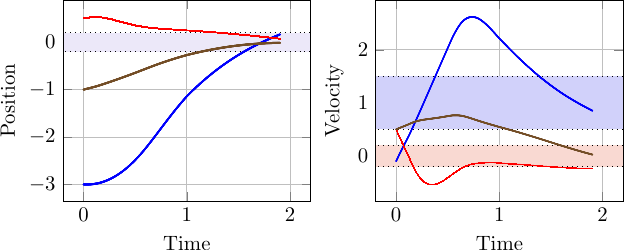}}
\end{subfigure}
\begin{subfigure}[t]{0.3\textwidth}
  \includegraphics[clip=true,trim=0 0 0 0,width=\textwidth]{{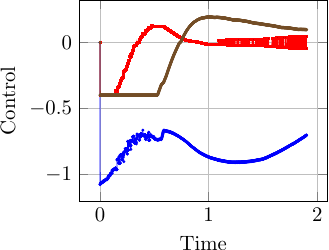}}
\end{subfigure}
\begin{subfigure}[t]{0.58\textwidth}
  \includegraphics[clip=true,trim=0 0 0 0,width=\textwidth]{{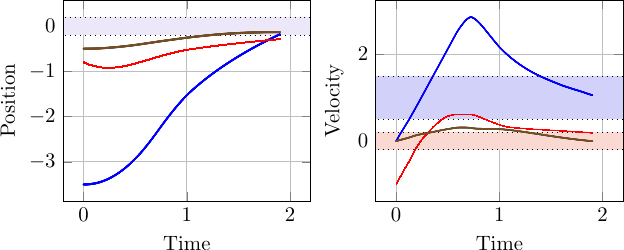}}
\end{subfigure}
\begin{subfigure}[t]{0.3\textwidth}
  \includegraphics[clip=true,trim=0 0 0 0,width=\textwidth]{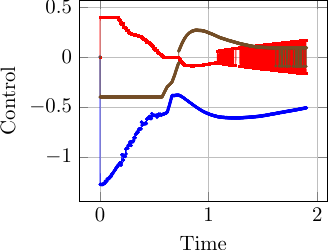}
\end{subfigure}
\caption[Trajectories of the quadcopter for three initial conditions.]{Three \rev{simulated} quadcopter trajectories using optimal low-rank feedback controller. Shaded region indicates the target positions and velocities.}
\label{fig:quadsim}
\end{center}
\end{figure*}

\begin{figure*}
\begin{center}
 \includegraphics[clip=true,trim=0 0 0 0]{{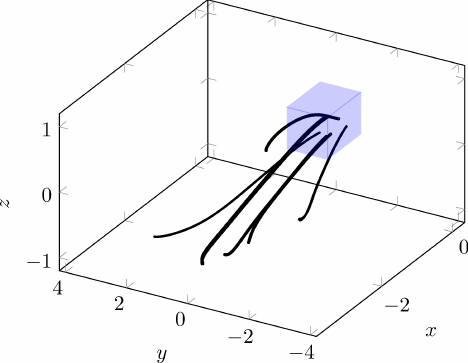}}
\caption{Three-dimensional \rev{simulated} trajectories showing the position of the quadcopter for various initial conditions.}
\label{fig:quadtraj}
\end{center}
\end{figure*}

Figure~\ref{fig:quadconverge} shows the convergence diagnostics. We used a one-way multigrid strategy with $n=20$, $n=60$, and $n=120$ discretization nodes. 
The difference between iterates, shown in the upper right panel, is between $1$ and $0.1$. 
In fact, this means that the relative difference is approximately $10^{-4}$, which matches well with the algorithm tolerances $\epsilon=(\crossdelta,\roundeps)$.

The lower right panel indicates computational savings between three and seven orders of magnitude for each iteration, in terms of the number of states evaluated. In terms of storage space, if we had used the standard value iteration algorithms, storing the value function as a complete lookup table for $n=120$ would require storing $10^6 = 2 \times 10^{12}$ floating point numbers, which is approximately 24 TB of data. The value function computed using the proposed algorithms required only 778 KB of storage space. \rev{The computational time was 40 minutes for $n=20$, 500 minutes for $n=60$, and 120 minutes for $n=120$.}
\begin{figure*}
\begin{center}
\begin{subfigure}[t]{0.43\textwidth}
  \includegraphics[clip=true,trim=0 0 0 0,width=\textwidth]{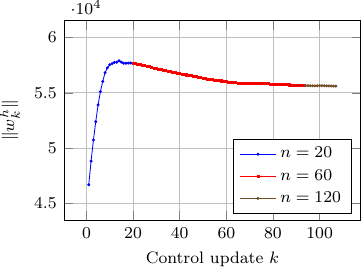}
\end{subfigure}
\begin{subfigure}[t]{0.43\textwidth}
  \includegraphics[clip=true,trim=0 0 0 0,width=\textwidth]{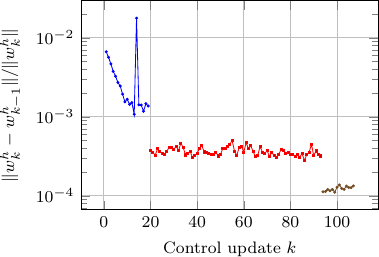}
\end{subfigure}
\begin{subfigure}[t]{0.43\textwidth}
  \includegraphics[clip=true,trim=0 0 0 0,width=\textwidth]{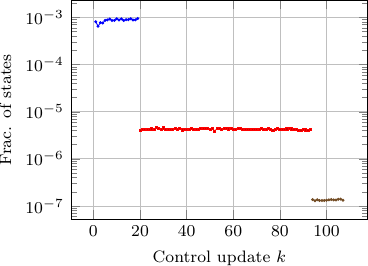}
\end{subfigure}
\begin{subfigure}[t]{0.43\textwidth}
  \includegraphics[clip=true,trim=0 0 0 0,width=\textwidth]{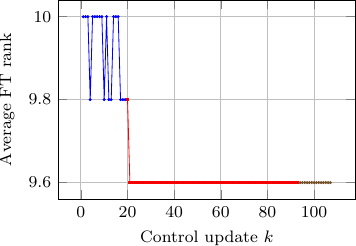}
\end{subfigure}
\caption[One-way multigrid for solving the quadcopter control problem.]{One-way multigrid for solving the quadcopter control problem. Maximum rank FT rank is restricted to 10.}
\label{fig:quadconverge}
\end{center}
\end{figure*}

\subsection{Experiments with a quadcopter in a motion capture room}\label{sec:experiments}
In this section, we report the results of experiments involving a quadcopter. We utilize a motion capture system for state estimation. We compute the controller offline, as described in the previous section, but we run the resulting controller on a computer on board the quadcopter in \emph{real time}. \rev{Our goal with this demonstration is to show that the proposed algorithms are practical, can be implemented onboard a system, and can be made to work in real time. As it stands, the evaluation of a cost function requires interpolation of univariate functions and multiplication of sets of matrices, and \textit{a priori} the feasibilty of this approach as a lookup table is not clear. In this section, we demonstrate that indeed we are able to achieve a performance level that allows for real time operation. 

Other approaches, e.g., using trajectory optimization with minimum snap control ~\cite{Mellinger2011,Richter13}, may be possible for this particular example. However, trajectory optimization approaches are fundamentally different from the offline-based optimal feedback control that we are presenting here. In particular, we do not first generate trajectories through 3D space that pass from some starting point to an end point, and then attempt to follow these trajectories. Instead, our optimization framework attempts to force the robot to discover what it can and cannot do using its dynamics and the specified cost. One advantage of this approach is that it is applicable to cases where trajectory generation is non-trivial (e.g., through complex configuration spaces). Furthermore, since it uses the dynamics of the system to discover behavior, it provides more information about controllability and better optimality properties. Future work can potentially attempt to couple these approaches by finding feasible trajectories using our optimal control formulation, and then attempting to follow them using minimum snap control.}

The experimental setup, the real-time control execution, and the experimental results are detailed below.

\subsubsection{Quadcopter hardware}
We use a custom-built quadcopter vehicle shown in Figure~\ref{fig:drone}. The vehicle is equipped with an embedded computer, namely an Nvidia Tegra K1. An Arduino Nano computer controls the propeller motors, and it is connected to the embedded computer via UART. The drone also includes a camera and an inertial measurement unit (IMU). THe IMU was used to estimate the drone's angular rates.

\begin{figure}
\centerline{\includegraphics[width=0.4\textwidth]{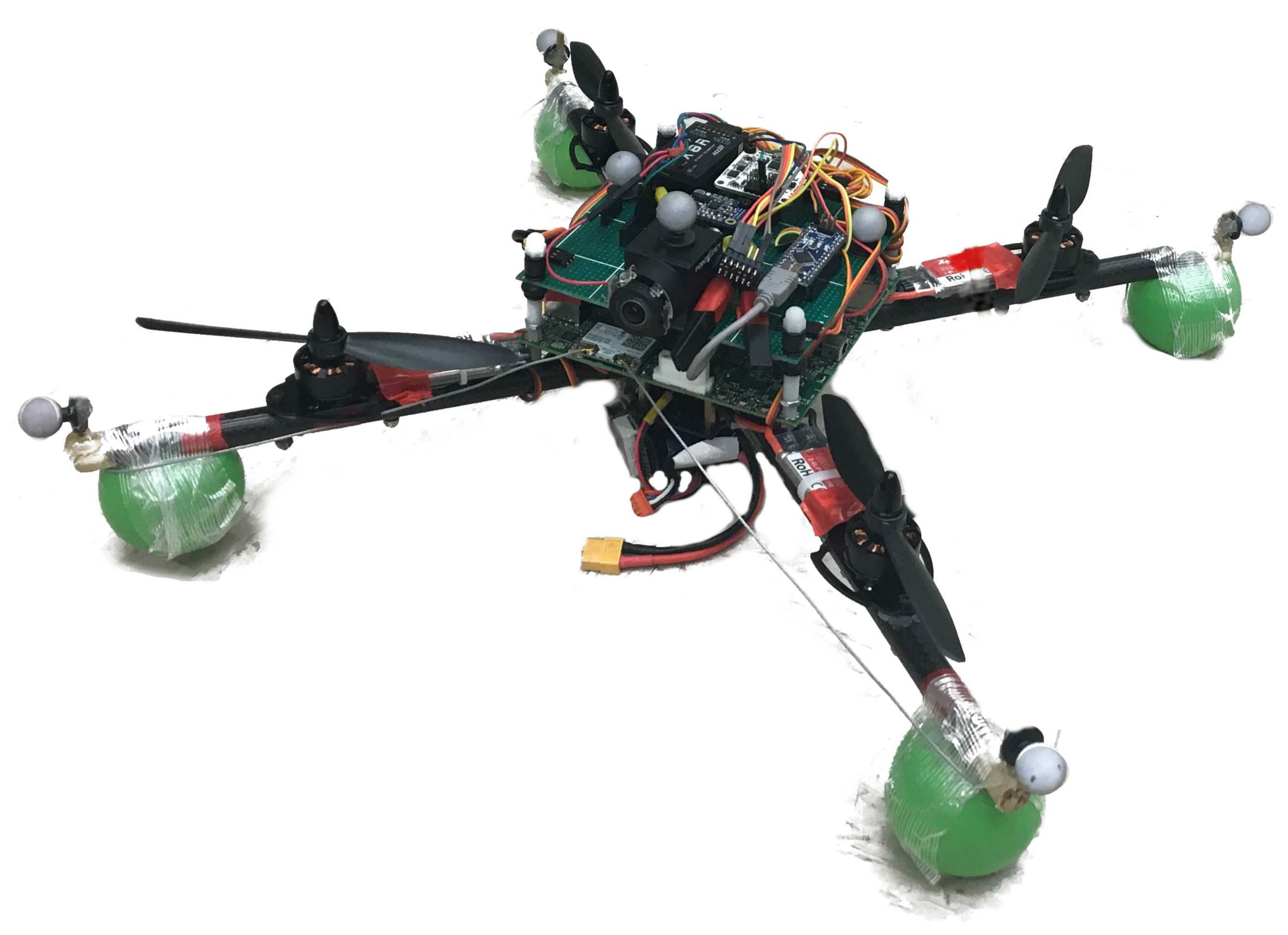}}
\caption{The quadcopter utilized in the experiments.}
\label{fig:drone}
\end{figure}

\subsubsection{Motion capture system}
We utilize an OptiTrack Motion Capture system~\footnote{http://www.optitrack.com/} with six Flex 13 cameras. We place the cameras such that the pose (position and orientation) estimates are reliably obtained within a roughly 2-meter wide, 5-meter long and 2-meter high volume. The system provides pose estimates in 360Hz. We place passive infrared markers on the quadcopter to track its pose when it is in this volume. The pose information is captured by the motion capture computer and sent to the drone via a wifi link. 

\subsubsection{Real-time execution of the the controller}
Once the vehicle gets the pose information via the wifi link, the control inputs, namely the four motor speeds, are computed using an Nvidia Tegra K1 computer that is on board the vehicle. 

We utilize a cascade controller architecture. First, high-level controller determines the desired pitch angle, roll angle, and thrust from the vehicle pose. Then, a low-level controller commands the motor speeds to follow the desired pitch angle, roll angle, and thrust. The high-level controller is designed using the proposed algorithm as in Section~\ref{sec:quad}. The low-level controller is a PD controller, described in~\cite{Riether2016}. 

The high-level controller execution follows exactly the same steps as for the simulated system and is described in the introduction to Section~\ref{sec:numexamples}. It consists of two steps: {\em (i)} computing the value corresponding to the neighboring states of the current state of the quadcopter, obtained from the motion capture system; {\em (ii)} obtaining the minimizer of Equation~\eqref{eq:dcih} through a multistart  Newton-based BFGS optimization scheme within the $C^3$ library. Note that the value function is computed offline and stored in the FT format, identically to the simulated results, and thus its evaluation at a particular state requires the multiplication of six matrices. \rev{The FT-based control is called at a rate of 100Hz.}

\subsubsection{Experimental results}
Figure~\ref{fig:quadwind} shows an image of the quadcopter at various times through its flight. Note that between the last two images, the quadcopter passes through the window. Figure~\ref{fig:mocap} shows motion capture data for the quadcopter entering the goal region; the velocities are shown on the left panel, and the path is shown on the right. 
As seen from the figures, the vehicle is able to fly the through the window with the intended velocities.

\begin{figure*}
\centering
\includegraphics[clip=true,trim=0 0 0 0,width=\textwidth]{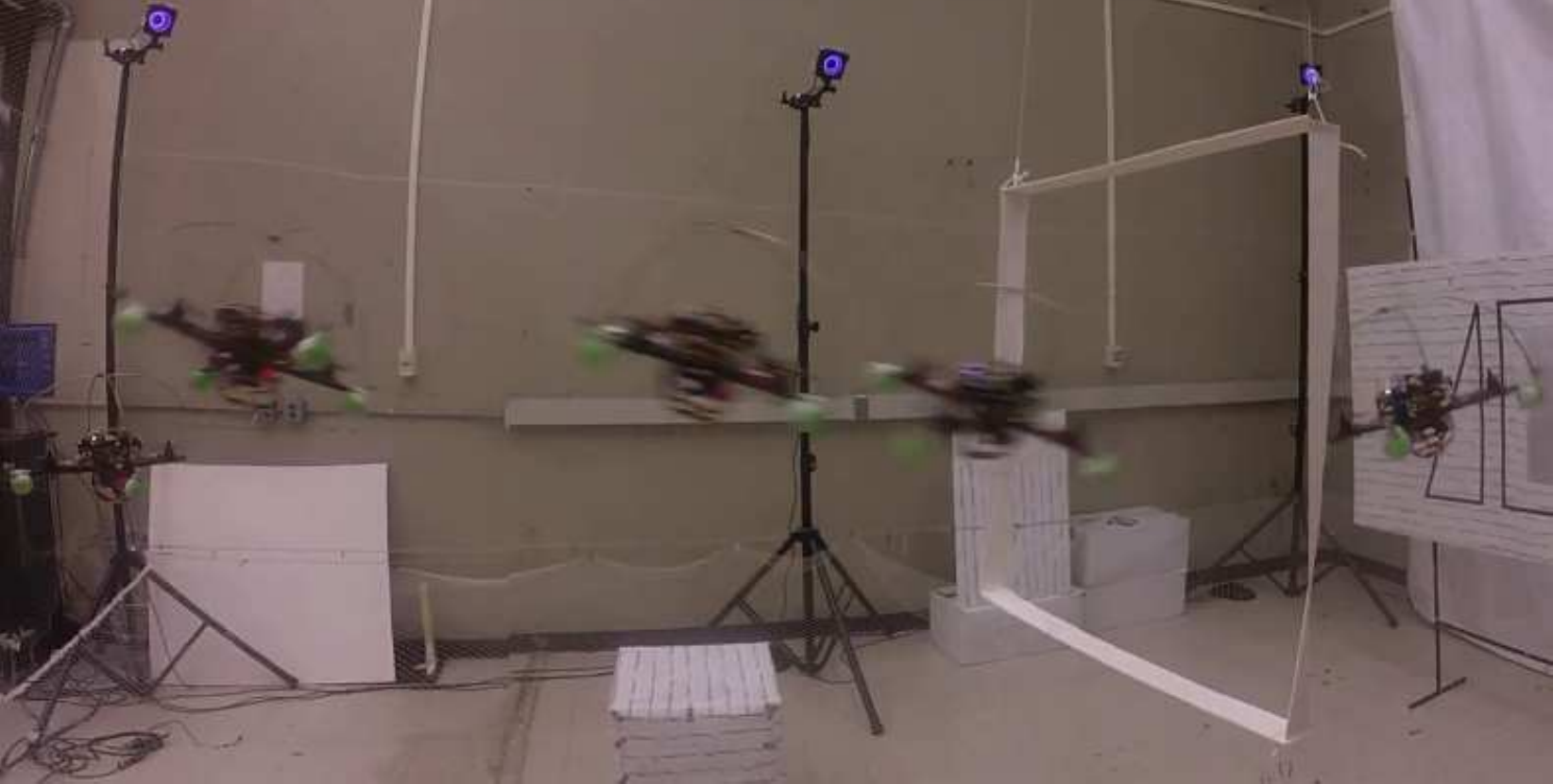}
\caption{Time lapse image showing the quadcopter passing through the window.}
\label{fig:quadwind}
\end{figure*}

\begin{figure*}
\centering
\begin{subfigure}[t]{0.45\textwidth}
  \includegraphics[width=\textwidth]{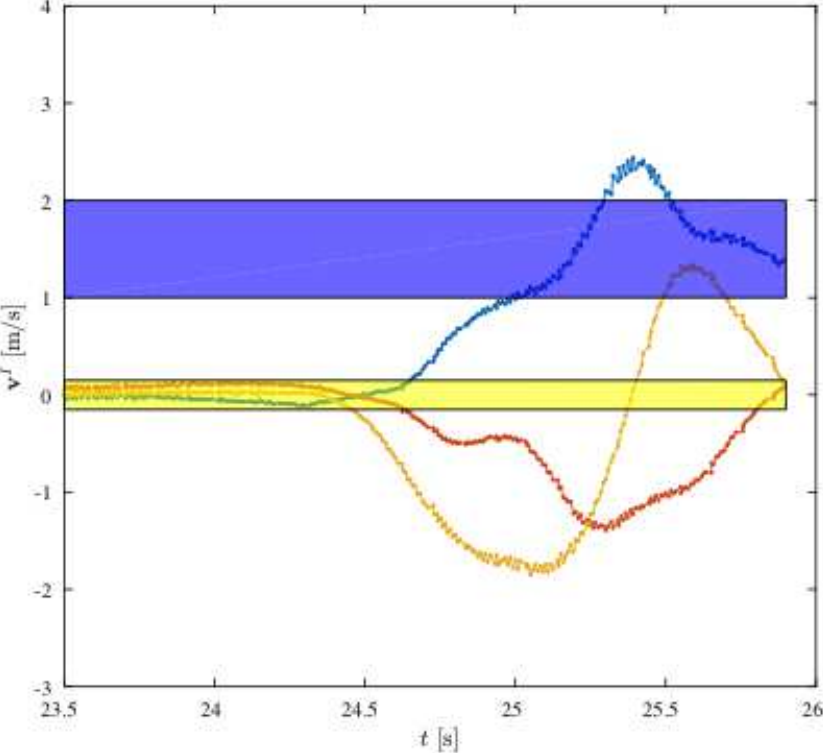}
  \caption{velocities. blue: $v_x^I$, red: $v_y^I$, yellow: $v_z^I$, with corresponding velocity-target regions.}
\end{subfigure}
\begin{subfigure}[t]{0.45\textwidth}
  \includegraphics[width=\textwidth]{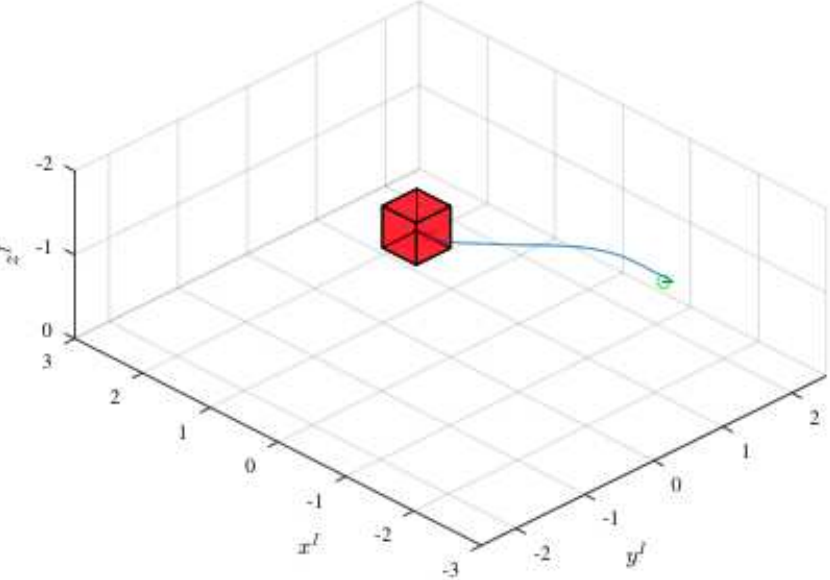}
  \caption{3D trajectory.}
\end{subfigure}
\caption{Motion capture results indicating quadcopter enters into target region~\cite{Riether2016}.}
\label{fig:mocap}
\end{figure*}

\section{Discussion and conclusion}\label{sec:conclusion}
In this paper, we proposed novel algorithms for dynamic programming, based on the compressed continuous computation framework using the function train (FT) decomposition algorithm.
Specifically, we considered continuous-time continuous-space stochastic optimal control problems. We proposed FT-based value iteration, FT-based policy iteration, and an FT-based one-way multigrid algorithms. All algorithms represent the value function in the FT format, and they apply multilinear algebra operations in the  (compressed) FT format. We analyzed the algorithms in terms of convergence and computational complexity. We have shown conditions under which the algorithms guarantee convergence to an optimal control. We have also shown that the algorithms scale polynomially with the dimensionality of the state space of the underlying stochastic optimal control problem and polynomially with the rank of the optimal value function. Furthermore, we demonstrated the proposed control synthesis algorithms on various problem instances. In particular, we have shown that the proposed algorithm can solve problem instances involving nonlinear nonholonomic dynamics with up to a seven-dimensional state space. The computational savings, evaluated in terms of computation time and storage space, can reach ten orders of magnitude, compared to the standard value iteration algorithm. Finally, we demonstrated the controller computed by the algorithms on a quadcopter flying quickly through a narrow window, using a motion capture system for state estimation. 

\rev{

Next, we comment on several directions for future work that are driven by our approach's relationship to reinforcement learning and approximate dynamic programming.

\subsection{Partial observability and learning}

Our contributions involved developing algorithms for solving dynamic programming problems for which the transition probabilities between states are known. It is based on separate \textit{offline} and \textit{online} procedures. The offline procedure is tasked with computing feedback controls for systems with known, but stochastic models, and the \textit{online} procedure recalls the optimal control from a state estimate during system operation. 

As a result, our approach does not incorporate learning and is therefore different from reinforcement learning strategies (see e.g., \cite{Mnih2013}). In reinforcement learning, a model is not known \textit{a priori}, but rather is learned during online training. A comparison between the methodology of reinforcement learning and stochastic optimal control is outside the scope of this work; however, we note that solutions to complex robotic systems will invariably require combining these approaches by updating solutions to known models with new information obtained during online operation. Indeed, we consider this direction an important topic for future work.

Furthermore, our work also assumes that the state is known. While that may be sufficient for fully observed systems with accurate sensors, more complex robotic systems will only be able to partially observe their states. To cope with both the learning problem and partial state estimation, future work will extend the proposed algorithms to include continuous-time partially observable Markov decision processes (POMDP)~\cite{Cassandra1998,Porta2006,Kurniawati2008}.

\subsection{Relation to approximate dynamic programming}\label{sec:relate_adp}

Our approach can be viewed as \textit{computationally enabling} value function approximation for high-dimensional problems within the context of approximate dynamic programming (ADP)~\cite{Powell2011,Bertsekas:2011tq}.

In particular, value function approximation itself typically suffers from the curse of dimensionality due to issues surrounding parameterization. Consider that in ADP with a continuous state space, 
 a general value function $\valbase: \reals^d \to \reals$ is often represented in a basis $\left(\phi_i\right)_{i=1}^N$ according to
\begin{equation*}
  \valbase(x) = \sum_{i=1}^N a_i \phi_i(x), \quad x \in \reals^d,
\end{equation*}
Thus, the infinite dimensional problem of optimizing over \textit{functions} is converted into a finite dimensional problem of optimizing over the \textit{coefficients} $(a_i)_{i=1}^n.$ Typically, the multidimensional basis functions $\phi_i:\reals^d \to \reals$ are themselves specified with a tensor product basis, i.e., $\phi_i = \phi_{i_1}(x_1)\phi_{i_2}(x_2) \cdots \phi_{i_d}(x_d)$, where each dimension is parameterized into $n_k$ variables with $1 \leq i_k \leq n_k$. As a result the number of basis functions scales exponentially with dimension $N = n^d$, and the value function representation becomes
\begin{align*}
  \valbase(x_1,&x_2,\ldots,x_d) =\\
& \sum_{i_1=1}^{n_1}\sum_{i_d=1}^{n_d}\sum_{i_d=1}^{n_d} a_{i_1i_2\ldots i_d} \phi_{i_1}(x_1)\phi_{i_2}(x_2)\cdots\phi_{i_d}(x_d),
\end{align*}
where $N^d$ coefficients $a_{i_1i_2\ldots i_d}$ must be stored. For this \textit{linear-approximation} setup, our algorithms essentially view the coefficients as a $n_1 \times n_2 \times \cdots \times n_d$ tensor. Instead of solving for each element of the tensor, our results suggest that assuming it is low-rank will allow us to reduce the dimensionality to a value that scales linearly with dimension and polynomially with rank. Moreover, our algorithms then solve the dynamic programming problem in the reduced space.

Nonlinear value function representations have also been used in the literature. For instance, in the context of control-affine systems, neural networks have been used to represent the value function~\cite{Liu2014,Wei2014}. Low-rank compression can be utilized for these nonlinear forms as well; see for example~\cite{Novikov2015}. However, convergence guarantees are difficult to provide for neural-network based approximation. The continuous tensor, or FT format, can also be used for accurate nonlinear approximation. For example, each fiber required by cross-approximation can be adaptively approximated by kernels or adaptive piecewise basis functions~\cite{Gorodetsky2015a}.

An important implication of this link is that for simplified problems, e.g., those with unbounded states and controls, those affine in control, etc., different parameterizations may improve computational feasibility. Indeed, in these setups, algorithms other than the Markov chain approximation method can lead to dynamic programming problems that are amenable to parameterizations with different basis functions. We envision our continuous tensor format to be applicable in these situations as well.

}

\subsection{Future directions}
We will also study the convergence properties of the proposed algorithms in more detail. In particular, we will investigate easily verifiable conditions for which solution ranks can be bounded and therefore algorithm convergence can be proven. If successful, these problems can be shown to be solvable in polynomial time. Finally, we will consider other application domains, such as distributed control systems. We conjecture that the power of the proposed algorithms is greatest when the underlying dynamical system is loosely coupled. We will investigate this conjecture in computational experiments. 

\section*{Acknowledgements}
We thank Ezra Tal for pointing out the proof for Lemma~\ref{th:afp}. We thank Fabian Riether for help in conducting the experiments on an autonomous quadcopter that he has developed in his masters thesis~\cite{Riether2016}. 

This work was supported by the National Science Foundation through grant IIS-1452019, and by the US Department of Energy, Office of Advanced Scientific Computing Research under award number DE-SC0007099.

\bibliography{jab}

\end{document}